\documentclass{article} 
\usepackage{iclr2027_conference,times}

\usepackage{amsmath,amsfonts,bm}

\def\eqref#1{equation~\ref{#1}}

\def\1{\bm{1}}

\DeclareMathAlphabet{\mathsfit}{\encodingdefault}{\sfdefault}{m}{sl}
\SetMathAlphabet{\mathsfit}{bold}{\encodingdefault}{\sfdefault}{bx}{n}

\usepackage[dvipsnames, table]{xcolor}
\usepackage{graphicx}
\usepackage{soul} 

\usepackage{amsthm}
\usepackage{empheq} 

\usepackage{booktabs}
\usepackage{multirow}
\usepackage{multicol}
\usepackage{threeparttable}
\usepackage{wrapfig}
\usepackage{enumitem}
\usepackage{footmisc}

\usepackage{subcaption}

\usepackage{algorithm}
\usepackage{algpseudocode}
\usepackage[most]{tcolorbox}
\usepackage{natbib}

\usepackage{hyperref}
\definecolor{ref_color}{rgb}{0.17, 0.48, 0.80}
\hypersetup{
    colorlinks=true,
    citecolor=ref_color,
    linkcolor=ref_color,
    urlcolor=ref_color
}
\usepackage{cleveref}
\usepackage{wrapfig}
\usepackage[normalem]{ulem}

\usepackage{tikz}
\usetikzlibrary{calc}
\usepackage{adjustbox}

\theoremstyle{plain}
\newtheorem{theorem}{Theorem}[section]

\theoremstyle{definition}

\theoremstyle{remark}

\title{MintFlow: Minimal Trajectory Intervention for Constrained Flow Matching}

\author{
Yesom Park\textsuperscript{1},
Kelvin Kan\textsuperscript{1},
Qifan Chen\textsuperscript{1},
Thomas Flynn\textsuperscript{2},
Hayden Schaeffer\textsuperscript{1},
Xihaier Luo\textsuperscript{2}\thanks{Corresponding author.}\\
\textsuperscript{1}Department of Mathematics, University of California, Los Angeles\\
\textsuperscript{2}Brookhaven National Laboratory\\
\texttt{\{yeisom, kelvin.kan, chenqifan, hayden\}@math.ucla.edu}\\
\texttt{\{tflynn, xluo\}@bnl.gov}}

\iclrfinalcopy 
\begin{document}

\maketitle

\begin{abstract}
Flow matching models excel at generative modeling, and many downstream applications require their samples to satisfy prescribed constraints, such as observed measurements and physical laws. However, existing constrained samplers often face a trade-off: \textit{enforcing constraints can substantially displace samples from the pretrained data distribution}. To address this trade-off, we introduce \textbf{MintFlow}, a training-free constrained sampling framework that formulates constraint enforcement as a minimal intervention on the pretrained flow trajectory. MintFlow seeks the minimal perturbation of an intermediate flow state such that its subsequent evolution under the pretrained flow field satisfies the target constraint. 
By minimally perturbing the flow state while keeping the pretrained flow field unchanged, MintFlow enforces the constraint while minimizing unnecessary deviation from the pretrained distribution. An adjoint formulation yields a closed-form expression for this perturbation, eliminating expensive iterative optimization. Furthermore, MintFlow adaptively selects the intervention time to balance the required perturbation magnitude with its amplification by the remaining flow. 
Across a range of tasks in generative vision and physical system modeling, MintFlow achieves competitive constraint satisfaction while preserving the pretrained generative distribution substantially better than state-of-the-art constrained methods.
\end{abstract}

\section{Introduction}
\label{sec:introduction}

Flow matching has emerged as a powerful paradigm for generative modeling, efficiently transporting simple reference distributions into complex data distributions through learned continuous-time dynamics~\citep{lipman2022flow,liu2026flow}. Yet many real-world applications demand more than plausible generation. 
Generated samples need to satisfy measurement constraints, structural requirements, physical laws, or other application-specific conditions~\citep{utkarsh2026physics,martin2025pnp,patel2024flowchef}. 
In these settings, the central objective is to enforce specified constraints while preserving the distributional structure already learned by the pretrained model.

Existing methods approach this problem either during training or at inference time. Training-time methods \citep{li2026gauge,tauberschmidt2026physics} incorporate constraints through specialized objectives, regularization, or architectural design, but require costly retraining for each new constraint, which is often impractical. Inference-time methods avoid this by enforcing the constraint during sampling, either by modifying the pretrained dynamics via guidance \citep{chung2024dps,feng2025guidance}, flow states through trajectory optimization \citep{ben2024d,patel2024flowchef}, or generated samples using post-hoc projection \citep{cheng2025gradient,utkarsh2026physics}.\footnote{Refer to Appendix~\ref{appen:related_work} for a detailed discussion of related work.} However, these approaches face a key challenge: \emph{constraint satisfaction alone does not prevent large deviations to the pretrained generation process, which can push samples away from the pretrained distribution} (Figure~\ref{fig:intro}).

Therefore, we adopt a simple principle: \emph{modify only what is necessary, and only as much as necessary}. As the desired constrained samples remain within the distribution learned by the pretrained flow, the learned dynamics themselves need not be altered. This suggests preserving the original vector field and intervening only on the trajectory with the smallest correction needed to satisfy the constraint while limiting unnecessary deviation from the pretrained distribution.

Motivated by this principle, we propose \textbf{MintFlow}, a training-free framework that formulates constrained generation as \emph{minimal trajectory intervention} (Figure~\ref{fig:comparison_correction_methods}). Given a trajectory generated by the pretrained flow, MintFlow finds the smallest perturbation to an intermediate state such that its subsequent evolution under the \emph{unchanged pretrained flow} satisfies the terminal constraint. The minimum-norm intervention is computed directly from the terminal constraint sensitivity via an adjoint equation, avoiding explicit construction of the high-dimensional flow Jacobian and costly iterative optimization through the full generation trajectory. MintFlow further adaptively selects the intervention time to balance the correction magnitude against the remaining flow propagation time.

Under mild regularity conditions, we characterize the constraint satisfaction error and show that the Wasserstein distance between the corrected and pretrained distributions is bounded by the intervention magnitude and the remaining flow propagation time. Experiments on image and physics data with diverse constraints, including image inverse problems, image editing, and physics-informed generation, demonstrate that MintFlow effectively satisfies constraints while preserving the learned distribution across a broad range of applications.
\begin{figure}
    \centering
    \includegraphics[width=\linewidth]{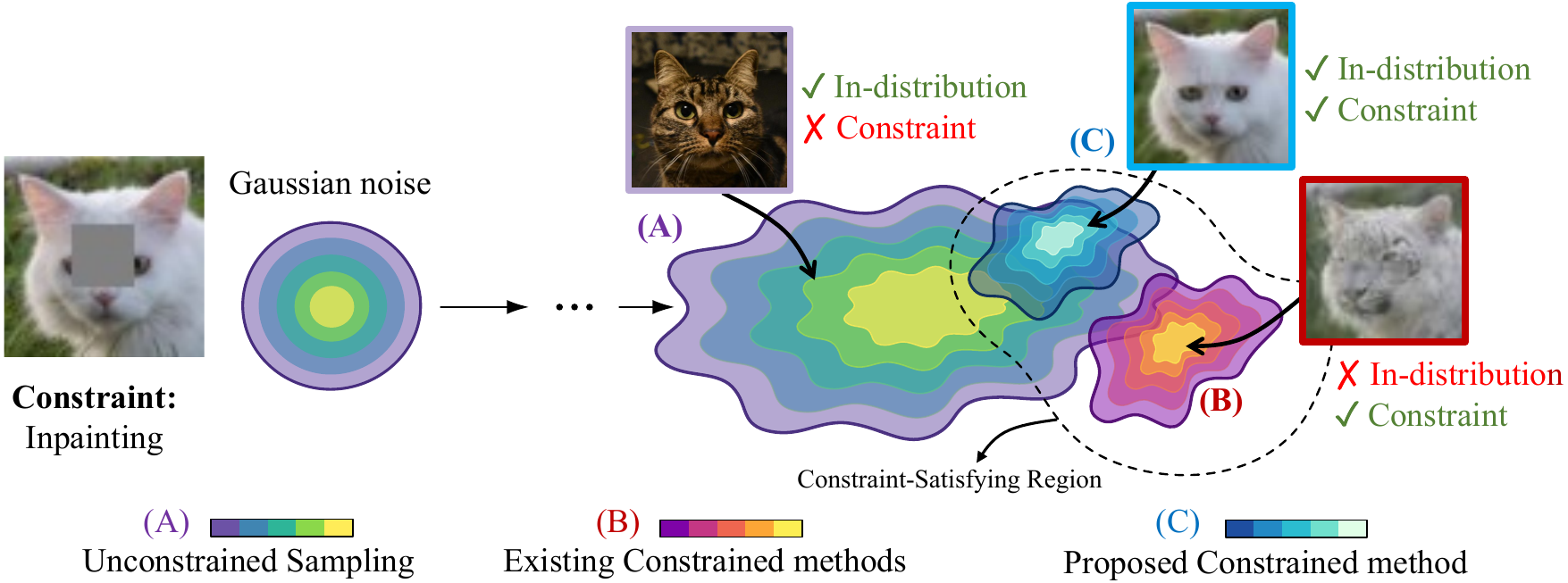}
    \caption{Balancing constraint satisfaction and distributional fidelity. \textbf{Unconstrained sampling (A)} generates samples from the pretrained distribution but often violates the constraints, while \textbf{prior constrained sampling methods (B)} enforce the constraints at the cost of large distributional shifts. \textbf{The proposed method (C)} achieves constraint satisfaction while better preserving the pretrained distribution. Image inpainting examples illustrate this trade-off.}
    \label{fig:intro}
\end{figure}
\section{Background}
Let $f_\theta$ be the velocity field of a pretrained flow matching model, which generates samples by integrating the ordinary differential equation (ODE)
\begin{equation}\label{eq:fm_ode}
\frac{du}{dt}=f_\theta(u,t), \qquad t\in[0,T].
\end{equation}
We denote by $\Phi(t,s)$ the associated flow map that moves a state from time $s$ to $t$ along (\ref{eq:fm_ode}), i.e., $\Phi(t,s)(u_s)=u_t$.
For a base distribution $\mu$ at $t=0$, the induced marginal distribution at time $t$ is
\[
\nu_t\coloneqq\Phi(t,0)_\#\mu,
\]
where $(\cdot)_\#$ denotes the pushforward of a distribution through the flow.
The terminal distribution $\nu_T=\Phi(T,0)_\#\mu$ is assumed to approximate the target data distribution.

\paragraph{Constrained Generation.}
Given a pretrained flow matching model with terminal distribution $\nu_T$ and a differentiable constraint operator $H$, we aim to generate samples $u$ satisfying
\begin{equation}\label{eq:constrained_sampling}
u\sim\nu_T, \qquad H(u)=0.
\end{equation}
As the pretrained velocity field $f_\theta$ is learned to transport the base distribution $\mu$ toward the unconstrained terminal distribution $\nu_T$, it is not explicitly trained to enforce the constraint $H(u)=0$. As such, we seek a constrained generation procedure that generates samples from the terminal distribution $\nu_T$ that satisfy the prescribed constraint using only the pretrained flow field $f_\theta$ and the constraint operator $H$ without additional model training.

The key challenge is to simultaneously satisfy the constraint $H(u)=0$ while 
keeping the samples close to the pretrained distribution $\nu_T$.
Existing approaches often enforce the constraint at the cost of deviating from $\nu_T$ (see Figure~\ref{fig:intro}). To address this, we formulate constrained generation as a \emph{minimal-intervention problem}, seeking the smallest perturbation to the intermediate states of the pretrained flow trajectory necessary to enforce the constraint while maintaining the pretrained distribution.
\section{Method}\label{sec:method}
\subsection{Trajectory Intervention for Constrained Generation}

Rather than applying post-hoc projection or modifying the pretrained flow, we enforce the constraint through a minimal intervention on the generated trajectory. Given a trajectory $\{u_t\}_{t\in[0,T]}$, we seek an intermediate time $s\in[0,T)$ and a perturbation $\delta$ such that
\[
H\left(\Phi(T,s)(u_s+\delta)\right)=0.
\]
Thus, the perturbed state $u_s+\delta$ is propagated to the terminal time using the original pretrained flow.
For a fixed intervention time $s$, we find the smallest perturbation satisfying the terminal constraint:
\begin{tcolorbox}[
    colback=gray!8,
    colframe=gray!8,
    boxrule=0pt,
    arc=2pt,
    left=4pt,
    right=4pt,
    top=4pt,
    bottom=4pt
]
\begin{equation}\label{eq:minimal_correction_problem}
\min_{\delta}\  \frac{1}{2}\left\Vert\delta\right\Vert^2, \qquad \text{subject to }\  F_s(\delta)\coloneqq H\left(\Phi(T,s)(u_s+\delta)\right)=0.
\end{equation}
\end{tcolorbox}
This formulation casts constrained generation as a \textbf{minimal trajectory intervention}. The flow dynamics remain unchanged, while only the intermediate state is minimally perturbed to reach the constraint manifold at $t=T$.

\begin{figure}[t]
\centering
\captionsetup[subfigure]{font=scriptsize, labelformat=parens, labelsep=space}

\begin{subfigure}[t]{0.24\linewidth}
    \centering
    \includegraphics[width=\linewidth]{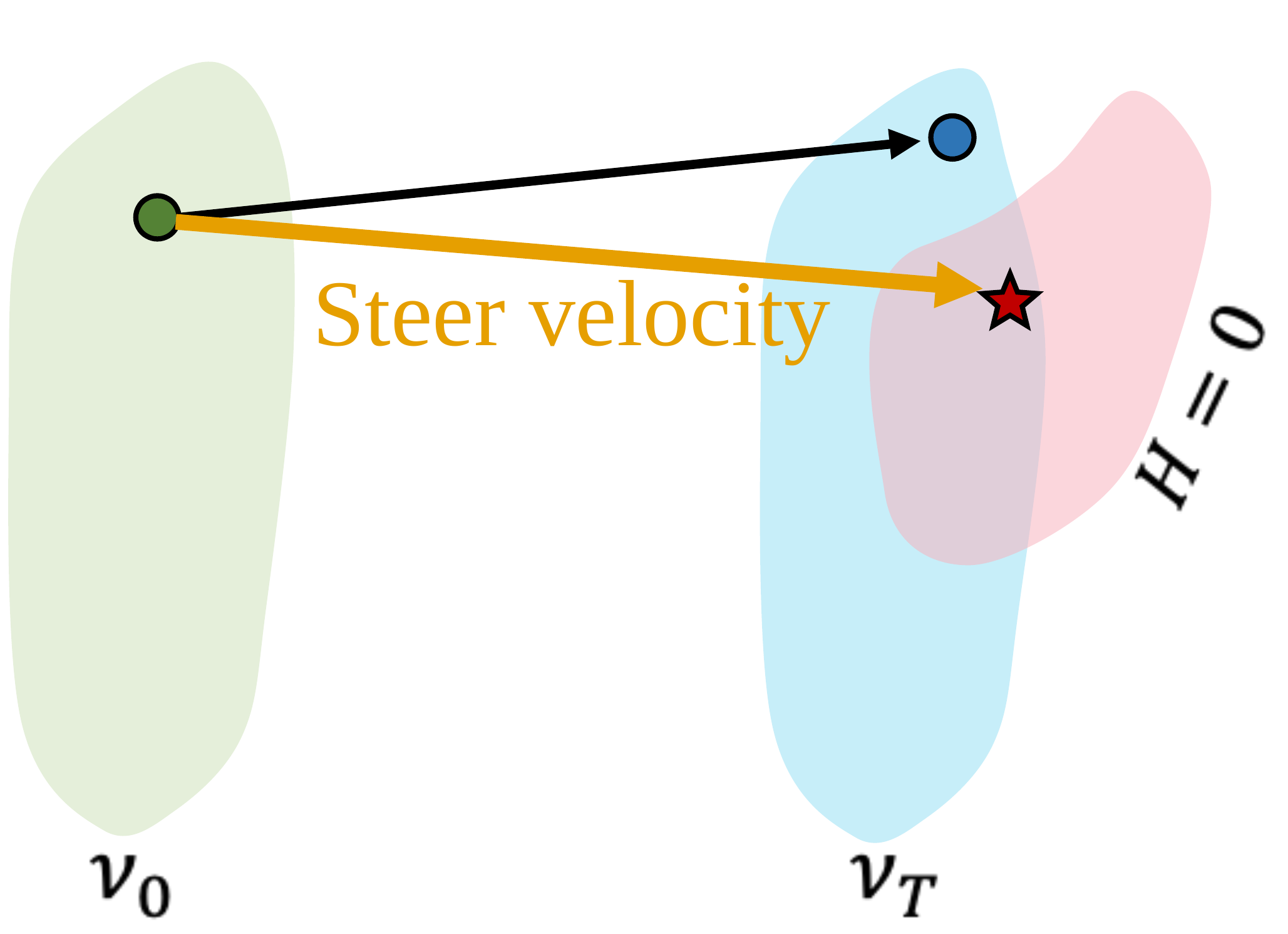}
    \caption{Steer Velocity Field}
    \label{fig:steer_velocity}
\end{subfigure}
\hfill
\begin{subfigure}[t]{0.24\linewidth}
    \centering
    \includegraphics[width=\linewidth]{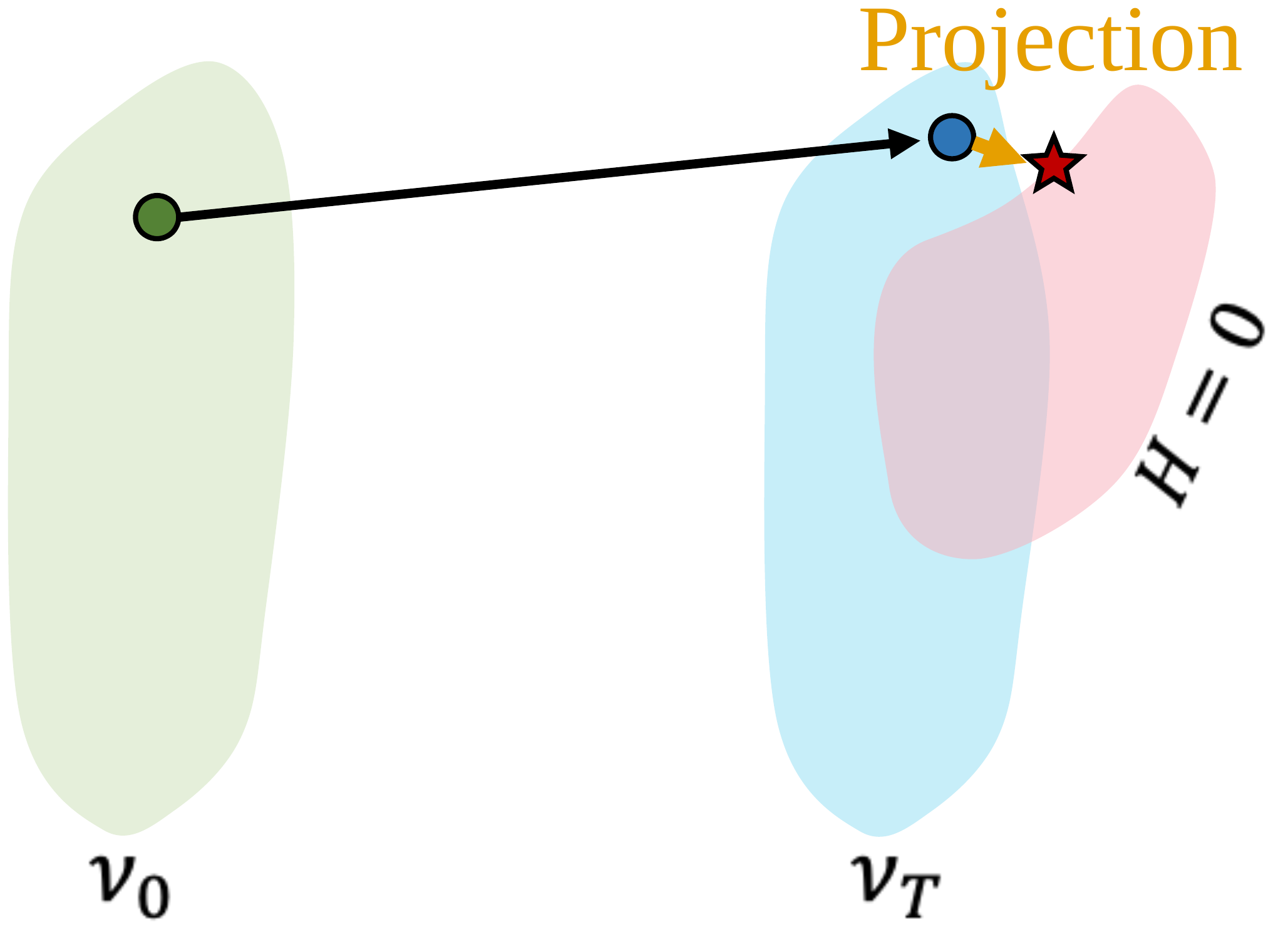}
    \caption{Terminal Projection}
    \label{fig:projection}
\end{subfigure}
\hfill
\begin{subfigure}[t]{0.24\linewidth}
    \centering
    \includegraphics[width=\linewidth]{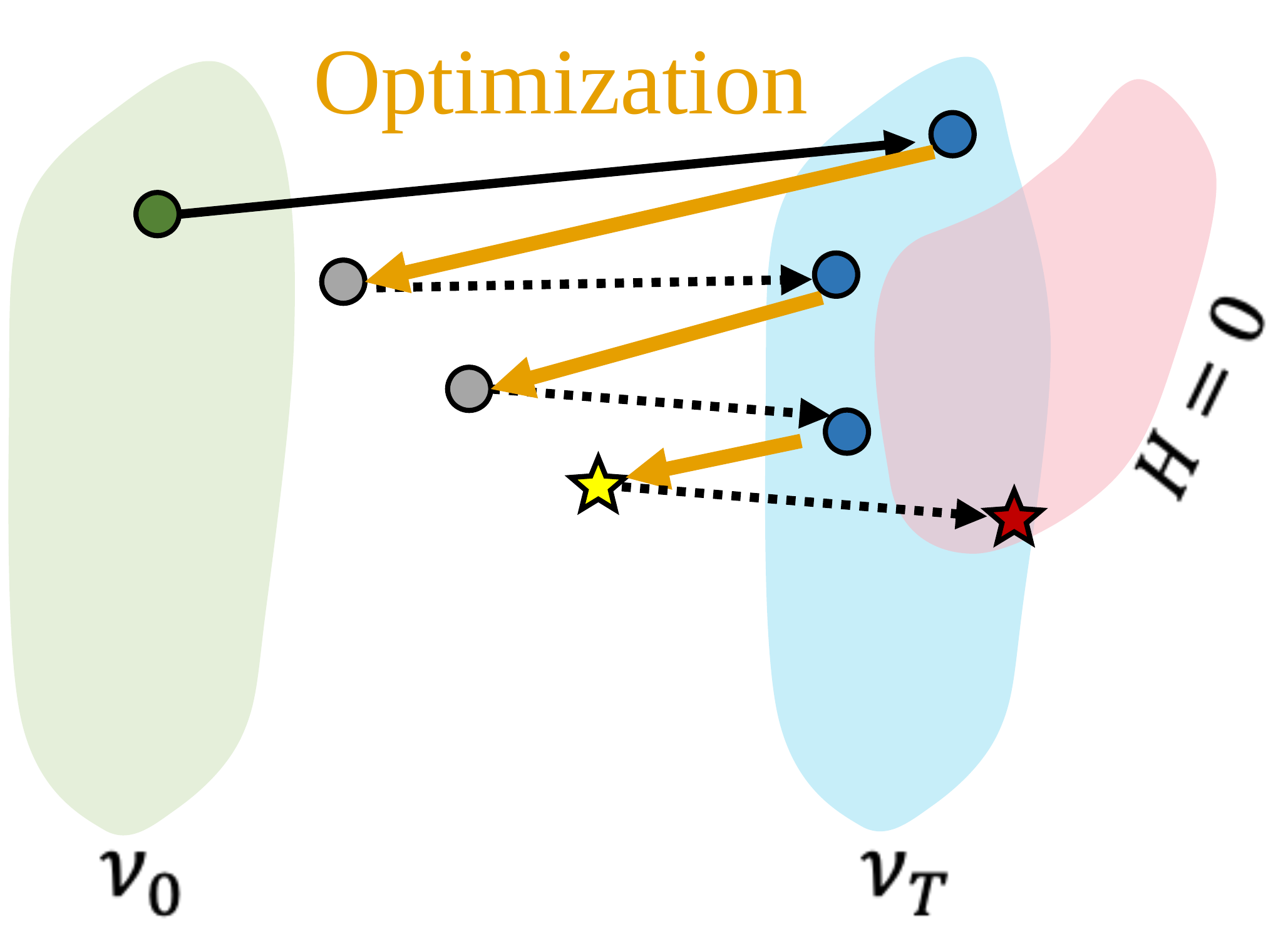}
    \caption{Trajectory Optimization}
    \label{fig:optimization}
\end{subfigure}
\hfill
\begin{subfigure}[t]{0.24\linewidth}
    \centering
    \includegraphics[width=\linewidth]{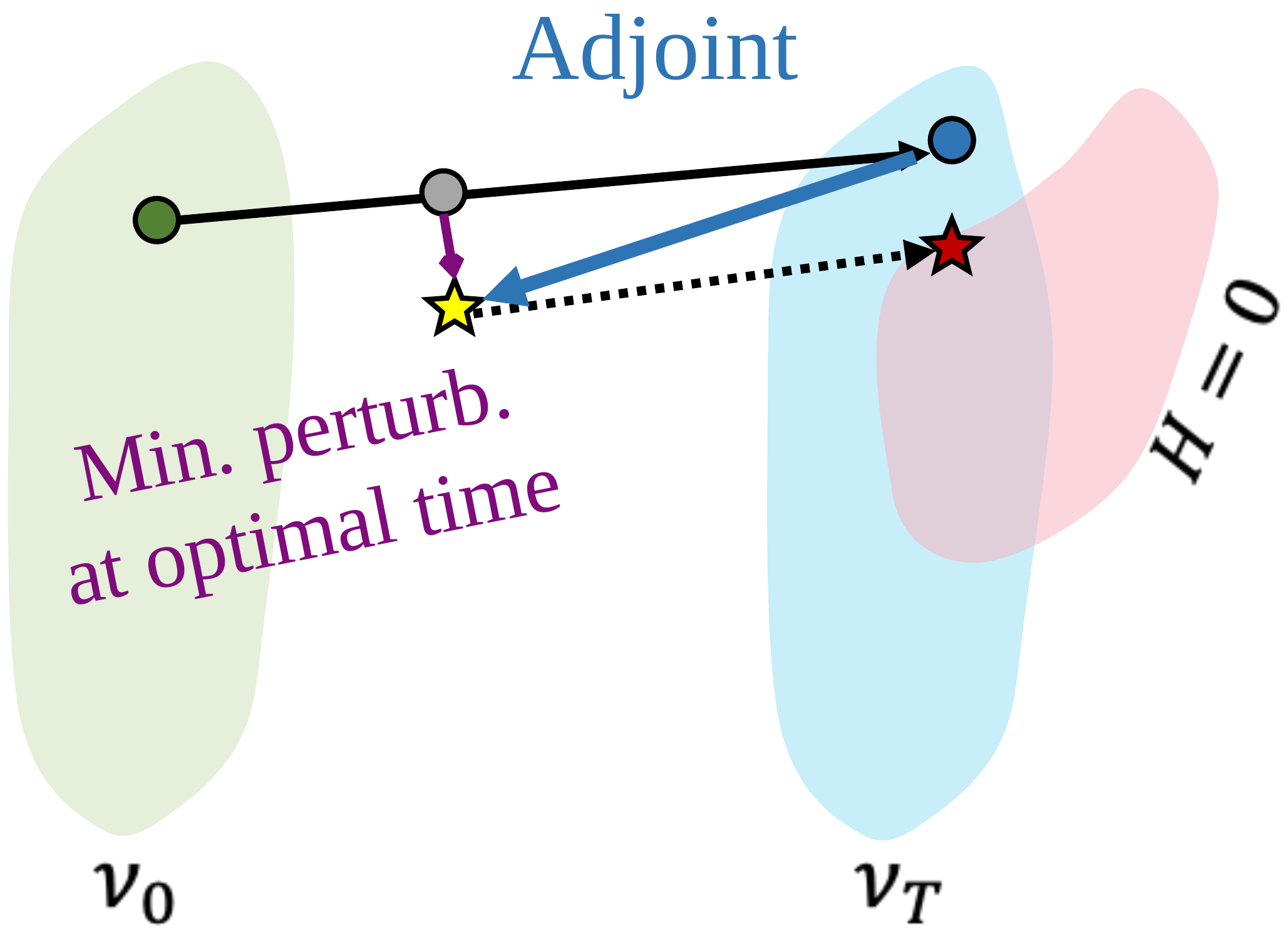}
    \caption{Trajectory Intervention (Ours)}
    \label{fig:adjoint}
\end{subfigure}

\caption{
Comparison of constrained generation strategies.
\textbf{Steer Velocity Field} modifies the flow by learning an additional velocity field.
\textbf{Terminal Projection} projects the generated sample onto the constraint manifold.
\textbf{Trajectory Optimization} iteratively optimizes states along the flow.
In contrast, our \textbf{Trajectory Intervention} minimally perturbs the trajectory at an intermediate time then follows the original flow.
}
\label{fig:comparison_correction_methods}
\end{figure}

\subsection{Adjoint-Based Optimal Correction}\label{sec:adjoint}
We solve (\ref{eq:minimal_correction_problem}) by computing the minimum-norm perturbation via adjoint sensitivity of the flow map.
For a generated trajectory $\{u_t\}_{t\in[0,T]}$, consider a perturbation at an intermediate time $s$:
\[
u_s'=u_s+\delta .
\]
For a sufficiently small $\delta$, the terminal constraint can be linearized around the original trajectory as
\begin{equation}\label{eq:linearization}
H(u')=F_s(\delta)
\approx
H(u_T) + \nabla H(u_T) D\Phi(T,s)(u_s)\delta,\quad D\Phi(T,s)\coloneqq \frac{\partial\Phi(T,s)(u)}{\partial u}(u).
\end{equation}
Then, enforcing $F_s(\delta)=0$ under the linear approximation gives
\begin{equation}\label{eq:As}
A_s\delta=-H(u_T),\quad A_s \coloneqq \nabla H(u_T)  D\Phi(T,s)(u_s)
\end{equation}
whose minimum-norm solution is
\begin{tcolorbox}[
    colback=gray!8,
    colframe=gray!8,
    boxrule=0pt,
    arc=2pt,
    left=4pt,
    right=4pt,
    top=4pt,
    bottom=4pt
]
\begin{equation}\label{eq:optimal_delta}
\delta_s^\ast=-A_s^\dagger H(u_T),
\end{equation}
\end{tcolorbox}
\noindent where \(A_s^\dagger\) denotes the Moore--Penrose pseudoinverse, which can be computed by solving the associated least-squares problem. In particular, when $H$ is scalar-valued, (\ref{eq:optimal_delta}) reduces to
\[
\delta_s^\ast = -\frac{A_s^\top}{\|A_s\|^2}\,H(u_T),\quad A_s\neq 0.
\]
\vspace{-1em}

Consequently, for each intervention time $s$, the minimum-norm perturbation is determined under the linear approximation. This correction can also be interpreted as a single \emph{Newton step} for solving the nonlinear equation $F_s(\delta)=0$. Appendix~\ref{appen:control_perspective} also provides an \emph{optimal control} interpretation.

\paragraph{Adjoint Computation.}
The computation of $A_s$ in in (\ref{eq:As}) requires the sensitivity of the terminal state with respect to the intermediate state.
Directly computing the full Jacobian $D\Phi(T,s)$ is computationally prohibitive for high-dimensional systems. Instead, we compute its vector-Jacobian product through the adjoint equation.

Specifically, we define the adjoint variable $\lambda(t)$ as the solution of the \emph{adjoint equation}
\begin{equation}\label{eq:adjoint_eq}
\dot{\lambda}(t) = -\left(\frac{\partial f_\theta}{\partial u}(u_t,t)\right)^\top \lambda(t),\quad \lambda(T)=\nabla H(u_T)^\top.
\end{equation}
Solving this equation backward from $t=T$ to $t=s$ yields exactly the required vector-Jacobian product  \citep{coddington1955theory, pontryagin1987mathematical}
\[
\lambda(s) = D\Phi(T,s)^\top \nabla H(u_T)^\top = A_s^\top .
\]
Thus, $A_s$ is obtained without explicitly forming the flow Jacobian, which enables efficient computation of the optimal correction for high-dimensional flow matching models.

We provide theoretical guarantees for the proposed adjoint-based intermediate correction in terms of both constraint satisfaction and distributional preservation. 
\begin{theorem}[Informal]\label{thm:main_informal_correction}
Under mild regularity conditions on $f_\theta$ and $H$, the proposed linearized correction in (\ref{eq:optimal_delta}) achieves a second-order terminal constraint residual:
\[
\left\| H\left(\Phi(T,s)(u_s+\delta_s^\ast)\right) \right\| \leq C\|\delta_s^\ast\|^2.
\]
Moreover, if the flow field is $L_f$-Lipschitz in $u$, the corrected terminal distribution $\nu_T^\ast$ satisfies
\[
W_2(\nu_T^\ast,\nu_T) \leq e^{L_f(T-s)} \|\delta_s^\ast\|_{L^2(\nu_s)}.
\]
\end{theorem}
The formal statements and proofs are provided in Appendix~\ref{appen:linearized_correction}. The first result quantifies the constraint satisfaction of the proposed correction, while the second shows the corrected terminal distribution deviation from the original flow-matching distribution is bounded by both the magnitude of the intermediate correction and the remaining integration time $T-s$.

\subsection{Adaptive Intervention Time Selection}\label{sec:time_seclection}
The previous subsection determines the minimum-norm correction for a fixed intervention time $s$. We next select $s$ by balancing the correction magnitude and the remaining propagation time $T-s$. Theorem~\ref{thm:main_informal_correction} suggests that two factors should be controlled: the magnitude of the intervention and the remaining flow integration time. Because the Lipschitz constant $L_f$ is generally unknown, we instead consider the following practical surrogate:
\begin{tcolorbox}[
    colback=gray!8,
    colframe=gray!8,
    boxrule=0pt,
    arc=2pt,
    left=4pt,
    right=4pt,
    top=4pt,
    bottom=4pt
]
\begin{equation}\label{eq:optimal_time}
s^\ast = \arg\min_{s\in[0,T)} J(s)\coloneqq \|\delta_s^\ast\|^2 + \lambda (T-s)^2,
\end{equation}
\end{tcolorbox}
\noindent
where $\delta_s^\ast$ is given by (\ref{eq:optimal_delta}) and $\lambda$
controls the preference for later intervention.
The first term favors intervention times that require a small correction.
The second term favors later intervention by penalizing the remaining propagation time after the intervention. This reduces both the potential amplification of the perturbation through the remaining flow evolution and the numerical cost required to reach the terminal time.
Thus, $J(s)$ provides a practical criterion that balances distributional preservation and computational efficiency.

\begin{wrapfloat}{algorithm}{r}{0.54\linewidth}
\begin{algorithmic}[1]
\Require
$u_0\sim\mu$, $f_\theta$, $H$, $\{s_k\}_{k=1}^K$, $\lambda$

\State Generate trajectory $\{u_t\}_{t=0}^T$ and compute $H(u_T)$.

\State Solve the adjoint backward from $T$ and obtain sensitivities at $\{s_k\}_{k=1}^K$.

\State In parallel, compute $\delta_{s_k}^\ast$ via (\ref{eq:optimal_delta}) and evaluate $J(s_k)$ for all $s_k$.

\State Select $s^\ast=\arg\min_{s_k}J(s_k)$.

\State Re-integrate from $u_{s^\ast}+\delta_{s^\ast}^\ast$ to $T$.

\noindent\Return
$u_T^{\mathrm{corr}} = \Phi(T,s^\ast) \left(u_{s^\ast}+\delta_{s^\ast}^\ast\right)$
\end{algorithmic}

\caption{MintFlow Constrained Generation}
\label{alg:mintflow}
\end{wrapfloat}
In practice, we evaluate $J(s)$ on a finite set of candidate times $0\leq s_1<\cdots<s_K<T$:
\begin{equation}\label{eq:finite_time_optimization}
s^\ast = {\arg\min}_{k=1,\cdots,K} J(s_k).
\end{equation}

After generating the trajectory once, we integrate the adjoint equation backward to the earliest candidate time $s_1$, storing adjoint states at ${s_k}{k=1}^K$. We then compute $\delta{s_k}^\ast$ and $J(s_k)$ in parallel for all $k$ and select $s^\ast$ via (\ref{eq:finite_time_optimization}). We then apply
\[
u_{s^\ast}\leftarrow u_{s^\ast}+\delta_{s^\ast}^\ast
\]
and reintegrate the original flow from $s^\ast$ to $T$ to obtain the constrained sample.
The complete MintFlow algorithm is summarized in Algorithm~\ref{alg:mintflow}.

\paragraph{Iterative Refinement.}
The correction in (\ref{eq:optimal_delta}) is a single Newton step for $F_s(\delta)=0$. If the resulting constraint residual exceeds a prescribed tolerance, 
we iteratively apply Algorithm~\ref{alg:mintflow} on the corrected trajectory until the residual falls below the tolerance.
At iteration $i$, we compute the adjoint sensitivity, select an
intervention time $s^{(i)}$, apply $\delta_{s^{(i)}}^\ast$, and reintegrate the original flow to $T$. We restrict the intervention times to $s^{(1)} < s^{(2)} < \cdots < T$, so each subsequent iteration starts from the trajectory produced by the previous correction and requires a shorter backward adjoint integration. Algorithm~\ref{alg:iterative_adj} and its convergence analysis are provided in Appendix~\ref{appen:convergence_newton}.
\section{Experiments}
We evaluate \textbf{MintFlow} on three constrained generation settings: \textit{(i)} image inverse problems, \textit{(ii)} image editing, and \textit{(iii)} physical system modeling. The experiments span different flow models, data domains, and constraint types, ranging from local to global and from linear to nonlinear, demonstrating the broad applicability of our approach.  We compare MintFlow against existing training-free constrained generation methods, using the same pretrained flow for all methods.\footnote{All image experiments were conducted on a single NVIDIA H200 GPU, and all physical system experiments were conducted on a single NVIDIA A100 GPU.} We assess each method along three core dimensions, \textit{constraint satisfaction}, \textit{preservation of the pretrained distribution}, and \textit{computational efficiency}, supplemented by task-specific performance metrics.
The code is publicly available at \url{https://github.com/Yebbi/MintFlow}.

\subsection{Image Inverse Problems}\label{sec:exp_image_inverse}
\paragraph{Setup.} 
Following \cite{patel2024flowchef}, we evaluate MintFlow on noise-free image inverse problems, including inpainting, super resolution, and deblurring, using the AFHQ-Cat and FFHQ datasets. We use the same pretrained Rectified Flow++ (RF) models~\citep{lee2024improvingtrainingrectifiedflows} across all experiments, and resize all images to $64\times64$.
For each task, given a degraded observation $u^{\mathrm{obs}}$, we define the measurement constraint as $H(u)=A(u)-u^{\mathrm{obs}}=0$, where $A$ denotes the task-specific degradation operator. We compare MintFlow with existing image-constraint methods—including D-Flow, DPS, Freedom, PNPFlow, and FlowChef—while sharing the same frozen RF model and source noises. In addition to constraint satisfaction, distribution preservation, and efficiency, we report reconstruction quality using PSNR, SSIM, and LPIPS. For further details, please refer to Appendix \ref{appen:inverse_problems}.

\paragraph{Quantitative Results.}
Table \ref{tab:RF_inverse_results} summarizes the quantitative results. Across all tasks and datasets, MintFlow consistently provides a substantially better balance between constraint satisfaction and reconstruction quality than existing methods. In particular, it achieves the lowest constraint error and Fréchet Inception Distance (FID) score across all evaluated tasks. This indicates that MintFlow can closely satisfy the measurement constraints without requiring excessive modification of the pretrained generative prior. Importantly, the improved constraint satisfaction does not come at the expense of perceptual or reconstruction quality. MintFlow achieves the best FID, PSNR, SSIM, and LPIPS across the evaluated settings, demonstrating its ability to simultaneously preserve the learned data distribution and recover high-quality image details. Furthermore, MintFlow shows computational efficiency of a minimal-intervention approach. In contrast, while competing methods can reduce the constraint error, they often produce reconstructions that deviate from the learned data distribution, resulting in higher FID and inferior reconstruction quality (Figure~\ref{fig:image_inverse_main}).

\begin{table}[t]
\centering
\caption{Quantitative comparison of methods for image inverse problems on AFHQ-Cat and FFHQ. We report constraint error $H$, FID, PSNR (dB), SSIM, LPIPS, inference time (s), and peak GPU memory (MB). Light green (\colorbox{green!15}{green})  indicates the top performer, while light red (\colorbox{red!10}{red}) represents the second-best performer across all methods for each metric (applies to all tables herein).}
\label{tab:RF_inverse_results}
\resizebox{\textwidth}{!}{
\begin{tabular}{lccccc|ccccc|cc}
\toprule
& \multicolumn{5}{c|}{\textbf{AFHQ-Cat}} & \multicolumn{5}{c|}{\textbf{FFHQ}} & \multicolumn{2}{c}{\textbf{Efficiency}}\\
\textbf{Method} 

& \textbf{$H \downarrow$} 
& \textbf{FID $\downarrow$}
& \textbf{PSNR $\uparrow$} 
& \textbf{SSIM $\uparrow$}
& \textbf{LPIPS $\downarrow$}

& \textbf{$H \downarrow$} 
& \textbf{FID $\downarrow$}
& \textbf{PSNR $\uparrow$} 
& \textbf{SSIM $\uparrow$}
& \textbf{LPIPS $\downarrow$} 
& \textbf{Time $\downarrow$} 
& \textbf{Memory $\downarrow$} \\
\midrule

\multicolumn{13}{l}{\textbf{\textit{Task: Inpainting}}}
\\

D-Flow
& \cellcolor{red!10}1.775E-2 & 23.662 & 20.8597 & \cellcolor{red!10}0.7180 & \cellcolor{red!10}0.1578 & \cellcolor{red!10}1.592E-2 & 29.692 & \cellcolor{red!10}22.7378 & \cellcolor{red!10}0.7910 & \cellcolor{red!10}0.0915 & 153.47 & 3597.02 \\

DPS
& 2.583E-2 & \cellcolor{red!10}14.532 & 20.9537 & 0.6232 & 0.1806 & 2.897E-2 & 33.498 & 21.3944 & 0.6913 & 0.1146 & 6.61 & 646.74 \\

Freedom
& 6.013E-2 & 19.999 & 17.8908 & 0.4377 & 0.2502 & 6.789E-2 & 39.998 &  18.0828 & 0.5334 & 0.1605 &4.60 & 637.36\\

PNPFlow
& 1.926E-2 & 24.063 & \cellcolor{red!10}22.2547 & 0.6877 & 0.1897 &  2.126E-2 & 63.579 & 22.3143 & 0.7388 & 0.1281 & 5.56 & 338.28\\

FlowChef
& 2.476E-2 & 18.292 & 19.4181 & 0.6110 &  0.1824 & 2.572E-2 & \cellcolor{red!10}17.226 & 20.5481 & 0.6934 & 0.1028 & \cellcolor{red!10}0.21 & \cellcolor{red!10}330.74 \\

\textbf{MintFlow}
& \cellcolor{green!15}5.498E-4 & \cellcolor{green!15}6.506 & \cellcolor{green!15}26.3913 & \cellcolor{green!15}0.9207 & \cellcolor{green!15}0.0376 & \cellcolor{green!15}5.172E-4 & \cellcolor{green!15}6.046 & \cellcolor{green!15}28.7002 & \cellcolor{green!15}0.9417 & \cellcolor{green!15}0.0207 & \cellcolor{green!15}0.19 & \cellcolor{green!15}329.99
\\
\midrule

\multicolumn{13}{l}{\textbf{\textit{Task: Super Resolution}}}
\\
D-Flow
& 1.371E-2 & 27.351 & 19.9141 & 0.5397 & 0.2309 & 1.347E-2  & 65.788 & 20.2276 & 0.6187 & 0.1565 & 188.96 & 3596.63\\

DPS
& 1.723E-2 & \cellcolor{red!10}12.771 & 19.9188 & 0.5220 & \cellcolor{red!10}0.1686 & 1.838E-2 & \cellcolor{red!10}9.608 & 19.7618 & 0.5924 & \cellcolor{red!10}0.1173 & 6.42 & 646.69\\

Freedom
& 4.814E-2 & 19.048 & 17.5691 & 0.3816 & 0.2452 & 5.124E-2 & 25.978 & 17.3415 & 0.4694 & 0.1568 & 4.53 & 637.32\\

PNPFlow
& \cellcolor{red!10}1.193E-2 & 52.9554 & \cellcolor{red!10}21.6583 & \cellcolor{red!10}0.6187 & 0.2289 &  \cellcolor{red!10}1.342E-2 & 79.951 & \cellcolor{red!10}21.1512 & \cellcolor{red!10}0.6670 & 0.1649 & 4.9 & 338.24 \\

FlowChef
& 1.842E-2 & 23.601 & 18.9815 & 0.4435 & 0.2352 & 1.950E-2 & 18.998 & 18.7612 &  0.5141 & 0.1381 & \cellcolor{red!10}0.25 & \cellcolor{red!10}330.69\\

\textbf{MintFlow}
& \cellcolor{green!15}9.901E-5 & \cellcolor{green!15}8.296 & \cellcolor{green!15}22.9958 & \cellcolor{green!15}0.7147 & \cellcolor{green!15}0.0757 &  \cellcolor{green!15}1.248E-4 & \cellcolor{green!15}8.313 & \cellcolor{green!15}22.4535 & \cellcolor{green!15}0.7468 & \cellcolor{green!15}0.0543 & \cellcolor{green!15}0.22 & \cellcolor{green!15}329.86\\

\midrule

\multicolumn{13}{l}{\textbf{\textit{Task: Deblurring}}}
\\

D-Flow
& \cellcolor{red!10}9.015E-3 & 47.8724 & 18.2846 & 0.3734 & 0.3025 & \cellcolor{red!10}9.850E-3 & 81.818 & 17.9925 & 0.4277 & 0.1528 & 187.30 & 3596.97 \\

DPS
& 1.950E-2 & \cellcolor{red!10}12.5016 & 20.5182 & 0.5629 & \cellcolor{red!10}0.1694 & 2.018E-2 & 23.919 & 20.2454 & 0.6272 & \cellcolor{red!10}0.1145 & 6.23 & 646.69\\

Freedom
& 5.174E-2 & 18.962 & 17.6911 & 0.3983 & 0.2384 & 5.191E-2 & 28.108 & 17.5191 & 0.4886 & 0.1566 & 4.11 & 637.32\\

PNPFlow
& 1.211E-2 & 25.170& \cellcolor{red!10}22.3216 & \cellcolor{red!10}0.6711 & 0.2005 & 1.348E-2 & 73.589 & \cellcolor{red!10}21.8534 & \cellcolor{red!10}0.7116 & 0.1463 & 4.43 & 338.24\\

FlowChef
& 2.365E-2 & 18.548 & 19.2998 & 0.4924 & 0.2237 & 2.543E-2 & \cellcolor{red!10}19.195 & 19.0493 & 0.5593 & 0.1334 & \cellcolor{red!10}0.23 & \cellcolor{red!10}330.69\\

\textbf{MintFlow}
& \cellcolor{green!15}2.877E-4 & \cellcolor{green!15}6.819 & \cellcolor{green!15}28.3760 & \cellcolor{green!15}0.8988 & \cellcolor{green!15}0.03445 & \cellcolor{green!15}2.916E-4 & \cellcolor{green!15}6.417 &  \cellcolor{green!15}28.4017 & \cellcolor{green!15}0.9221 & \cellcolor{green!15}0.0237 & \cellcolor{green!15}0.22 & \cellcolor{green!15}329.94\\

\bottomrule
\end{tabular}
}
\end{table}

\paragraph{Qualitative Comparison.}
Figure~\ref{fig:image_inverse_main} presents representative qualitative results with additional uncurated results provided in Appendix~\ref{appen:additional_images_inverse}. These results corroborate the quantitative findings featured in Table~\ref{tab:RF_inverse_results}. MintFlow produces natural-looking reconstructions while faithfully preserving the measurement constraints. In contrast, baselines often produce blurry reconstructions or introduce visually implausible structures and artifacts, indicating deviations from the learned image distribution. Overall, the qualitative results demonstrate that MintFlow more effectively satisfies the measurement constraints while remaining faithful to the learned image distribution.

\begin{figure}[h]
    \centering
    \resizebox{\textwidth}{!}{%
    \begin{tikzpicture}

        \node[inner sep=0pt, anchor=south west] (inpaint) at (0,0)
            {\includegraphics[
                width=\textwidth
            ]{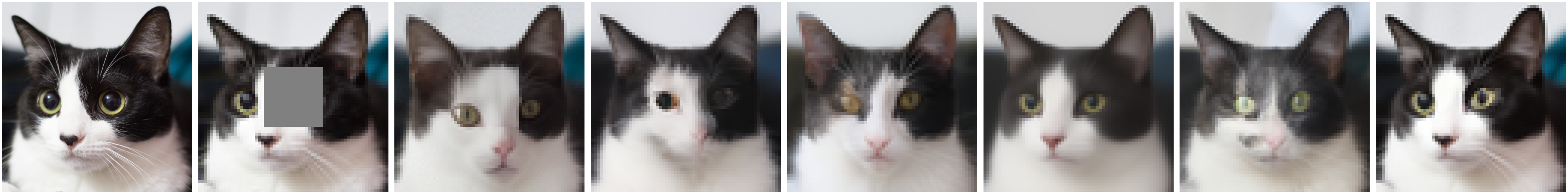}};

        \node[inner sep=0pt, anchor=north west] (sr) at
            (inpaint.south west)
            {\includegraphics[
                width=\textwidth
            ]{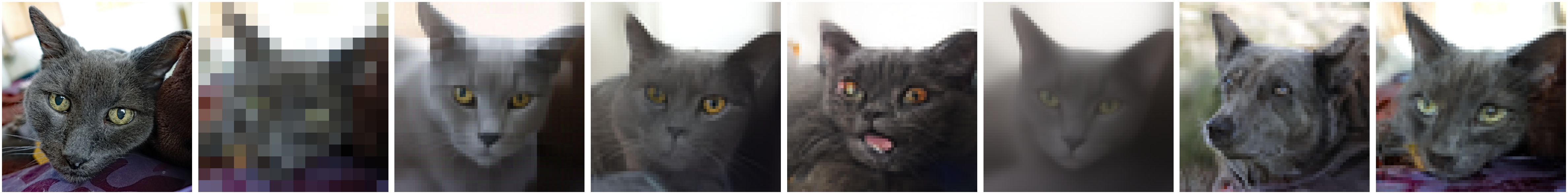}};

        \node[inner sep=0pt, anchor=north west] (deblur) at
            (sr.south west)
            {\includegraphics[
                width=\textwidth
            ]{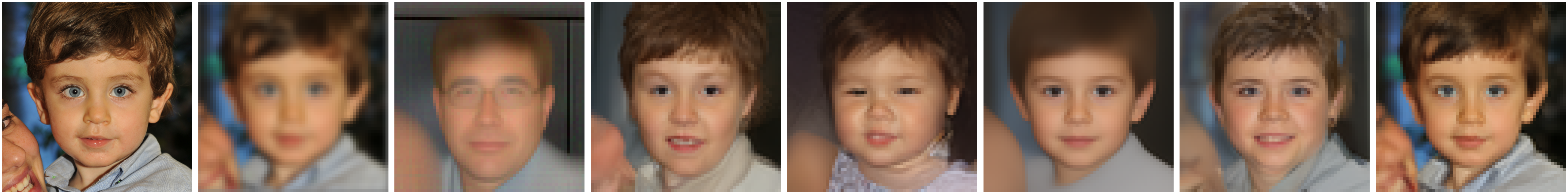}};

        \node[anchor=south] at
            ($(inpaint.north west)!1/16!(inpaint.north east)$)
            {GT};

        \node[anchor=south] at
            ($(inpaint.north west)!3/16!(inpaint.north east)$)
            {Degraded};

        \node[anchor=south] at
            ($(inpaint.north west)!5/16!(inpaint.north east)$)
            {D-Flow};

        \node[anchor=south] at
            ($(inpaint.north west)!7/16!(inpaint.north east)$)
            {DPS};

        \node[anchor=south] at
            ($(inpaint.north west)!9/16!(inpaint.north east)$)
            {Freedom};

        \node[anchor=south] at
            ($(inpaint.north west)!11/16!(inpaint.north east)$)
            {PnPFlow};

        \node[anchor=south] at
            ($(inpaint.north west)!13/16!(inpaint.north east)$)
            {FlowChef};

        \node[anchor=south] at
            ($(inpaint.north west)!15/16!(inpaint.north east)$)
            {\textbf{MintFlow}};

        \node[
            rotate=90,
            anchor=south,
        ] at
            ($(inpaint.west)+(-0.015\textwidth,0.5*0)$)
            {Inpainting};

        \node[
    rotate=90,
    anchor=south,
] at
    ($(sr.west)+(-0.005\textwidth,0)$)
    {\shortstack[c]{Super\\Resolution}};
    
        \node[
            rotate=90,
            anchor=south,
        ] at
            ($(deblur.west)+(-0.015\textwidth,0)$)
            {Deblurring};

    \end{tikzpicture}}

    \caption{Representative qualitative results for image inverse problems. We compare ground-truth images (GT), degraded inputs, and results from different methods. Inpainting and super resolution are evaluated on AFHQ-Cat, while deblurring is evaluated on FFHQ.}
    \label{fig:image_inverse_main}
\end{figure}
\subsection{Image Editing}
\label{sec:exp_image_editing}
\paragraph{Setup.}
Following \cite{patel2024flowchef}, we conduct image editing experiments on PIE-Bench \citep{ju2024pnp} using InstaFlow \citep{liu2023instaflow} as the flow-matching model. Given a reference image, mask, and text guidance, the task requires preserving the unmasked regions while editing the masked region according to the provided text prompt. Thus, the primary challenge is to simultaneously satisfy the preservation constraint and the text-guided editing objective. We compare against the same baselines as in Section~\ref{sec:exp_image_inverse}, while utilizing the identical model and source noises. D-Flow is omitted due to its excessive GPU memory requirements. In addition to constraint satisfaction and efficiency, we evaluate prompt adherence and image similarity using CLIP scores, measuring how well the edited image follows the text prompt and remains visually consistent with the reference image, respectively. Additional details are provided in Appendix \ref{appen:image_editing}.
\begin{figure}[b!]
    \centering

    \begin{tikzpicture}


        \node[inner sep=0] (img1) {
            \includegraphics[
                width=\linewidth
            ]{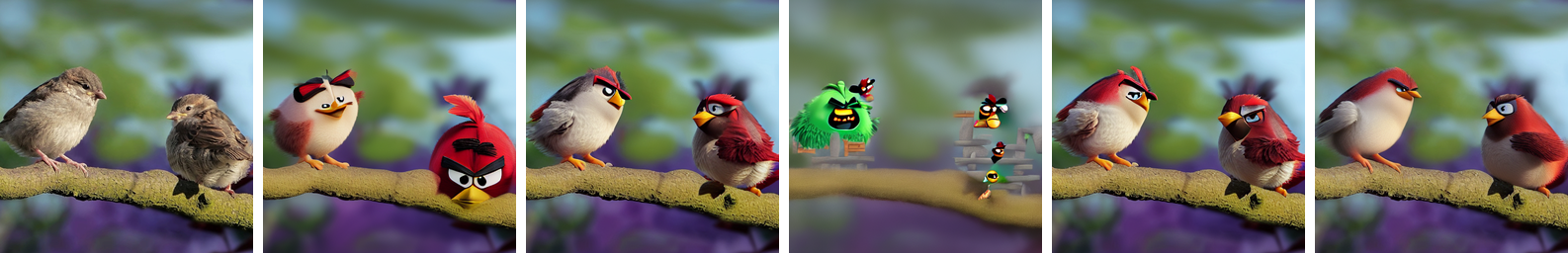}
        };

        \node[
            font=\small
        ] at ($(img1.north west)!1/12!(img1.north east) + (0, 7pt)$)
        {Input};

        \node[
            font=\small
        ] at ($(img1.north west)!3/12!(img1.north east) + (0, 7pt)$)
        {DPS};

        \node[
            font=\small
        ] at ($(img1.north west)!5/12!(img1.north east) + (0, 7pt)$)
        {Freedom};

        \node[
            font=\small
        ] at ($(img1.north west)!7/12!(img1.north east) + (0, 7pt)$)
        {PnPFlow};

        \node[
            font=\small
        ] at ($(img1.north west)!9/12!(img1.north east) + (0, 7pt)$)
        {FlowChef};

        \node[
            font=\small\bfseries
        ] at ($(img1.north west)!11/12!(img1.north east) + (0, 7pt)$)
        {MintFlow};


        \node[
            font=\small,
            anchor=north
        ] (text1) at ($(img1.south) + (0,-3pt)$)
        {
            \tikz[baseline=-1.5ex]{
        \node[
            fill=gray!15,
            rounded corners=2pt,
            inner xsep=3pt,
            inner ysep=1pt
        ] {\textcolor{gray}{\textit{Prompt:}}};
    }
    \;
            \texttt{two \sout{birds} \textcolor{red}{angry birds} sitting on a branch}
        };


        \node[
            inner sep=0,
            anchor=north
        ] (img2) at ($(text1.south) + (0,-3pt)$) {
            \includegraphics[
                width=\linewidth
            ]{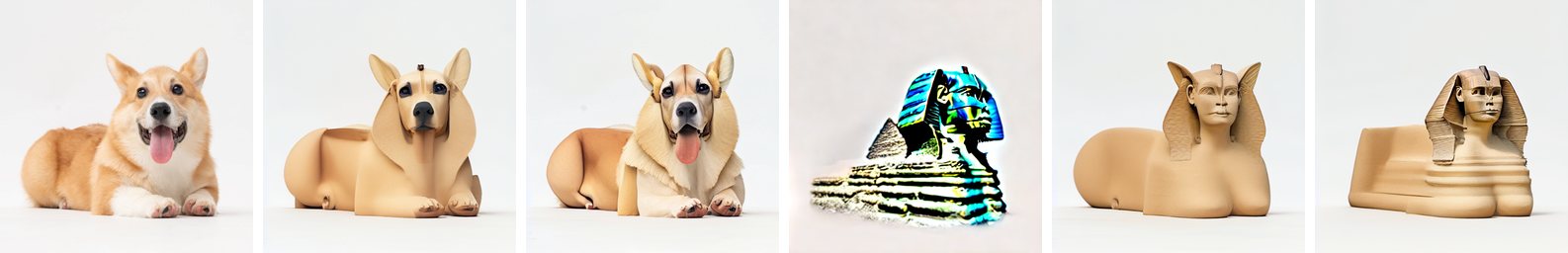}
        };


        \node[
            font=\small,
            anchor=north
        ] (text2) at ($(img2.south) + (0,-3pt)$)
     {
    \tikz[baseline=-1.5ex]{
        \node[
            fill=gray!15,
            rounded corners=2pt,
            inner xsep=3pt,
            inner ysep=1pt
        ] {\textcolor{gray}{\textit{Prompt:}}};
    }
    \;
    \texttt{a \sout{sitting dog} \textcolor{red}{Sphinx} on a gray background}
};


        \node[
            inner sep=0,
            anchor=north
        ] (img3) at ($(text2.south) + (0,-3pt)$) {
            \includegraphics[
                width=\linewidth
            ]{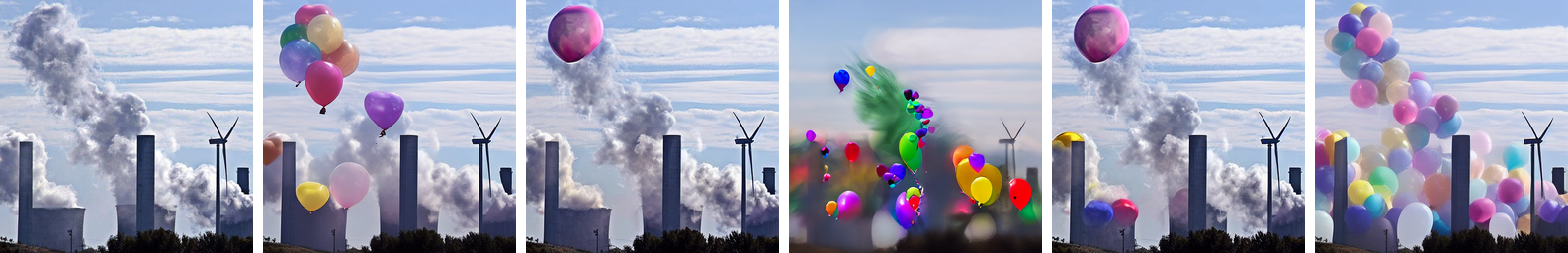}
        };


        \node[
            font=\small,
            anchor=north
        ] (text3) at ($(img3.south) + (0,-3pt)$)
     {
    \tikz[baseline=-1.5ex]{
        \node[
            fill=gray!15,
            rounded corners=2pt,
            inner xsep=3pt,
            inner ysep=1pt
        ] {\textcolor{gray}{\textit{Prompt:}}};
    }
    \;
    \texttt{power plant with \sout{smoke} \textcolor{red}{a mass of colorful party balloons}}
};
    \end{tikzpicture}
\caption{
Qualitative comparison of image editing results. Given an input image and its corresponding text prompt, we compare the editing results produced by different methods. Strikethrough and red text indicate the original and edited text, respectively.}
\label{fig:image_edit_main}
\end{figure}

\paragraph{Quantitative Results.}
\begin{wraptable}{r}{0.55\textwidth}
\vspace{-10pt}
\centering
\caption{Comparison on image editing task.}
\label{tab:image_edit_results}
\resizebox{0.55\textwidth}{!}{
\begin{tabular}{lccccc}
\toprule
Method & H $\downarrow$ & CLIP-T $\uparrow$ & CLIP-I $\uparrow$ & Time $\downarrow$ & Memory $\downarrow$ \\
\midrule
DPS & 0.0115  & 0.2916 & 0.7911 & \cellcolor{red!10}9.17 & 8881.75\\
Freedom & \cellcolor{red!10}0.0071 & 0.2756 & \cellcolor{green!15}0.8466 & 17.65 & 8122.47\\
PNPFlow & 0.0503 & \cellcolor{red!10}0.2918 & 0.6897 & 42.45 & \cellcolor{red!10}6507.95\\
FlowChef & 0.0136 & 0.2891 & 0.7939 & 9.29 & \cellcolor{green!15}6507.26\\
\textbf{MintFlow} & \cellcolor{green!15}0.0069 & \cellcolor{green!15}0.2926 & \cellcolor{red!10}0.8151 & \cellcolor{green!15}8.79 & \cellcolor{green!15}6507.26\\
\bottomrule
\end{tabular}
}
\vspace{-8pt}
\end{wraptable}
Table~\ref{tab:image_edit_results} summarizes the quantitative comparison on the image editing task. MintFlow demonstrates robust overall performance across the evaluated metrics. 
In particular, it attains the best constraint satisfaction, indicating it more faithfully preserves the unmasked regions while performing the desired edit. It also achieves the highest prompt adherence, demonstrating that the resulting images better follow the given text guidance. Although Freedom achieves higher image similarity, it comes at the cost of lower prompt adherence, indicating a stronger tendency to preserve the reference image. As shown in Figure \ref{fig:image_edit_main}, semantic elements that conflict with the prompt, such as the dog's ears, tongue, and power-plant smoke, are retained. MintFlow, in contrast, provides a more favorable balance between preserving the reference image and following the editing instruction.
MintFlow also is the most computationally efficient method, achieving the lowest inference time while using the least memory. This demonstrates the scalability and efficiency of MintFlow even for text-guided image editing. Overall, these results show that MintFlow can simultaneously satisfy the preservation constraint, adhere closely to the editing prompt, and perform the task efficiently, providing a favorable balance between editing quality and computational cost.

\paragraph{Qualitative Comparison.}
Representative qualitative results are presented in Figure \ref{fig:image_edit_main}, with additional results in Appendix \ref{appen:additional_images_editing}. As shown, baseline methods often struggle to maintain structural and visual coherence, producing artifacts such as prompt-inconsistent semantic elements retained from the reference image (e.g., untransformed ears), anatomical inconsistencies (e.g., misplaced eyes), and unnatural color shifts or blurring. In contrast, our method performs the desired edits in accordance with the given prompt while preserving visual coherence and image quality. These results demonstrate that our approach effectively satisfies the prompt constraints while preserving the contextual information of the input image without straying outside of the learned data distribution.

\subsection{Physical System Modeling}
\label{sec:physical_systems}


\paragraph{Setup.}
Following the benchmark protocol of~\cite{utkarsh2026physics}, we evaluate MintFlow on physical-system generation using \emph{Functional Flow Matching} (FFM), which extends flow matching to function-valued data~\citep{kerrigan2023functional}. We consider two spatiotemporal systems: the periodic heat and a one-dimensional reaction--diffusion (RD1D) equations. Each system uses its own pretrained FFM checkpoint, and all samplers share the same frozen weights and source noises. The constraint operator jointly encodes conservation, initial-condition (IC), and boundary-condition (BC) residuals. Heat uses structural periodicity, while RD1D imposes prescribed Neumann fluxes. We compare MintFlow against Vanilla FFM, DiffusionPDE, D-Flow, ECI, and PCFM without retraining or fine-tuning the pretrained model. In addition to the introduced core criteria, we evaluate physical consistency using the governing-equation residual and the error to a numerical trajectory initialized from each generated sample's own IC. Distributional agreement is measured against a numerical reference set generated under the same FFM-derived physical conditions, so the reported set-level distances quantify agreement with the corresponding conditioned numerical distribution rather than divergence from Vanilla FFM itself. Further details are provided in Appendix~\ref{app:pde_details}.

\paragraph{Quantitative Results.}
Table~\ref{tab:physical_constraints_benchmark} presents the compact main comparison using three distribution metrics, the mean constraint residual, and mean simulation error. Table~\ref{tab:app_physical_constraints_benchmark} additionally reports KID and partial differential equation (PDE) residuals together with sample-level standard deviations for the constraint and physical-consistency metrics. MintFlow has the smallest FID, MMD, $W_2$, and simulation-error point estimates on both systems, and the full Table~\ref{tab:physical_constraints_benchmark} shows the same ordering for KID and PDE residual. Together, these rankings indicate improved agreement with the conditioned numerical reference and improved dynamical consistency. The constraint metric exposes the complementary trade-off: post-hoc projection methods, ECI and PCFM, obtain the smallest residuals, but their distribution and physical-consistency metrics deteriorate sharply. By contrast, MintFlow reduces the residual by roughly three orders of magnitude relative to Vanilla, DiffusionPDE, and D-Flow while avoiding ECI and PCFM's pronounced distributional shift and dynamical error. The results place MintFlow on a distinct Pareto point rather than establishing uniform dominance. ECI and PCFM remain preferable if minimizing the explicit constraint residual is the sole objective, whereas MintFlow has the lowest reported distribution and physical-consistency point estimates.

\begin{table*}[t]
\centering
\caption{Physical system modeling benchmark results. Performance evaluation across $N=1000$ generated samples comparing System 1 (Heat Equation) and System 2 (Reaction-Diffusion Equation). Lower values ($\downarrow$) indicate superior performance.}
\label{tab:physical_constraints_benchmark}
\resizebox{\textwidth}{!}{
\begin{tabular}{l ccccc ccccc}
\toprule
& \multicolumn{5}{c}{\textbf{System 1: Heat Equation}} & \multicolumn{5}{c}{\textbf{System 2: Reaction-Diffusion Equation}} \\
\cmidrule(lr){2-6} \cmidrule(lr){7-11}
\textbf{Method}
& \textbf{FID} $\downarrow$ & \textbf{MMD} $\downarrow$ & $\mathbf{W_2}\downarrow$ & $\mathbf{H}\downarrow$ & \textbf{Sim. Error} $\downarrow$
& \textbf{FID} $\downarrow$ & \textbf{MMD} $\downarrow$ & $\mathbf{W_2}\downarrow$ & $\mathbf{H}\downarrow$ & \textbf{Sim. Error} $\downarrow$ \\
\midrule

Vanilla FFM
& \cellcolor{red!10}$1.255\mathrm{E}{-3}$
& \cellcolor{red!10}$2.496\mathrm{E}{-2}$
& $4.872\mathrm{E}{-2}$
& $5.16\mathrm{E}{-3}$
& $5.57\mathrm{E}{-2}$
& \cellcolor{red!10}$1.74\mathrm{E}{-4}$
& \cellcolor{red!10}$2.130\mathrm{E}{-2}$
& \cellcolor{red!10}$2.078\mathrm{E}{-2}$
& $1.907\mathrm{E}{0}$
& $3.55\mathrm{E}{-2}$
\\

DiffusionPDE
& $4.712\mathrm{E}{-3}$
& $5.361\mathrm{E}{-2}$
& $7.431\mathrm{E}{-2}$
& $1.86\mathrm{E}{-2}$
& $8.49\mathrm{E}{-2}$
& $9.34\mathrm{E}{-4}$
& $6.063\mathrm{E}{-2}$
& $3.712\mathrm{E}{-2}$
& $9.56\mathrm{E}{-1}$
& $3.29\mathrm{E}{-2}$
\\

D-Flow
& $1.290\mathrm{E}{-3}$
& $2.554\mathrm{E}{-2}$
& \cellcolor{red!10}$4.098\mathrm{E}{-2}$
& $4.71\mathrm{E}{-3}$
& \cellcolor{red!10}$4.34\mathrm{E}{-2}$
& $3.95\mathrm{E}{-4}$
& $3.861\mathrm{E}{-2}$
& $4.085\mathrm{E}{-2}$
& $4.44\mathrm{E}{-1}$
& \cellcolor{red!10}$3.02\mathrm{E}{-2}$
\\

ECI
& $1.865\mathrm{E}{-1}$
& $2.067\mathrm{E}{-1}$
& $4.502\mathrm{E}{-1}$
& \cellcolor{red!10}$2.27\mathrm{E}{-7}$
& $4.47\mathrm{E}{-1}$
& $1.829\mathrm{E}{-2}$
& $1.299\mathrm{E}{-1}$
& $1.647\mathrm{E}{-1}$
& \cellcolor{green!15}$8.02\mathrm{E}{-6}$
& $2.28\mathrm{E}{-1}$
\\

PCFM
& $1.851\mathrm{E}{-1}$
& $2.062\mathrm{E}{-1}$
& $4.485\mathrm{E}{-1}$
& \cellcolor{green!15}$1.30\mathrm{E}{-7}$
& $4.45\mathrm{E}{-1}$
& $1.831\mathrm{E}{-2}$
& $1.301\mathrm{E}{-1}$
& $1.648\mathrm{E}{-1}$
& \cellcolor{red!10}$8.06\mathrm{E}{-6}$
& $2.28\mathrm{E}{-1}$
\\

\textbf{MintFlow}
& \cellcolor{green!15}$2.91\mathrm{E}{-5}$
& \cellcolor{green!15}$5.44\mathrm{E}{-3}$
& \cellcolor{green!15}$1.455\mathrm{E}{-2}$
& $3.57\mathrm{E}{-6}$
& \cellcolor{green!15}$1.23\mathrm{E}{-2}$
& \cellcolor{green!15}$1.09\mathrm{E}{-4}$
& \cellcolor{green!15}$1.342\mathrm{E}{-2}$
& \cellcolor{green!15}$1.870\mathrm{E}{-2}$
& $3.18\mathrm{E}{-4}$
& \cellcolor{green!15}$2.58\mathrm{E}{-2}$
\\

\bottomrule
\end{tabular}
}
\end{table*}

\paragraph{Qualitative Comparison.}
Figure~\ref{fig:physical_qualitative} examines dynamical self-consistency at the sample level. For each source-noise index, the corresponding unconstrained FFM trajectory determines the target initial condition and the physical parameters required to define the numerical problem. These quantities are then held fixed across all samplers. For each generated field, we extract its own IC, integrate the corresponding PDE under the fixed sample-indexed physical parameters and BCs, and display the pointwise absolute discrepancy to this numerical rollout. The results show that ECI and PCFM produce broad, spatially coherent errors despite their small explicit constraint residuals. MintFlow remains visibly closer to its numerical rollout, consistent with its lower aggregate simulation error relative to these baselines. Appendix~\ref{app:pde_qualitative} provides additional fixed-ID examples.

\begin{figure}[h]
    \centering
    \includegraphics[width=\linewidth]{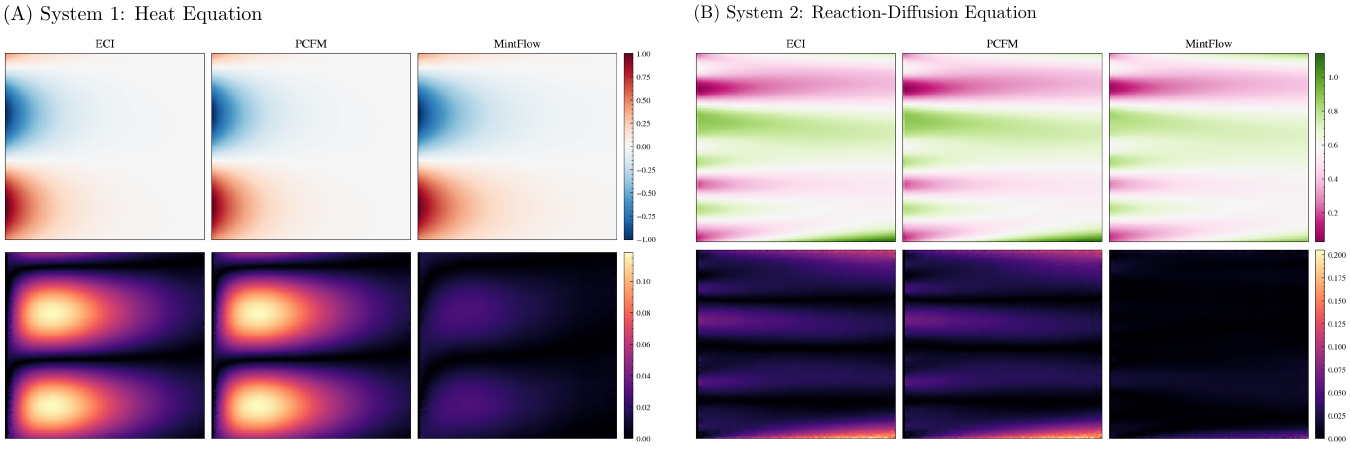}
    \caption{Qualitative physical consistency under physical constraints. The top row shows samples generated by ECI, PCFM, and MintFlow, while the bottom row illustrates the pointwise absolute discrepancy between the generated field and this sample-conditioned numerical rollout.}
    \label{fig:physical_qualitative}
\end{figure}

\paragraph{Model Analysis.} 
\begin{wrapfigure}{r}{0.65\textwidth}
\vspace{-1.5em}
  \begin{center}
    \includegraphics[width=0.32\textwidth]{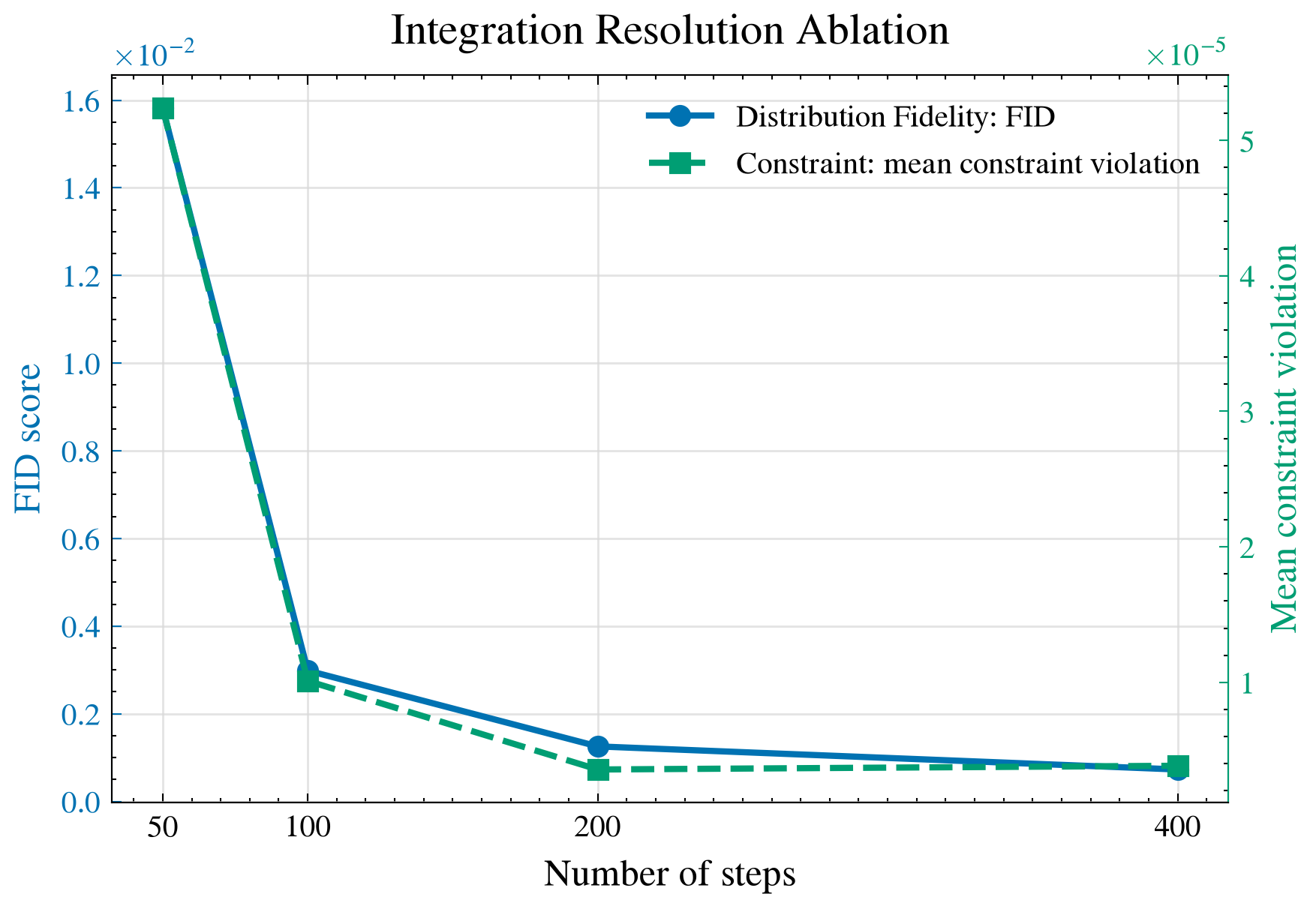}
    \includegraphics[width=0.32\textwidth]{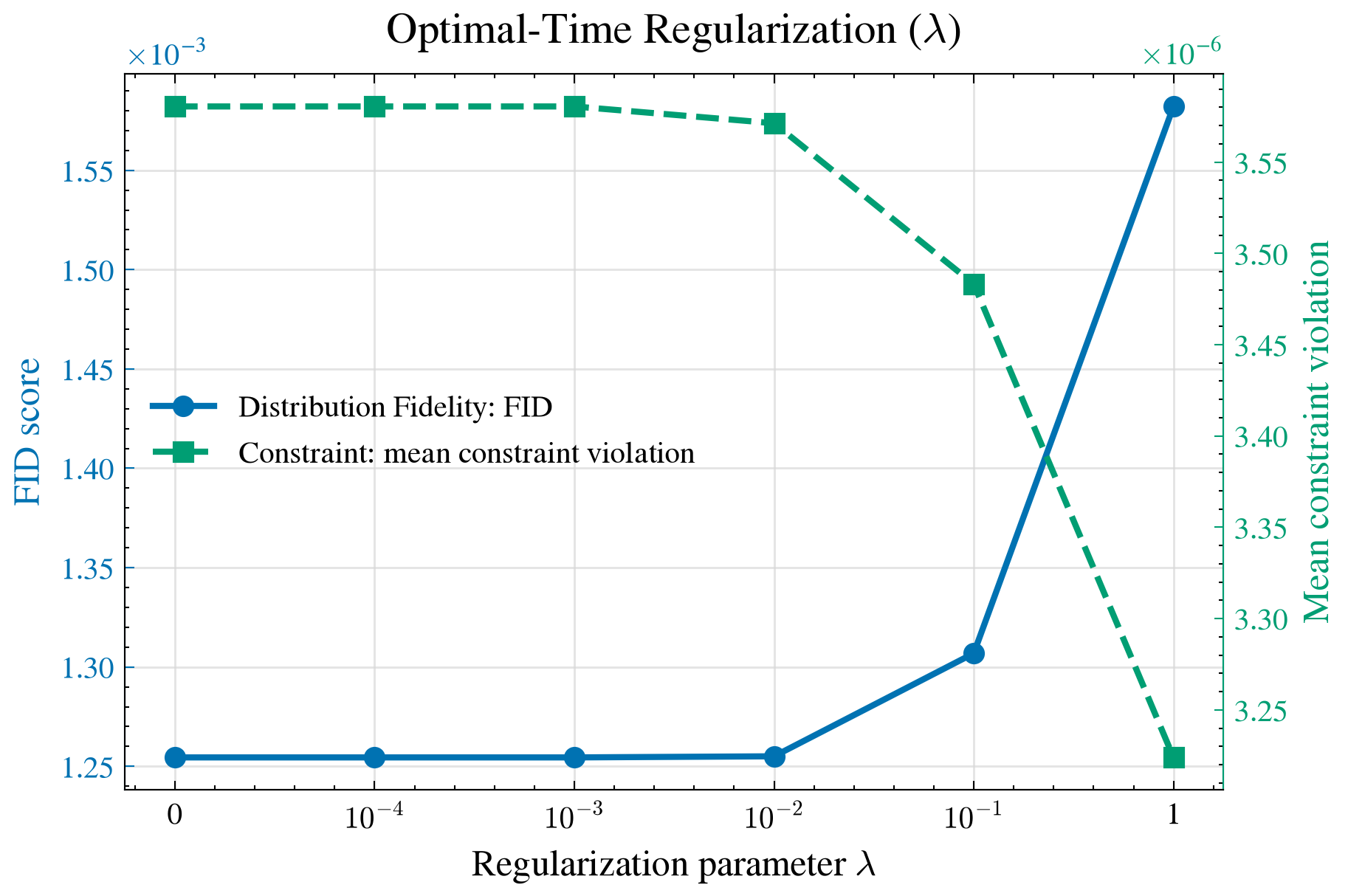}
  \end{center}
  \caption{Sensitivity of MintFlow to sampling steps and regularization parameter $\lambda$.}\label{fig:ablation_heat}
  \vspace{-1.2em}
\end{wrapfigure}
Figure~\ref{fig:ablation_heat} ablates the sampling steps and regularization parameter $\lambda$ in (\ref{eq:optimal_time}) on the Heat dataset. While the benchmark uses 100 Euler sampling steps by default, increasing the sampling resolution improves both distributional fidelity and constraint satisfaction. This is because discretization affects both the forward integration and the backward adjoint computation: finer sampling yields a more accurate trajectory and adjoint, leading to better preservation of the pretrained distribution and more accurate local constraint corrections. Nevertheless, the method remains relatively robust, with noticeable degradation only below 100 sampling steps.
Varying $\lambda$ reveals a trade-off between constraint satisfaction and distribution preservation. Since $\lambda$ controls the preferred intervention time rather than the intervention magnitude directly, larger $\lambda$ favors later interventions. When the intervention is too close to the terminal time, the correction approaches a direct projection of the terminal sample toward the constraint manifold, which can better satisfy the constraint but may deviate further from the pretrained distribution $\nu_T$. The method remains stable for a moderate range of $\lambda$. Further details are provided in Appendix~\ref{app:pde_ablation}.

\section{Conclusion}
This work introduces \textbf{MintFlow}, a training-free framework for constrained generation in flow matching models. By formulating constrained generation as a minimal trajectory intervention, MintFlow preserves the learned generative dynamics while minimally perturbing the flow state. Leveraging an adjoint-based sensitivity formulation, it efficiently computes the minimum-norm intermediate perturbation with theoretical analysis of constraint satisfaction and distributional stability. Extensive experiments across diverse domains, constraints, and backbone models demonstrate that MintFlow effectively satisfies constraints while preserving the pretrained distribution. 
While MintFlow relies on a first-order linearization for high efficiency, it can be seamlessly complemented by iterative refinements or higher-order approximations when dealing with highly nonlinear constraints. Looking forward, extending MintFlow to more challenging nonlinear or non-differentiable constraints, reducing its computational overhead, and broadening the framework to other dynamics-based generative models pose promising directions for future work.
\paragraph{Acknowledgements.}
This work was supported by the U.S. Department of Energy (DOE), Office of Science, Advanced Scientific Computing Research program under award B\&R-KJ0401010 and FWP-CC147 and grant DE-SC0026262.
\paragraph{Reproducibility Statement.}
We provide the necessary details to facilitate reproducibility of our method and experiments. The complete algorithms for single-step correction and iterative refinement are given in Algorithms~\ref{alg:adjoint} and~\ref{alg:iterative_adj}. All theoretical assumptions and complete proofs are provided in Appendix \ref{appen:thoery}. Detailed experimental settings, including datasets, model configurations, hyperparameters, and evaluation protocols, are documented in Appendices \ref{appen:images}-\ref{app:pde_details}. Source code is provided as anonymous supplementary material.

\paragraph{AI Disclosure.} We used generative AI tools solely to identify typographical and grammatical errors and to suggest improvements to the English wording of the manuscript. All suggestions were manually reviewed by the authors, and appropriate changes were incorporated by the authors. We did not use generative AI tools for any of the required-disclosure tasks listed in the ICLR 2027 AI policy. We take responsibility for the final content of this work.

\bibliography{iclr2027_conference}
\bibliographystyle{iclr2027_conference}

\appendix
\section{Theoretical Justification}\label{appen:thoery}

\subsection{Control-Theoretic Interpretation of the Minimum-Norm Correction}\label{appen:control_perspective}

The proposed correction strategy admits a natural interpretation from the perspective of optimal control in (\ref{eq:minimal_correction_problem}) and constrained endpoint optimization.
Introduce a Lagrangian $\eta$:
\[
\frac12 |\delta|^2
+
\eta^\top F_s(\delta),
\]
the first-order optimality conditions are
\[
\delta + DF_s(\delta)^\top \eta = 0,
\quad F_s(\delta) = 0.
\]
Therefore, the exact constrained correction is characterized by the nonlinear Karush–Kuhn–Tucker (KKT) system.

The proposed adjoint correction algorithm can be interpreted as a first-order approximation of the nonlinear KKT system associated with (\ref{eq:minimal_correction_problem}).
Substituting the first-order Taylor expansion of $DF_s(\delta)$ and $F(\delta)$ into the KKT conditions gives
\[
\delta + A_s^\top \eta = 0,\quad 
A_s\delta + H\left(u_T\right) = 0,
\]
the first-order linearization of the nonlinear KKT system around the generated trajectory.
Eliminating the Lagrange multiplier \(\eta\) yields the minimum-norm correction in (\ref{eq:optimal_delta}). Therefore, the proposed correction is the solution of the linearized KKT system associated with the minimum-perturbation endpoint-constrained optimization problem.

\subsection{Validity of the Linearized Correction}
\label{appen:linearized_correction}
We provide the formal statements and proofs corresponding to the informal theoretical guarantees in Theorem~\ref{thm:main_informal_correction}.
In particular, Theorem \ref{theorem:existence_and_constraint_residual} establishes the existence of the minimum-norm linearized correction and quantifies the resulting constraint satisfaction. Theorem \ref{thm:distributional_stability} further shows that the deviation of the corrected terminal distribution from the learned flow-matching distribution is controlled by the magnitude of the intermediate correction and the remaining integration time.

\begin{theorem}[Existence and Accuracy of the Linearized Correction]\label{theorem:existence_and_constraint_residual}

Assume that $f_\theta$ and $H$ are twice continuously differentiable with bounded derivatives
\[
\|D_u f_\theta(u,t)\|_{\mathrm{op}}\le L_1, \qquad \|D_u^2 f_\theta(u,t)\|_{\mathrm{op}}\le L_2,
\]
and
\[
\|\nabla H(y)\|\le M_{H,1}, \qquad \|\nabla^2 H(y)\|_{\mathrm{op}}\le M_{H,2}.
\]
and
$A_s=DF_s(0)$ has full row rank.  Then the linearized constraint
\[
A_s\delta + H(u_T)=0
\]
admits a solution, and its unique minimum-norm solution is (\ref{eq:optimal_delta}).
Moreover, if
\[
\frac{M_F}{2}\|\delta^\ast\|^2 < \|H(u_T)\|,\quad  M_F\coloneqq M_{H,2}e^{2L_1(T-s)} + M_{H,1}\frac{L_2}{L_1}
\left( e^{2L_1(T-s)}-e^{L_1(T-s)} \right)
\]
then
\[
\|H(\Phi(T,s)(u_s+\delta^\ast))\| < \|H(u_T)\|.
\]
\end{theorem}
\begin{proof}
For fixed $u$, define $J(t):=D_u\Phi(t,s)(u)$. The variational equation is
\[
\frac{d}{dt}J(t) = D_u f_\theta(\Phi(t,s)(u),t)J(t), \qquad J(s)=I.
\]
Hence, by Grönwall's inequality,
\[
\|D_u\Phi(T,s)(u)\|_{\mathrm{op}} \le e^{L_1(T-s)}.
\]
Differentiating the variational equation once more with respect to $u$ gives
\[
\frac{d}{dt}D_u^2\Phi(t,s)(u) = D_u f_\theta(\Phi(t,s)(u),t)D_u^2\Phi(t,s)(u)+ D_u^2f_\theta(\Phi(t,s)(u),t) [D_u\Phi(t,s)(u),D_u\Phi(t,s)(u)],
\]
with $D_u^2\Phi(s,s)(u)=0$.
It follows that
\[
\frac{d}{dt} \|D_u^2\Phi(t,s)(u)\|_{\mathrm{op}} \le L_1\|D_u^2\Phi(t,s)(u)\|_{\mathrm{op}} + L_2e^{2L_1(t-s)}.
\]
Another application of Grönwall's inequality yields
\begin{align*}
\|D_u^2\Phi(T,s)(u)\|_{\mathrm{op}}
&\le e^{L_1(T-s)} \int_s^T L_2e^{L_1(t-s)}\,dt\\
&= \frac{L_2}{L_1} e^{L_1(T-s)} \left(e^{L_1(T-s)}-1\right)\\
&= \frac{L_2}{L_1} \left(e^{2L_1(T-s)}-e^{L_1(T-s)}\right).
\end{align*}

Let $\phi(\delta):=\Phi(T,s)(u_s+\delta)$
The second-order chain rule gives, for $v,w\in\mathbb{R}^d$,
\[
D^2F_s(\delta)[v,w] = D^2H(\phi(\delta)) [D\phi(\delta)v,D\phi(\delta)w] + DH(\phi(\delta)) [D^2\phi(\delta)[v,w]].
\]
Thus,
\begin{align*}
\|D^2F_s(\delta)\|_{\mathrm{op}}
&\le \|\nabla^2H(\phi(\delta))\|_{\mathrm{op}} \|D\phi(\delta)\|_{\mathrm{op}}^2 + \|\nabla H(\phi(\delta))\| \|D^2\phi(\delta)\|_{\mathrm{op}}\\
&\le M_{H,2}e^{2L_1(T-s)} + M_{H,1}\frac{L_2}{L_1} \left( e^{2L_1(T-s)}-e^{L_1(T-s)} \right) =: M_F.
\end{align*}

Since
\[
u_T=\Phi(T,s)(u_s),
\]
we have
\[
F_s(0) = H\!\left(\Phi(T,s)(u_s)\right) = H(u_T).
\]
The first-order Taylor expansion of $F_s$ around $\delta=0$ gives
\begin{equation}\label{eq:pf_taylor_F}
F_s(\delta) = F_s(0) + DF_s(0)\delta + R(\delta^\ast),
\end{equation}
where the above estimation gives
\[
\|R(\delta^\ast)\| \leq \frac{M_F}{2}\|\delta^\ast\|^2.
\]
Therefore
\[
F_s(\delta) = H(u_T) + A_s\delta + O(\|\delta\|^2).
\]

Consider the linearized constraint
\[
A_s\delta+H(u_T)=0.
\]
Since $A_s\in\mathbb{R}^{m\times d}$ has full row rank,
\[
-H(u_T)\in\operatorname{Range}(A_s).
\]
Therefore, the linearized constraint admits a solution.

The minimum-norm solution of
\[
A_s\delta=-H(u_T)
\]
is given by the Moore--Penrose pseudoinverse:
\[
\delta_s^\ast = -A_s^\dagger H(u_T).
\]
Indeed, the general solution can be written as
\[
\delta = -A_s^\dagger H(u_T) + \left(I-A_s^\dagger A_s\right)z, \qquad z\in\mathbb{R}^d.
\]
The first term belongs to
\[
\operatorname{Range}(A_s^\top) = \operatorname{Null}(A_s)^\perp,
\]
whereas
\[
\left(I-A_s^\dagger A_s\right)z \in \operatorname{Null}(A_s).
\]
Thus, the two terms are orthogonal, and
\[
\|\delta\|^2 = \left\|A_s^\dagger H(u_T)\right\|^2 + \left\| \left(I-A_s^\dagger A_s\right)z \right\|^2.
\]
Therefore, the norm is uniquely minimized when
\[
\left(I-A_s^\dagger A_s\right)z=0,
\]
which yields
\[
\delta_s^\ast = -A_s^\dagger H(u_T).
\]

Using
\[
F_s(0)=H(u_T), \qquad DF_s(0)=A_s,
\]
the linearization (\ref{eq:pf_taylor_F}) gives
\[
F_s(\delta_s^\ast) = H(u_T) + A_s\delta_s^\ast + R_s(\delta_s^\ast).
\]
By construction of $\delta_s^\ast$,
\[
A_s\delta_s^\ast=-H(u_T).
\]
Hence the zeroth- and first-order terms cancel, giving
\[
F_s(\delta_s^\ast) = R_s(\delta_s^\ast).
\]
Consequently,
\[
\|F_s(\delta^\ast)\| \leq \frac{M_F}{2}\|\delta^\ast\|^2.
\]
and therefore
\[
F_s(\delta_s^\ast) = O\!\left(\left\|\delta_s^\ast\right\|^2\right).
\]
Thus, the minimum-norm solution of the linearized constraint satisfies the original nonlinear constraint up to a second-order residual.
\end{proof}

The assumptions on $f_\theta$ and $H$ impose standard smoothness conditions needed for our perturbation analysis. Bounded first- and second-order derivatives are reasonable for smooth neural networks and differentiable constraint operators, at least on the bounded region traversed by the generated trajectories. The full-row-rank condition on $A_s=DF_s(0)$ is a standard local non-degeneracy assumption, ensuring that the constraint directions are locally independent and that the constrained solution is well-defined. However, this condition may not hold for general constraints, particularly for highly nonlinear constraints where the Jacobian can become rank-deficient at certain points. Thus, the full-row-rank assumption should be viewed as a local regularity condition that restricts the analysis to locally regular constraint configurations.

The following theorem states that a small intermediate correction induces a controlled deviation from the terminal distribution $\nu_T$ generated by the pretrained flow.
\begin{theorem}[Distributional Stability under Intermediate Correction]\label{thm:distributional_stability}
Let $u_s\sim\nu_s$ and define the uncorrected terminal state by
\[
u_T=\Phi(T,s)(u_s),
\]
so that $u_T\sim\nu_T$. Let
\[
u_s^\ast = u_s+\delta_s^\ast(u_s)
\]
be a measurable intermediate correction, and define the corrected terminal
state by
\[
u_T^\ast = \Phi(T,s)(u_s^\ast) = \Phi(T,s)\bigl(u_s+\delta_s^\ast(u_s)\bigr).
\]
Let $\nu_T^\ast$ denote the distribution of $u_T^\ast$.

Assume that the flow field $f_\theta$ is globally $L_f$-Lipschitz with respect to its state variable, uniformly in $t\in[s,T]$, i.e.,
\[
\|f_\theta(u,t)-f_\theta(v,t)\| \leq L_f\|u-v\|, \qquad \forall u,v,\quad t\in[s,T].
\]
Then, for every $u_s$,
\[
\|u_T^\ast-u_T\| \leq e^{L_f(T-s)} \|\delta_s^\ast(u_s)\|.
\]
Consequently, the corrected terminal distribution satisfies
\[
W_2(\nu_T^\ast,\nu_T) \leq e^{L_f(T-s)} \|\delta_s^\ast\|_{L^2(\nu_s)}.
\]
\end{theorem}

\begin{proof}
Fix $u_s\in\mathbb{R}^d$ and consider the uncorrected and corrected
trajectories
\[
u(t)=\Phi(t,s)(u_s), \qquad u^\ast(t) = \Phi(t,s)\bigl(u_s+\delta_s^\ast(u_s)\bigr), \qquad t\in[s,T].
\]
Define the difference by
\[
e(t)=u^\ast(t)-u(t).
\]
Since both trajectories evolve according to the same flow field,
\[
\frac{d}{dt}e(t) = f_\theta(u^\ast(t),t)-f_\theta(u(t),t).
\]
By the $L_f$-Lipschitz assumption,
\[
\left\|\frac{d}{dt}e(t)\right\| \leq L_f\|e(t)\|.
\]
Moreover, at the correction time $s$,
\[
e(s) = u_s^\ast-u_s = \delta_s^\ast(u_s).
\]
Therefore,
\[
\frac{d}{dt}\|e(t)\| \leq L_f\|e(t)\|
\]
for almost every $t\in[s,T]$. Applying Gr\"onwall's inequality gives
\[
\|e(T)\| \leq  e^{L_f(T-s)} \|e(s)\|. 
\]
Since
\[
e(T)=u_T^\ast-u_T
\]
and
\[
e(s)=\delta_s^\ast(u_s),
\]
we obtain the pointwise bound
\[
\|u_T^\ast-u_T\| \leq e^{L_f(T-s)} \|\delta_s^\ast(u_s)\|.
\]

We next establish the distributional bound. Since $u_s\sim\nu_s$ and $u_T=\Phi(T,s)(u_s)$, the pushforward property of the flow gives
\[
u_T\sim \Phi(T,s)_\#\nu_s = \nu_T.
\]
Similarly, by definition,
\[
u_T^\ast = \Phi(T,s) \bigl(u_s+\delta_s^\ast(u_s)\bigr) \sim \nu_T^\ast.
\]
Hence the joint random variable $(u_T,u_T^\ast)$ defines a coupling of $\nu_T$ and $\nu_T^\ast$. By the definition of the Wasserstein-2 distance,
\[
W_2^2(\nu_T^\ast,\nu_T) \leq  \mathbb{E}\left[ \|u_T^\ast-u_T\|^2 \right].
\]
Applying the pointwise stability bound gives
\[
\begin{aligned}
W_2^2(\nu_T^\ast,\nu_T) &\leq \mathbb{E}\left[e^{2L_f(T-s)}\|\delta_s^\ast(u_s)\|^2 \right]\\
&=e^{2L_f(T-s)}\mathbb{E}_{u_s\sim\nu_s}\left[\|\delta_s^\ast(u_s)\|^2\right].
\end{aligned}
\]
Taking the square root yields
\[
W_2(\nu_T^\ast,\nu_T) \leq e^{L_f(T-s)} \left(\mathbb{E}_{u_s\sim\nu_s} \left[\|\delta_s^\ast(u_s)\|^2\right] \right)^{1/2},
\]
which is equivalently
\[
W_2(\nu_T^\ast,\nu_T) \leq e^{L_f(T-s)} \|\delta_s^\ast\|_{L^2(\nu_s)}.
\]
\end{proof}

The global Lipschitz assumption on $f_\theta$ is a standard regularity condition for neural ODEs, ensuring existence, uniqueness, and stability of the flow trajectories. It is generally reasonable for smooth neural networks on bounded regions containing the generated trajectories. However, global Lipschitzness over the entire state space can be restrictive for practical neural networks. Thus, the assumption can be interpreted as requiring Lipschitz continuity only on the relevant region of the state space explored by the flow.

\subsection{Convergence of Iterative Refinement}
\label{appen:convergence_newton}
\begin{theorem}[Convergence of Iterative Refinement]
Let $f_\theta:\mathbb{R}^d\times[s,T]\to\mathbb{R}^d$ be $C^2$ in $u$, let $H:\mathbb{R}^d\to\mathbb{R}$ be $C^2$, and set $\Delta=T-s$. Suppose that, on the relevant reachable set,
\[
\|D_u f_\theta(u,t)\|_{\mathrm{op}}\le L_1,
\qquad \|D_u^2 f_\theta(u,t)\|_{\mathrm{op}}\le L_2,
\]
and
\[
0<m_H\le \|\nabla H(y)\|\le M_{H,1},
\qquad \|\nabla^2 H(y)\|_{\mathrm{op}}\le M_{H,2}.
\]

Then, the correction provided by (\ref{eq:optimal_delta})
\[
u^{(n+1)} = u^{(n)} + \delta^{(n)}
\]
satisfies
\[
|F(u^{(n+1)})| \le K |F(u^{(n)})|^2,
\]
where
\[
K= \frac{1}{2m_H^2} \left[M_{H,2}e^{4L_1\Delta} + M_{H,1}\frac{L_2}{L_1} \left(e^{4L_1\Delta}-e^{3L_1\Delta}\right)\right].
\]
Consequently, if $K|F(u^{(0)})|<1,$
then 
\[
|F(u^{(n)})| \le \frac{1}{K} \left(K|F(u^{(0)})|\right)^{2^n}.
\]
In particular, a sufficient number of iterations for $|F(u^{(n)})|\le\epsilon$ is
\[
n\ge \log_2 \left( \frac{\log(1/(K\epsilon))} {\log(1/(K|F(u^{(0)})|))} \right).
\]
\end{theorem}

\begin{proof}
Let $\phi(u):=\Phi(T,s)(u)$. For fixed $u$, define $J(t):=D_u\Phi(t,s)(u)$. The variational equation is
\[
\frac{d}{dt}J(t) = D_u f_\theta(\Phi(t,s)(u),t)J(t), \qquad J(s)=I.
\]
Hence, by Grönwall's inequality,
\[
\|D_u\Phi(T,s)(u)\|_{\mathrm{op}} \le e^{L_1\Delta}.
\]

Since the flow is a diffeomorphism,
\[
D_u\Phi(T,s)(u)^{-1} = D_y\Phi(s,T)(y), \qquad y=\Phi(T,s)(u),
\]
and the same estimate applied to the inverse flow gives
\[
\|D_u\Phi(T,s)(u)^{-1}\|_{\mathrm{op}} \le e^{L_1\Delta}.
\]
Therefore,
\[
\sigma_{\min}(D_u\Phi(T,s)(u)) = \frac{1}{\|D_u\Phi(T,s)(u)^{-1}\|_{\mathrm{op}}} \ge e^{-L_1\Delta}.
\]

Next, differentiating the variational equation once more with respect to $u$, we obtain
\[
\frac{d}{dt}D_u^2\Phi(t,s)(u) = D_u f_\theta(\Phi(t,s)(u),t)D_u^2\Phi(t,s)(u)+ D_u^2f_\theta(\Phi(t,s)(u),t) [D_u\Phi(t,s)(u),D_u\Phi(t,s)(u)],
\]
with $D_u^2\Phi(s,s)(u)=0$.
Using
\[
\|D_u\Phi(t,s)(u)\|_{\mathrm{op}} \le e^{L_1(t-s)},
\]
we obtain
\[
\frac{d}{dt} \|D_u^2\Phi(t,s)(u)\|_{\mathrm{op}} \le L_1\|D_u^2\Phi(t,s)(u)\|_{\mathrm{op}} + L_2e^{2L_1(t-s)}.
\]
Another application of Grönwall's inequality yields
\begin{align*}
\|D_u^2\Phi(T,s)(u)\|_{\mathrm{op}}
&\le e^{L_1\Delta} \int_s^T L_2e^{L_1(t-s)}\,dt\\
&= \frac{L_2}{L_1} e^{L_1\Delta} \left(e^{L_1\Delta}-1\right)\\
&= \frac{L_2}{L_1} \left(e^{2L_1\Delta}-e^{L_1\Delta}\right).
\end{align*}

By the chain rule gives $\nabla F(u) = D\phi(u)^T\nabla H(\phi(u))$, it follows that 
\[
\|\nabla F(u)\|\ge \sigma_{\min}(D\phi(u)) \|\nabla H(\phi(u))\| \ge m_He^{-L_1\Delta}.
\]

The second-order chain rule gives, for $v,w\in\mathbb{R}^d$,
\[
D^2F(u)[v,w]= D^2H(\phi(u)) [D\phi(u)v,D\phi(u)w] + DH(\phi(u)) [D^2\phi(u)[v,w]].
\]
Thus,
\begin{align*}
\|D^2F(u)\|_{\mathrm{op}}
&\le \|\nabla^2H(\phi(u))\|_{\mathrm{op}} \|D\phi(u)\|_{\mathrm{op}}^2 + \|\nabla H(\phi(u))\| \|D^2\phi(u)\|_{\mathrm{op}}\\
&\le M_{H,2}e^{2L_1\Delta} + M_{H,1}\frac{L_2}{L_1}
\left( e^{2L_1\Delta}-e^{L_1\Delta} \right) =: M_F.
\end{align*}

Since the optimal perturbation is given by
\[
\delta^{(n)}:=u^{(n+1)}-u^{(n)} = -\frac{F(u^{(n)})}{\|\nabla F(u^{(n)})\|^2}\nabla F(u^{(n)}),
\]
we have
\[
\|\delta^{(n)}\| = \frac{|F(u^{(n)})|}{\|\nabla F(u^{(n)})\|}
\le \frac{e^{L_1\Delta}}{m_H}|F(u^{(n)})|.
\]

By Taylor's theorem,
\[
F(u^{(n)}+\delta^{(n)}) = F(u^{(n)}) + \nabla F(u^{(n)})^T\delta^{(n)} + R^{(n)},
\]
where
\[
|R^{(n)}| \le \frac{M_F}{2}\|\delta^{(n)}\|^2.
\]
Since $\delta^{(n)}$ satisfies
\[
\nabla F(u^{(n)})^T\delta^{(n)} = -F(u^{(n)}),
\]
it follows that
\[
F(u^{(n+1)})=R^{(n)},
\]
and therefore
\[
|F(u^{(n+1)})| \le\frac{M_F}{2}\|\delta^{(n)}\|^2 \le \frac{M_F e^{2L_1\Delta}}{2m_H^2} |F(u^{(n)})|^2.
\]
Substituting the definition of $M_F$ gives
\[
|F(u^{(n+1)})| \le K|F(u^{(n)})|^2,
\]
where
\[
K= \frac{1}{2m_H^2} \left[ M_{H,2}e^{4L_1\Delta} + M_{H,1}\frac{L_2}{L_1} \left( e^{4L_1\Delta}-e^{3L_1\Delta} \right) \right].
\]

Finally, let $r^{(n)}:=|F(u^{(n)})|$.  Then $r^{(n+1)}\le K\left(r^{(n)}\right)^2.$
If $Kr^{(0)}<1$, induction gives
\[
r^{(n)} \le \frac{1}{K}(Kr^{(0)})^{2^n}.
\]
Indeed, if the bound holds at iteration $n$, then
\[
r^{(n+1)} \le K\left[\frac{1}{K}(Kr^{(0)})^{2^n}\right]^2 = \frac{1}{K}(Kr^{(0)})^{2^{n+1}}.
\]
Thus,
\[
|F(u^{(n)})|\le\epsilon
\]
whenever
\[
(Kr^{(0)})^{2^n}\le K\epsilon.
\]
Since $Kr^{(0)}<1$, this is equivalent to
\[
2^n \ge \frac{\log(1/(K\epsilon))} {\log(1/(Kr^{(0)}))},
\]
which proves the stated iteration bound and the asymptotic estimate
\[
N_\epsilon = O\left(\log\log\frac{1}{\epsilon}\right).
\]
\end{proof}

\subsection{Specialization to Rectified Flow}
\label{appen:rectified_flow}
The proposed trajectory intervention admits a particularly simple form when the pretrained flow is a rectified flow (RF) model and sampling is performed using an Euler discretization. Specifically, for a single sampling step from an intermediate time $s$ to the terminal time $T$, the RF sampler takes the form
\begin{equation}
\label{eq:rf_one_step}
u_T = u_s+(T-s)v_\theta(u_s,s),
\end{equation}
where $v_\theta$ denotes the learned RF velocity field. After applying an intermediate perturbation $\delta$, the corresponding terminal state becomes
\begin{equation}
\label{eq:rf_perturbed_step}
u_T^\delta = u_s+\delta + (T-s)v_\theta(u_s+\delta,s).
\end{equation}
Consequently, the terminal constraint can be written as
\begin{equation}
\label{eq:rf_constraint}
F_s(\delta) = H\left(u_s+\delta + (T-s)v_\theta(u_s+\delta,s)\right).
\end{equation}

Unlike the general continuous-time formulation, the sensitivity of the terminal state with respect to the intermediate perturbation can be obtained directly from the local Jacobian of the RF velocity field. In particular,
\begin{equation}
\label{eq:rf_local_jacobian}
\left. \frac{\partial u_T^\delta}{\partial \delta}\right|_{\delta=0}
= I+(T-s)J_v(u_s,s),
\qquad
J_v(u_s,s) \coloneqq \frac{\partial v_\theta}{\partial u}(u_s,s).
\end{equation}
Therefore, the constraint Jacobian in (\ref{eq:minimal_correction_problem}) becomes
\begin{equation}
\label{eq:rf_constraint_jacobian}
A_s = \nabla H(u_T) \left[I+(T-s)J_v(u_s,s)\right].
\end{equation}

For a scalar-valued constraint $H$, define the pulled-back constraint gradient
\begin{equation}
\label{eq:rf_pulled_gradient}
g_s \coloneqq \left[I+(T-s)J_v(u_s,s)\right]^\top \nabla H(u_T)^\top.
\end{equation}
The minimum-norm linearized correction is then
\begin{equation}
\label{eq:rf_correction}
\delta_s^\ast = -\frac{g_s}{\|g_s\|^2}H(u_T), \qquad g_s\neq 0.
\end{equation}
Importantly, computing $g_s$ does not require explicitly forming the full Jacobian $J_v$. It can be obtained by a single vector--Jacobian product through the pretrained velocity network. Thus, for a one-step RF sampler, the continuous adjoint ODE can be replaced by a local discrete sensitivity computation.

More generally, when the RF sampler uses multiple Euler steps,
\begin{equation}
\label{eq:rf_multistep}
u_{k+1} = u_k+\Delta t_k\,v_\theta(u_k,t_k),
\end{equation}
the sensitivity from an intervention at step $k$ to the terminal state is given by the product of the corresponding local Jacobians,
\begin{equation}
\label{eq:rf_multistep_jacobian}
D\Phi_{\mathrm{RF}}(1,t_k) = \prod_{j=k}^{K-1} \left[I+\Delta t_j J_v(u_j,t_j)\right].
\end{equation}
Its transpose can be computed efficiently by a discrete adjoint or standard reverse-mode differentiation through the remaining sampling steps. Hence, the proposed method naturally extends from the continuous-time adjoint formulation to the discrete RF sampler used in practice.

Finally, (\ref{eq:rf_correction}) reveals the relationship between our method and ordinary terminal projection. If the RF velocity field is locally insensitive to the state, such that
\begin{equation}
J_v(u_s,s)\approx 0,
\end{equation}
then
\begin{equation}
g_s\approx\nabla H\left(u_T\right)^\top,
\end{equation}
and the proposed correction reduces to the standard minimum-norm projection onto the linearized constraint manifold at the terminal state. In contrast, when $J_v(u_s,s)$ is non-negligible, the factor
\begin{equation}
\left[I+(T-s)J_v(u_s,s)\right]^\top
\end{equation}
accounts for how an intermediate perturbation is transformed by the RF dynamics before reaching the terminal state. Thus, the proposed correction can be viewed as a \emph{dynamics-aware projection} rather than a purely terminal post-hoc projection.

\newpage
\section{Algorithm}
\label{appen:algorithm}
We summarize the proposed constrained generation algorithm from Section~\ref{sec:method}. The single-correction procedure is presented in Algorithm~\ref{alg:adjoint}, while its iterative extension is presented in Algorithm~\ref{alg:iterative_adj}.

\begin{algorithm}[h]
\caption{Adjoint-Based Optimal Correction}
\label{alg:adjoint}

\begin{algorithmic}[1]

\Require
Initial sample $u_0\sim\mu$, vector field $f_\theta$, constraint $H$, candidate times $0<s_1<\cdots<s_K<T$, trade-off parameter $\lambda$

\Statex

\State \textbf{Forward sampling}
\State Solve
\[
\dot u=f_\theta(u,t),\qquad u(0)=u_0,
\]
from $0$ to $T$, and store $\{u(s_k)\}_{k=1}^K$ and the terminal state $u_T$.

\State Compute the terminal residual
\[
r=H(u_T).
\]

\Statex

\State \textbf{Backward adjoint solve}
\State Initialize
\[
\lambda(T)=\nabla H(u_T)^\top.
\]

\State Solve the adjoint ODE backward only once,
from $T$ to $s_1$ and store $\lambda(s_k)$ whenever a candidate time is reached:
\[
\dot\lambda = -\left(\nabla_u f_\theta(u,t)\right)^\top \lambda.
\]

\Statex

\State \textbf{Candidate evaluation}

\For{$k=1,\ldots,K$}

\State Compute
\[
A_{s_k} = \nabla H(u_T) D\Phi(T,s_k).
\]

\State Compute
\[
\delta_k^\ast = -A_{s_k}^{\dagger}r.
\]

\State Evaluate
\[
J_k = \|\delta_k^\ast\|^2 + \lambda(T-s_k)^2.
\]

\EndFor

\Statex

\State Select
\[
k^\ast=\arg\min_k J_k.
\]

\State Re-integrate the flow from
$s_{k^\ast}$:
\[
u(s_{k^\ast})
\leftarrow
u(s_{k^\ast})
+
\delta_{k^\ast}^\ast,
\]
via 
\[
\dot u=f_\theta(u,t), \qquad t\in[s_{k^\ast},T].
\]

\Return corrected sample $u_T^{\mathrm{corr}}$.

\end{algorithmic}
\end{algorithm}

\begin{algorithm}[h]
\caption{Iterative Adjoint-Based Optimal Correction}
\label{alg:iterative_adj}

\begin{algorithmic}[1]

\Require Initial sample $u_0\sim\mu$, vector field $f_\theta$, constraint $H$, candidate number $K$, iteration number $N$

\State $s_0\gets 0,\quad u(s_0)\gets u_0$

\For{$i=1,\ldots,N$}

    \State Sample trajectory on $[s_{i-1},T]$

    \State Select candidate times
    \[
    s_{i-1}<s_1<\cdots<s_K<T
    \]

    \State Solve the adjoint ODE once from $T$ to $s_1$

    \For{$k=1,\ldots,K$}
        \State Compute $\delta_k^\ast$ by (\ref{eq:optimal_delta})
        \State Evaluate $J(s_k)$
    \EndFor

    \State $s_i\gets\arg\min_{s_k}J(s_k)$

    \State Re-integrate the flow from 
    $
    u(s_i)\leftarrow u(s_i)+\delta_i^\ast
    $

\EndFor

\Return corrected sample $u_T^{\mathrm{corr}}$

\end{algorithmic}
\end{algorithm}
\newpage
\section{Related Work}
\label{appen:related_work}
We consider inference-time constraint enforcement for pretrained flow matching \citep{lipman2022flow} and rectified flow models \citep{liu2022flow}. Existing approaches can be broadly categorized according to how they incorporate constraints during sampling: modifying the velocity field, projecting onto the constraint manifold, or optimizing the source or intermediate states, as illustrated in Figures~\ref{fig:steer_velocity}, \ref{fig:projection}, and \ref{fig:optimization}, respectively.

\paragraph{Velocity modification.}
A line of work modifies the pretrained velocity field to steer the generated trajectory toward a desired target or constraint. General flow matching guidance derives exact or approximate corrections for energy-tilted laws \citep{feng2025guidance}. Image inverse solvers correct the path toward the data manifold with gradient-based guidance~\citep{chung2022improving,chung2023diffusion,song2023pseudoinverse}, adapting the same idea to flow matching via likelihood gradients, endpoint estimates, measurement interpolants, or posterior-score surrogates \citep{pokle2024trainingfree,song2026fmps,yan2025fig,kim2025flowdps}. TOCFlow formulates terminal constraints as optimal control with geometry-aware feedback based on the downstream-flow metric \citep{gao2026tocflow}. Lagrangian Dual Flows augment the ODE with dual and slack dynamics that drive equality and inequality residuals to zero near terminal time \citep{kurtz2026ldf}. These methods directly alter the sampling dynamics and require recomputing constraint-specific velocity corrections for each constraint, whereas our method leaves the pretrained vector field unchanged.

\paragraph{Projection.}
Projection methods alternate flow integration with corrections that project the state onto a feasible or observation-consistent set~\citep{christopher2024constrained}. ECI extrapolates an intermediate state to the endpoint, applies a constraint correction, and interpolates the result back into the trajectory \citep{cheng2025gradient}. PCFM combines terminal shooting with a Gauss-Newton projection and optimal-transport-based reinsertion \citep{utkarsh2026physics}. ProFlow replaces the terminal correction with a proximal problem over physical and observational constraints, while CCFM projects onto time-dependent chance-feasible sets \citep{yu2026proflow,liang2025chance}. Image inverse solvers combine data-consistency corrections with flow-path reprojection \citep{pourya2026flower}, and SNAP-FM accelerates PCFM-style nonlinear projections through sparse optimization \citep{kolli2026snapfm}. While projection can effectively enforce the constraints, the explicit state corrections may induce deviations from the distribution learned by the pretrained model. In contrast, our method avoids explicit projection and instead seeks to reduce constraint violations through small perturbations of the intermediate states.

\paragraph{Optimization.}
Another family formulates constraint enforcement as an optimization problem over the generated trajectory. These methods treat the source noise, intermediate states, or distributed controls as decision variables and optimize them to produce samples that satisfy the constraints. D-Flow optimizes the source noise by differentiating a terminal objective through the complete flow solver \citep{ben2024d}. However, backpropagating through the entire sampling ODE requires computing gradients across all integration steps, making the optimization computationally expensive. FlowChef improves efficiency by adapting this formulation to rectified flows, which eliminates the need to backpropagate through the ODE integration \citep{patel2024flowchef}. Rather than optimizing the source noise, FlowChef sequentially optimizes intermediate states. HardFlow solves receding-horizon trajectory-optimization problems with hard constraints on an estimated terminal state \citep{li2026hardflow}, whereas OC-Flow optimizes distributed controls using adjoint costates and a running deviation penalty \citep{wang2024training}. FlowGrad~\citep{liu2023flowgrad} likewise optimizes distributed controls by backpropagating gradients from a terminal guidance loss. Other related works propagate terminal gradients to intermediate controls and optimize corrupted trajectories \citep{liu2023flowgrad,zhang2024flow}.
These approaches leave the pretrained velocity field unchanged, allowing them to accommodate a broad range of constraints without altering the underlying sampling dynamics. However, solving the resulting optimization problems can be computationally demanding, and optimizing intermediate states may introduce deviations from the distribution learned by the pretrained model. In contrast, our method uses minimal perturbations of intermediate states and replaces iterative optimization with a single adjoint system, reducing deviations from the learned distribution while enabling more efficient constraint enforcement.

\section{Application 1: Image Generation}
\label{appen:images}

\begin{table*}[t]
\centering
\caption{Overview of experiments and ablation studies. We evaluate our method
across diverse tasks, datasets, and flow matching (FM) models, and further
conduct ablation studies. Detailed experimental settings and results are
provided in the corresponding sections.}
\label{tab:experimental_summary}
\resizebox{\textwidth}{!}{
\begin{tabular}{lllll}
\toprule
Task & Details & Data & FM Method & Constraint\\
\midrule

Image generation & App.~\ref{appen:images} & MNIST, CIFAR-10 & Flow Matching & Intensity, fidelity \\

Image inverse problems & App.~\ref{appen:inverse_problems} (Sec.~\ref{sec:exp_image_inverse}) & AFHQ-Cat, FFHQ & Rectified Flow & Observation consistency \\

Image editing & App.~\ref{appen:image_editing} (Sec.~\ref{sec:exp_image_editing}) & PIE-Bench & Latent text-conditioned FM & Text-guided generation with region preservation\\

Physics-informed generation & App.~\ref{app:pde_details} (Sec.~\ref{sec:physical_systems}) & Heat, RD1D PDE & Functional Flow Matching & Boundary condition, conservation laws \\
\midrule
\multicolumn{5}{l}{\textit{Ablation studies:} App.~\ref{app:pde_ablation}} \\
\bottomrule
\end{tabular}}
\end{table*}

This appendix provides additional experiments on image generation that are not included in the main manuscript, including benchmark definitions, model and sampling configurations, evaluation protocols, and additional diagnostics. We evaluate a standard flow matching model on image generation tasks with two types of constraints: \textbf{color intensity constraints} and a \textbf{nonlinear fidelity constraint} defined through a pretrained autoencoder. Together with the experiments in the main manuscript, these results further demonstrate that our approach generalizes across different flow matching models and diverse image-domain constraints, including nonlinear constraints, highlighting the versatility of MintFlow.

As in the main experiments, we separately quantify two aspects of constrained generation quality:
\begin{itemize}
\item \textbf{Constraint satisfaction:} how accurately the generated image satisfies the prescribed constraint.
\item \textbf{Distributional fidelity:} how well the distribution of generated images is preserved relative to the reference image distribution.
\end{itemize}

\subsection{Datasets and Constraints}

We use the standard \textbf{MNIST} and \textbf{CIFAR-10} datasets provided by \texttt{torchvision}. All images are converted to \texttt{float32} tensors and scaled to the range $[-1,1]$.

\paragraph{Color Intensity Constraint.}
For an image $u \in [-1,1]^{C\times H\times W}$, we define the channel-wise average intensity vector as
\begin{equation}
\bar{u}
\coloneqq
\frac{1}{HW}
\sum_{h=1}^{H}\sum_{w=1}^{W} u_{:,h,w}
\in\mathbb{R}^C.
\end{equation}
Given a target intensity vector $\mathbf{c}\in\mathbb{R}^C$, we impose
the constraint
\begin{equation}
H(u;\mathbf{c})
\coloneqq
\bar{u}-\beta
=
\frac{1}{HW}
\sum_{h=1}^{H}\sum_{w=1}^{W} u_{:,h,w}
-\mathbf{c}
=0.
\label{eq:intensity_constraint}
\end{equation}
For MNIST, where $C=1$, we impose a scalar intensity constraint with a prescribed target $c\in\mathbb{R}$:
\begin{equation}
H(u)
\coloneqq
\frac{1}{HW}
\sum_{h=1}^{H}\sum_{w=1}^{W}u_{1,h,w}-c
=0.
\label{eq:mnist_intensity_constraint}
\end{equation}
The target intensity $c$ primarily controls the overall stroke thickness of the generated digits.
For CIFAR-10, this constraint controls the overall color composition of the generated image through the mean intensity of each RGB channel.

For each dataset, we first compute the empirical distribution of the average intensity values over the original images. We then randomly sample multiple target values from the range between the 20th and 80th percentiles of this distribution. We restrict the target values to this central range to avoid overly rare intensity levels for which only a small fraction of the data naturally satisfy the constraint, thereby ensuring that the constraint remains sufficiently represented in the underlying data distribution.

\paragraph{AE Fidelity Constraint}
We further consider a fidelity-related constraint defined using an autoencoder (AE) trained on the ground-truth data. Let $\operatorname{Enc}$ and $\operatorname{Dec}$ denote the encoder and decoder, respectively. We define the reconstruction constraint as
\begin{equation}
H(u) = \operatorname{Dec}(\operatorname{Enc}(u))-u =0.
\label{eq:ae_fidelity_constraint}
\end{equation}
For a well-trained AE, the reconstruction error is expected to be small for samples that are consistent with the learned data distribution. Thus, enforcing this constraint encourages generated samples to preserve characteristics captured by the pretrained data manifold. We use this constraint not only to improve the fidelity of generated images, but also to examine the robustness of MintFlow under highly nonlinear and complex constraints defined by a trained neural network.

For MNIST, we train a convolutional AE with a $128$-dimensional latent representation for 200 epochs using the Adam optimizer with a learning rate of $10^{-3}$. The model is trained with the Smooth L1 reconstruction loss. Input images are normalized to $[-1,1]$ during training.

For CIFAR-10, we train a deeper convolutional AE with a $256$-dimensional latent representation for 200 epochs. The model is optimized using AdamW with a learning rate of $5\times10^{-5}$ and weight decay $10^{-4}$, together with a cosine learning-rate schedule. The model is trained using the L1 reconstruction loss. CIFAR-10 images are normalized to $[-1,1]$ and randomly horizontally flipped during training.

The trained AEs achieve reconstruction losses of $0.0008$ on MNIST and $0.00139$ on CIFAR-10 on the test sets, indicating accurate reconstruction of the data and providing a reliable basis for the proposed AE fidelity constraint. This experiment assumes that the pretrained AE provides a meaningful representation of the underlying data distribution; its purpose is to evaluate the applicability of MintFlow to nonlinear constraints induced by a pretrained neural network, rather than to claim AE reconstruction error as a universally optimal measure of fidelity.

\subsection{Flow Matching Model and Baselines}
For image datasets, all flow matching models were implemented based on the official Flow Matching repository.\footnote{\url{https://github.com/facebookresearch/flow_matching}}
We used a U-Net architecture and Table~\ref{tab:model-config} summarizes the the model configurations.  For all experiments, models were optimized using the AdamW optimizer with a learning rate of $1\times10^{-4}$ and momentum coefficients $\beta_1=0.9$ and $\beta_2=0.95$. The learning rate was kept constant throughout training using a constant learning rate scheduler.  At inference time, we solve the flow ODE using 100 integration steps for all experiments.
\begin{table}[t]
\centering
\caption{Flow Matching model configurations for image datasets.}
\label{tab:model-config}
\begin{tabular}{lcccccc}
\toprule
Dataset & Channels & ResBlocks & Attention & Dropout & Channel Mult.  \\
\midrule
MNIST & 64 & 2 & $\{2\}$ & 0.1 & $(1,2,2)$ \\
CIFAR-10 & 128 & 4 & $\{2\}$ & 0.3 & $(2,2,2)$  \\
\bottomrule
\end{tabular}
\end{table}

We consider the same baselines as in Section \ref{sec:physical_systems}. Specifically, DiffusionPDE guides the pretrained FM velocity using the gradient of the constraint, D-Flow optimizes the latent noise to satisfy the constraint, while ECI and PCFM enforce the constraint through a final projection step.
Although some of these methods were primarily developed for physics-informed constrained generation, their formulations are not inherently restricted to physics data and can, in principle, be applied to general data domains with differentiable constraints. We therefore include them here to examine their applicability to image-domain constraints and to provide a consistent empirical comparison. In particular, the color intensity constraint considered here shares the same global structure as the conservation constraints addressed by these methods, making it a natural setting for cross-domain evaluation. The AE fidelity constraint further examines how these approaches behave under a more complex and highly nonlinear constraint induced by a pretrained neural network. Our goal is not to suggest that image generation was the original design target of these methods, but rather to assess the generality of constrained generation approaches across different data domains and constraint structures. For our method on image data, we set the number of candidate intervention times to $K=10$, uniformly spaced across the entire sampling trajectory. All image experiments were conducted on a single NVIDIA H200 GPU.

\subsection{Evaluation Protocol}
To evaluate the quality of constraint-generated samples, we consider two aspects: constraint satisfaction and preservation of the data distribution.

\paragraph{Constraint Satisfaction.}
We measure the constraint satisfaction using the average constraint value over generated samples. For each method, including the ground-truth conditional samples, we generate 2000 samples and report the average constraint value.

\paragraph{Conditional Distribution Preservation.}
To evaluate whether the generated samples preserve the distribution learned by the flow matching model under the imposed constraint, we measure the discrepancy between constraint-generated samples and reference samples. The goal of constraint generation is to obtain samples that satisfy the imposed constraint while remaining within the distribution learned by the pretrained flow matching model. Therefore, the reference distribution used for evaluation is not the ground-truth data distribution, but a conditional distribution constructed from samples generated by the pretrained flow matching model.

Specifically, we first generate samples from the pretrained flow matching model and collect 2,000 samples that satisfy the given constraint. These constraint-satisfying samples are used as reference samples for evaluation. We then measure the discrepancy between the constraint-generated samples and the reference samples to quantify how well the generated samples preserve the flow-matching-induced distribution.

For image data, we evaluate the conditional distribution matching using FID, KID, MMD, and Wasserstein-2 ($W_2$) distance.
For MMD, we use a Gaussian RBF kernel:
\[
k(x,y)=\exp\left(-\frac{|x-y|^2}{2\sigma^2}\right),
\]
where $\sigma^2$ is determined using the median heuristic computed from pairwise distances between generated and reference samples.
Since directly estimating the empirical $W_2$ distance becomes computationally prohibitive in high-dimensional feature spaces, we approximate each feature distribution by a Gaussian distribution using its empirical mean and covariance, and compute the exact $W_2$ distance between the resulting Gaussian approximations. This provides a computationally tractable measure of distributional discrepancy while capturing differences in both the first- and second-order statistics.

For MNIST, feature-space metrics are computed using the penultimate-layer representations of a pretrained MNIST classifier. The classifier is a residual convolutional network trained on the MNIST training set for 200 epochs using the Adam optimizer with a learning rate of $10^{-3}$. The extracted features correspond to the 64-dimensional representation before the final classification layer. For CIFAR-10, feature representations are extracted using a pretrained Inception-v3 network.

\paragraph{Efficiency.}
Runtime is measured as synchronized wall-clock time per generated sample, and memory is reported as peak allocated GPU memory.

\subsection{Results}

\subsubsection{MNIST}
\begin{table}[t]
\centering
\caption{Quantitative results on the average-intensity constrained MNIST dataset. Light green (\colorbox{green!15}{green}) indicates the top performer, while light red (\colorbox{red!10}{red}) represents the second-best performer across all methods for each metric.}
\label{tab:mnist_feature_metrics}
\resizebox{0.95\textwidth}{!}{
\begin{tabular}{lccccccc}
\toprule
& \multicolumn{4}{c}{Distribution Metrics }
& \multicolumn{1}{c}{Constraint}
& \multicolumn{2}{c}{Efficiency} \\
\cmidrule(lr){2-5} \cmidrule(lr){6-6} \cmidrule(lr){7-8}
Method 
& FID $\downarrow$
& KID $\downarrow$
& MMD $\downarrow$
& $W_2$ $\downarrow$
& $H$ $\downarrow$
& Time (s) $\downarrow$
& Memory (GB) $\downarrow$\\
\midrule
Unconstrained 
&  0.879
&  0.02215
&  0.03158
&  0.938
& 2.79E-1
& 5.07
& 0.10
\\
DiffusionPDE & 0.899 & 0.06860 & \cellcolor{red!10}0.04427 & 1.064 & 5.36E-2 & \cellcolor{green!15}4.29 & 5.10\\
D-Flow & \cellcolor{red!10}0.813
& \cellcolor{red!10}0.05342
& 0.04726
& \cellcolor{red!10}0.902& 5.49E-2 & 703.40 & 125.27\\
ECI
& 1.851& 0.09982 &0.07863 & 1.361 
& \cellcolor{green!15}3.55E-15
& 8.89
& \cellcolor{green!15}0.35
\\

PCFM
& 1.736
& 0.12379
& 0.12064
& 1.317
& \cellcolor{red!10}1.43E-12
& 18.37
& 10.12
\\

Ours
& \cellcolor{green!15}0.262
& \cellcolor{green!15}0.01812
& \cellcolor{green!15}0.01289
& \cellcolor{green!15}0.602
& 4.55E-5
& \cellcolor{red!10}6.66
& \cellcolor{red!10}0.74
\\

\bottomrule
\end{tabular}}
\end{table}
\paragraph{Intensity Constraint.}
We evaluate MintFlow on MNIST under an average-intensity constraint with a target value of $-0.6$, generating 2,000 samples for each method. All methods generate images using the same pretrained flow-matching (FM) model. To construct a reference distribution for evaluating distributional metrics, we generate samples from the pretrained FM model and retain those that satisfy the target constraint, obtaining 2,000 constraint-satisfying samples. We use these samples as the ground-truth samples from the conditional distribution induced by the pretrained flow model. Table~\ref{tab:mnist_feature_metrics} summarizes the quantitative results, while Figure~\ref{fig:mnist_intensity_appendix} presents qualitative comparisons across methods.

Among existing constrained generation methods, ECI and PCFM first generate samples and then apply a post-hoc projection to enforce the constraint. Consequently, these methods achieve nearly exact constraint satisfaction, as indicated by their extremely small constraint errors. However, their distribution-matching performance deteriorates substantially, suggesting that strict post-hoc constraint enforcement can cause the generated samples to deviate from the distribution learned by the flow model. This behavior is also evident in the qualitative results. In particular, their final projection adjusts the generated samples to satisfy the target average intensity without explicitly preserving the visual structure of the original samples. Since the projection only enforces the global intensity constraint, it can significantly alter the appearance of the image.

D-Flow achieves the strongest distributional performance among the baselines. However, it incurs substantial computational overhead because its optimization requires retaining the computational graph throughout the sampling ODE from the latent noise to the generated sample in order to compute the required gradients. This makes D-Flow considerably more memory-intensive and time-consuming than methods that do not require backpropagation through the entire sampling trajectory.

In contrast, MintFlow provides a better balance between constraint satisfaction and distributional fidelity. Although its constraint error is slightly higher than those of PCFM and ECI, it achieves substantially better FID, KID, MMD, and $W_2$ scores, indicating closer agreement with the conditional distribution induced by the pretrained flow model. MintFlow is also faster than the iterative correction methods, D-Flow and PCFM, despite requiring Jacobian-vector products in its adjoint ODE formulation.

For MNIST, average image intensity is closely related to digit stroke thickness. As shown in Figure~\ref{fig:mnist_strokes}, varying the target intensity allows MintFlow to produce thicker or thinner strokes while preserving the overall structure of the digits. This demonstrates that the intensity constraint provides direct and interpretable control over a salient visual property. Unlike projection-based methods, which seek a constraint-satisfying point by modifying the final generated sample, MintFlow applies intermediate perturbations along the flow trajectory. Consequently, the corrected samples are not explicitly required to remain close to their unconstrained counterparts, allowing the model to reach the target constraint while better preserving the learned data distribution.

\begin{figure}[h]
    \centering
    
    \begin{subfigure}[t]{0.9\linewidth}
        \centering
        \includegraphics[width=\linewidth]{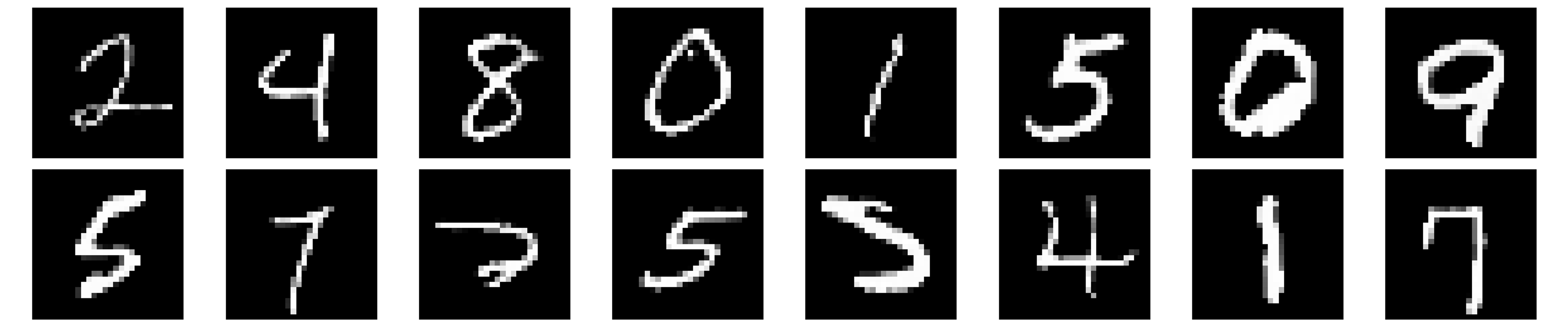}
        \caption{Unconstrained}
        \label{fig:mnist_strokes_orig}
    \end{subfigure}
    
    \begin{subfigure}[t]{0.9\linewidth}
        \centering
        \includegraphics[width=\linewidth]{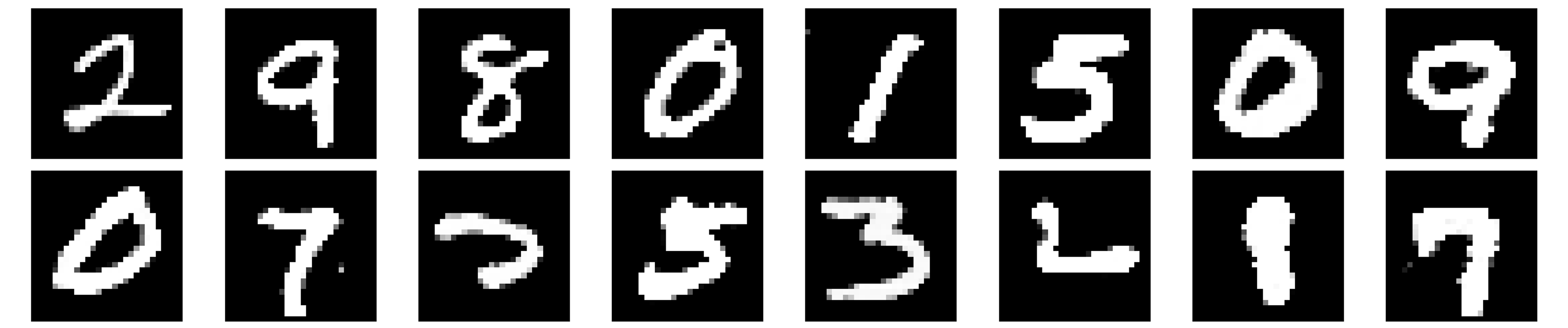}
        \caption{Higher intensity constraint}
        \label{fig:mnist_strokes_thicker}
    \end{subfigure}

    \begin{subfigure}[t]{0.9\linewidth}
        \centering
        \includegraphics[width=\linewidth]{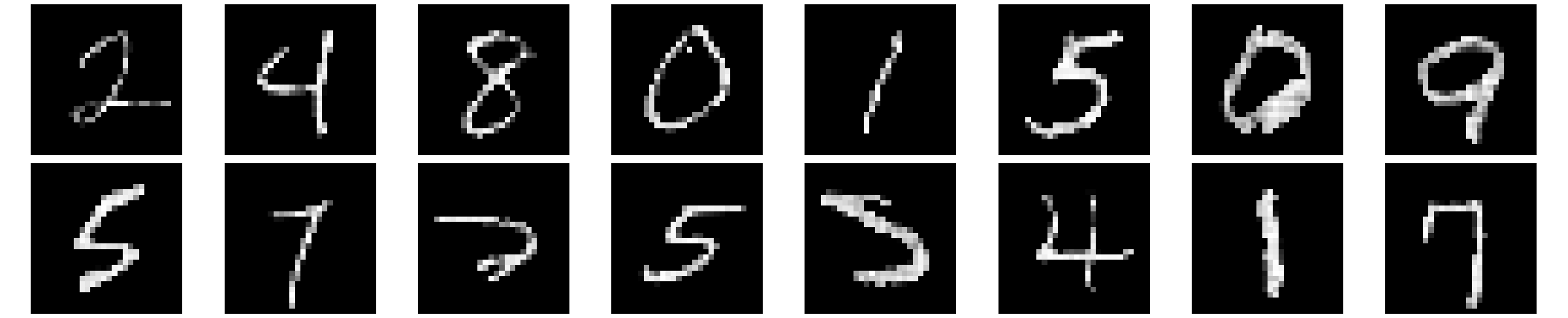}
        \caption{Lower intensity constraint}
        \label{fig:mnist_strokes_thinner}
    \end{subfigure}

    \caption{Controlling MNIST digit stroke thickness with MintFlow using average intensity constraints. Starting from unconstrained samples, our method modifies the generated digits according to different target intensity constraints, producing thicker or thinner strokes while preserving distribution.}
    \label{fig:mnist_strokes}
\end{figure}

\paragraph{Data Fidelity Constraint.}
To evaluate whether our method can improve the fidelity of generated samples, we consider a learned reliability constraint based on a AE. The AE is trained to reconstruct original MNIST images through the encoder-decoder pipeline. Therefore, for a well-trained AE, the reconstruction error in (\ref{eq:ae_fidelity_constraint})
provides a measure of how closely a generated sample lies on the learned data manifold. In this experiment, satisfying the AE constraint encourages generated samples to preserve characteristics of the original MNIST distribution. In this experiment, enforcing the AE constraint is intended to encourage generated samples to remain close to the original MNIST distribution. Accordingly, we use the original MNIST test set as the ground-truth distribution and compute distributional metrics between the generated samples and the ground-truth MNIST samples.

Quantitative results are summarized in Table~\ref{tab:mnist_vae_metrics} together with qualitative results in Figure \ref{fig:mnist_ae_appendix}. Compared with unconstrained sampling, our correction method consistently improves all distribution matching metrics. This demonstrates that the proposed learned constraint can guide the pretrained flow model toward samples with improved fidelity while remaining close to the original flow-induced distribution. Unlike approaches that improve sample quality through additional training or model refinement, our framework provides a complementary inference-time strategy that enhances pretrained flow matching models without modifying the underlying model parameters.

In contrast, existing constrained generation methods show limited performance under this learned nonlinear constraint. Projection-based correction methods struggle to effectively satisfy the AE-based reconstruction constraint, resulting in even larger constraint errors than unconstrained sampling. In particular, ECI's projection-based approach, which is designed to enforce hard constraints, struggles to handle such nonlinear learned constraints effectively. D-Flow incurs additional computational overhead because its optimization requires backpropagating gradients through both the sampling process and the pretrained AE. DiffusionPDE achieves the strongest constraint satisfaction among the baselines, but its overall distributional and constraint metrics remain worse than those of our method, while requiring substantially more memory. 

Meanwhile, our method successfully reduces the reconstruction error while achieving the best distribution-matching performance, demonstrating its ability to incorporate nonlinear learned constraints while preserving the underlying flow distribution. Although our method also requires differentiation through the AE encoder and decoder, it maintains a relatively low memory footprint and remains faster than D-Flow and PCFM. These results demonstrate that our approach can effectively incorporate learned nonlinear constraints without substantially compromising either distributional fidelity or computational efficiency.

Overall, this experiment demonstrates two important aspects of our framework: (1) its versatility in incorporating learned reliability constraints to improve generated sample quality, and (2) its ability to extend constrained generation beyond manually designed operators. We acknowledge that the effectiveness of such learned constraints naturally depends on the quality and generalization capability of the auxiliary model used to define them. Therefore, the purpose of this experiment is not to introduce the AE reconstruction constraint as a universally optimal fidelity metric, but rather to demonstrate the potential of our framework to flexibly integrate learned nonlinear constraints during inference and improve pretrained flow matching models without additional training.

\begin{table}[h]
\centering
\caption{Quantitative results on the AE-based fidelity constrained MNIST dataset.}
\label{tab:mnist_vae_metrics}
\resizebox{\textwidth}{!}{
\begin{tabular}{lccccccc}
\toprule
& \multicolumn{4}{c}{Distribution Metrics }
& \multicolumn{1}{c}{Constraint}
& \multicolumn{2}{c}{Efficiency} \\
\cmidrule(lr){2-5} \cmidrule(lr){6-6} \cmidrule(lr){7-8}
Method 
& FID $\downarrow$
& KID $\downarrow$
& MMD $\downarrow$
& $W_2$ $\downarrow$
& $H$ $\downarrow$
& Time (s) $\downarrow$
& Memory (GB) $\downarrow$\\
\midrule
Unconstrained 
& 0.261
& 0.0089 
&  0.0093 
&  0.0093
& 4.54E-3
& 4.45
& 0.12
\\
DiffusionPDE &  \cellcolor{red!10}0.242 & \cellcolor{red!10}0.0074 & \cellcolor{red!10}0.0078 & \cellcolor{red!10}0.0069 & \cellcolor{red!10}3.59E-3 &  \cellcolor{green!15}6.98 & 5.15\\
D-Flow & 0.257 & 0.0087 & 0.0092 & 0.5076 & 3.91E-3 & 1674.73 & 63.76 \\
ECI
& 0.932
& 0.0079
& 0.0164
& 0.9659
& 3.43E-2
& \cellcolor{red!10}11.91
& \cellcolor{green!15}0.40
\\
PCFM
& 10.225
& 0.3210
& 0.3336
& 3.1977
& 1.02E-1
& 79.56
& 10.17
\\

Ours
& \cellcolor{green!15}0.221
& \cellcolor{green!15}0.0030
& \cellcolor{green!15}0.0042
& \cellcolor{green!15}0.0048
& \cellcolor{green!15}2.67E-3
& 18.06
& \cellcolor{red!10}0.76
\\

\bottomrule
\end{tabular}}
\end{table}
\newpage
\begin{figure}[H]
    \centering
    \begin{tikzpicture}
        \node[inner sep=0pt] (img) {
            \includegraphics[width=0.95\textwidth]{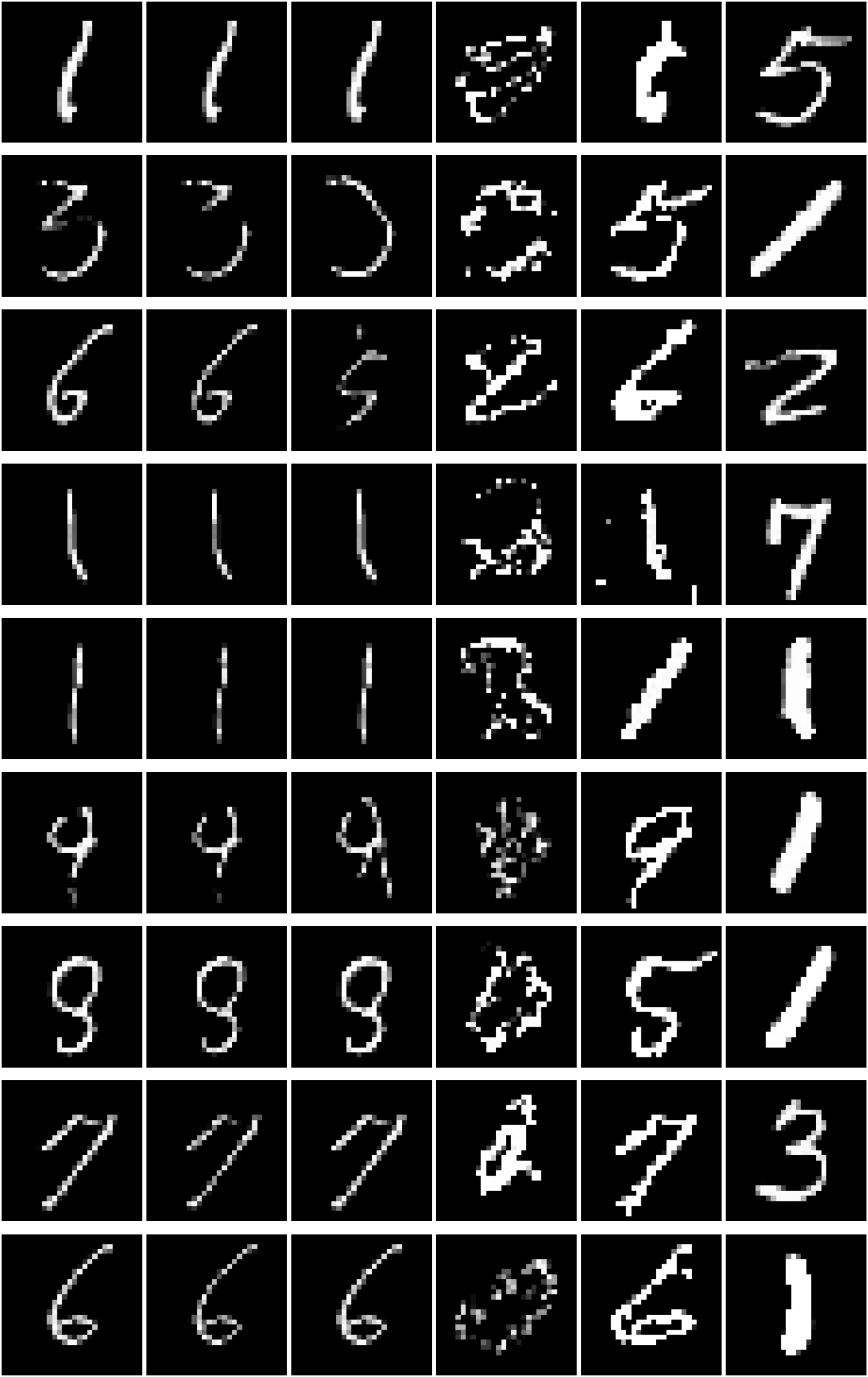}
        };

        \foreach \i/\label in {
            0/Unconstrained,
            1/DiffusionPDE,
            2/D-Flow,
            3/ECI,
            4/PCFM,
            5/\textbf{MintFlow}
        }{
            \node[anchor=south] at
                ($(img.north west)!\the\numexpr(2*\i+1)\relax/12!(img.north east)$)
                {\label};
        }
    \end{tikzpicture}
    \caption{Uncurated results for intensity constrained generation on MNIST. We compare unconstrained samples with constrained samples generated by different methods. All samples in each row are generated from the same initial noise across methods.}
    \label{fig:mnist_intensity_appendix}
\end{figure}
\newpage
\begin{figure}[H]
    \centering
    \begin{tikzpicture}
        \node[inner sep=0pt] (img) {
            \includegraphics[width=0.95\textwidth]{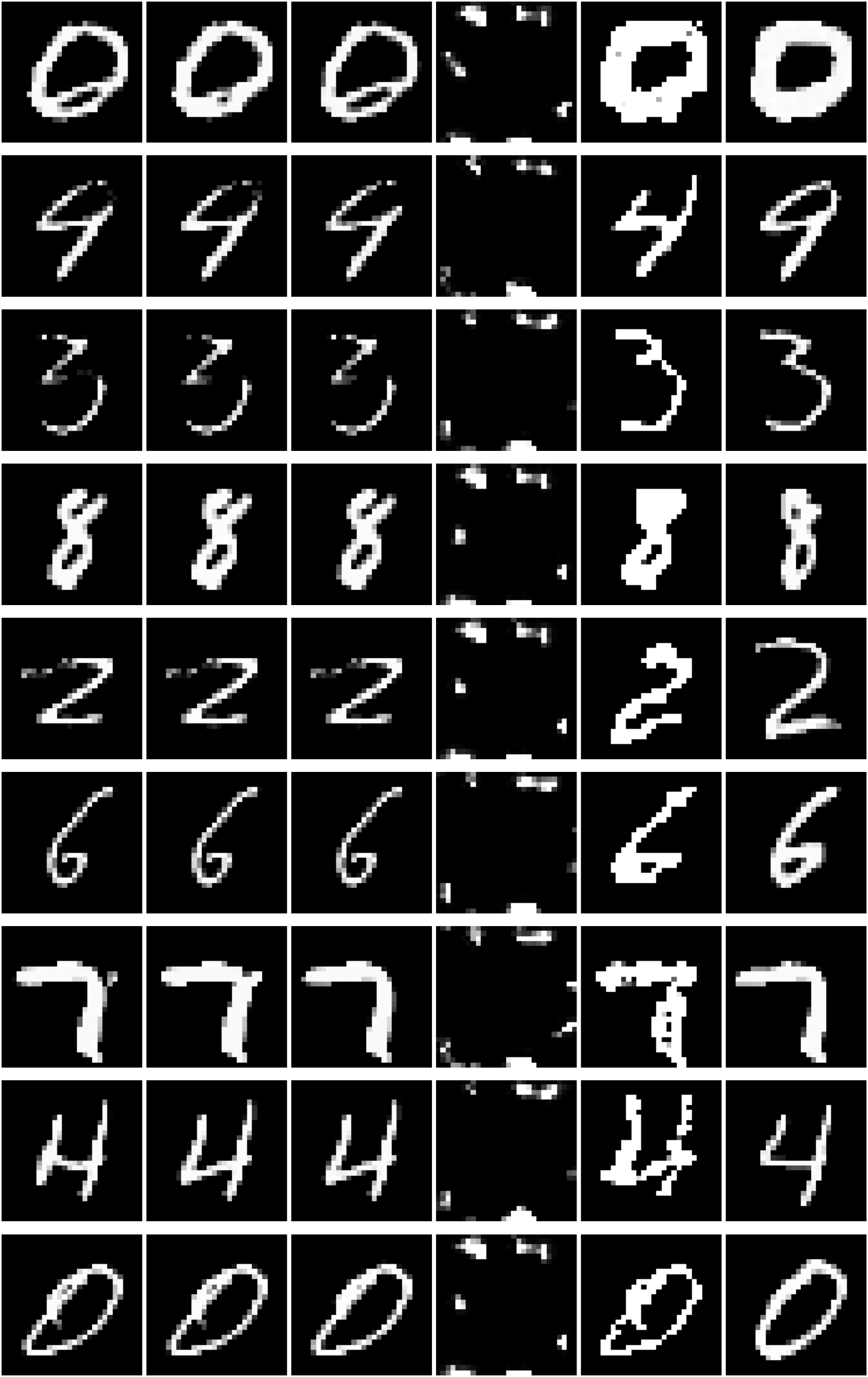}
        };

        \foreach \i/\label in {
            0/Unconstrained,
            1/DiffusionPDE,
            2/D-Flow,
            3/ECI,
            4/PCFM,
            5/\textbf{MintFlow}
        }{
            \node[anchor=south] at
                ($(img.north west)!\the\numexpr(2*\i+1)\relax/12!(img.north east)$)
                {\label};
        }
    \end{tikzpicture}
    \caption{Uncurated results for AE-based fidelity-constrained generation on MNIST. We compare unconstrained samples with constrained samples generated by different methods. All samples in each row are generated from the same initial noise across methods.}
    \label{fig:mnist_ae_appendix}
\end{figure}

\paragraph{Selected Intervention Times.}
To provide empirical evidence on the intervention times selected by our adaptive criterion in Section~\ref{sec:time_seclection}, we visualize the distribution of $s^\ast$ obtained from (\ref{eq:optimal_time}). Specifically, for each generated sample, we evaluate the objective $J(s)$ over the predefined candidate intervention times and select the minimizing time $s^\ast$. Figure~\ref{fig:selected_times} shows the resulting histograms for two constraints on MNIST: the intensity constraint and the fidelity constraint.

The selected intervention times exhibit distinct patterns across the two constraints. For the linear intensity constraint, the selected $s^\ast$ values are concentrated more heavily toward the earlier part of the trajectory. In contrast, the nonlinear fidelity constraint, defined using a trained autoencoder, produces a more dispersed distribution of intervention times, with selections spread across a wider range of the trajectory.

This difference can be understood from the structure of the two constraint functions. The intensity constraint is linear in the terminal state, resulting in a relatively simple and consistent constraint geometry. Consequently, the sensitivity of the constraint to an intermediate perturbation tends to vary more systematically along the trajectory, and the objective in (\ref{eq:optimal_time}) frequently favors earlier intervention times. The fidelity constraint, on the other hand, is induced by a nonlinear trained AE and therefore has a more complex local geometry. The corresponding constraint sensitivity can vary substantially across both samples. As a result, there is no single region of the trajectory that consistently minimizes the trade-off encoded by $J(s)$, leading to a broader distribution of selected intervention times.

Overall, these results suggest that the proposed time-selection criterion is not merely selecting a fixed or universally preferred intervention point. Instead, $s^\ast$ adapts to the structure of the constraint and to the individual generated trajectory. The contrast between the two constraints provides empirical evidence that the adaptive intervention mechanism responds to the geometry and sensitivity of the constraint rather than imposing a predetermined intervention schedule.

\begin{figure}[h]
    \centering
    \captionsetup[subfigure]{font=normalsize, labelformat=parens, labelsep=space}

    \begin{subfigure}[t]{0.48\linewidth}
        \centering
        \includegraphics[width=\linewidth]{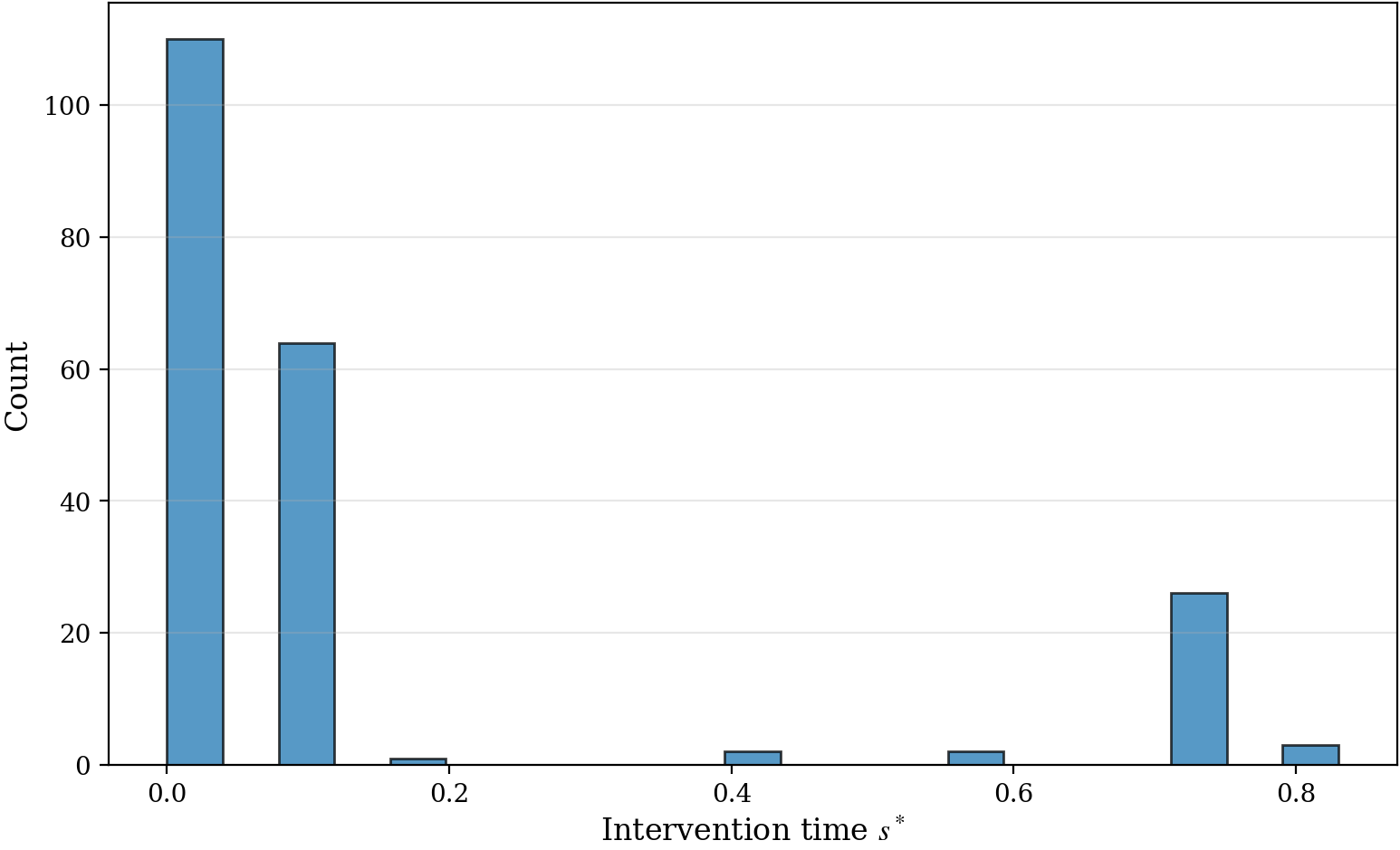}
        \caption{Intensity constraint}
        \label{fig:selected_time_intensity}
    \end{subfigure}
    \hfill
    \begin{subfigure}[t]{0.48\linewidth}
        \centering
        \includegraphics[width=\linewidth]{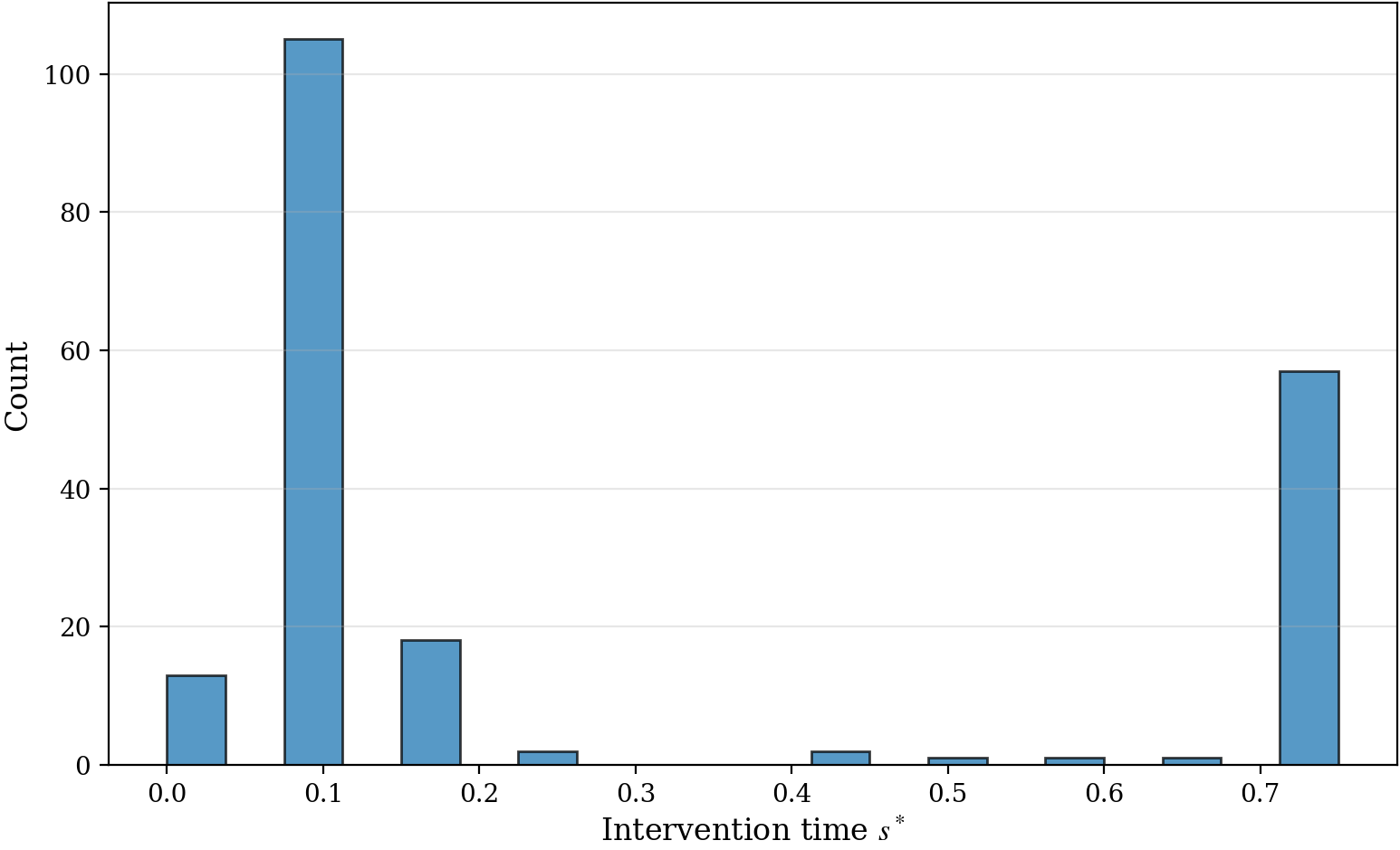}
        \caption{Fidelity constraint}
        \label{fig:selected_time_ae}
    \end{subfigure}

    \caption{
    Distribution of the selected intervention time $s^\ast$ on MNIST.
    Each histogram shows the intervention times selected by (\ref{eq:optimal_time})
    across generated samples for the corresponding constraint.
    \textbf{(a)} The linear intensity constraint tends to favor earlier intervention times,
    whereas \textbf{(b)} the nonlinear fidelity constraint defined by a trained autoencoder
    results in a more dispersed distribution of selected intervention times.
    }
    \label{fig:selected_times}
\end{figure}

\subsubsection{CIFAR10}
\paragraph{Color Tone Constraint.}
We evaluate our method on CIFAR-10 using global color-tone constraints. We consider multiple target color tones and generate a total of 10,000 samples for quantitative evaluation. The corresponding quantitative results are reported in Table~\ref{tab:cifar_feature_metrics}, while qualitative comparisons are shown in Figure~\ref{fig:cifar_intensity_appendix}. Overall, we observe trends similar to those in the MNIST experiments. Both quantitative and qualitative results demonstrate that our method achieves a favorable balance between constraint satisfaction and distribution preservation.
Existing methods that rely on explicit projection steps tend to introduce visible artifacts when enforcing the color-tone constraint, leading to degradation in image fidelity. This becomes particularly noticeable for CIFAR-10, where the higher image resolution and more complex visual structure make direct post-hoc modifications more likely to affect semantic and local image features. D-Flow also incurs substantially greater computational overhead than MNIST as the image resolution increases, making it considerably more expensive in this setting. DiffusionPDE achieves relatively strong distributional performance, but exhibits larger constraint errors. In particular, the qualitative results suggest that it does not sufficiently adjust the generated images toward the target color tone, consistent with its relatively large constraint error.

In contrast, our method consistently provides a better trade-off between constraint satisfaction and distributional fidelity. It effectively shifts generated samples toward the desired color-tone appearance while preserving their semantic content and overall visual quality. Furthermore, we demonstrate the flexibility of our approach by applying a diverse set of color-style constraints, showing that the proposed framework can provide controllable changes in global appearance without substantially disrupting the underlying image distribution.

\begin{table}[h]
\centering
\caption{Quantitative results on the color-style constrained CIFAR10 dataset.}
\label{tab:cifar_feature_metrics}
\resizebox{\textwidth}{!}{
\begin{tabular}{lccccccc}
\toprule
& \multicolumn{4}{c}{Distribution Metrics }
& \multicolumn{1}{c}{Constraint}
& \multicolumn{2}{c}{Efficiency} \\
\cmidrule(lr){2-5} \cmidrule(lr){6-6} \cmidrule(lr){7-8}
Method 
& FID $\downarrow$
& KID $\downarrow$
& MMD $\downarrow$
& $W_2$ $\downarrow$
& $H$ $\downarrow$
& Time (s) $\downarrow$
& Memory (GB) $\downarrow$\\
\midrule
Unconstrained 
& 11.1568  & 0.0421  & 0.0518  & 5.354  & 1.02E-1 & 14.17 & 0.44 \\
DiffusionPDE & \cellcolor{red!10}8.4464 & \cellcolor{red!10}0.0014 & \cellcolor{green!15}0.0012 & \cellcolor{red!10}2.9062 & 4.66E-2 & \cellcolor{green!15}12.08 & \cellcolor{red!10}2.25\\
D-Flow & 9.8484  & 0.0019 & 0.0015 & 3.138 & 5.789E-4 & 787.30 & 101.01 \\
ECI
& 322.500 & 0.3596 & 0.29951 & 17.9582
& \cellcolor{green!15}2.53E-16
& \cellcolor{red!10}19.81
& \cellcolor{green!15}0.34
\\
PCFM
& 125.262 & 0.0950 & 0.0727 & 11.1920
& 1.72E-12
& 27.10
& 4.26
\\

Ours
& \cellcolor{green!15}5.08 & \cellcolor{green!15}0.0017 & \cellcolor{red!10}0.0014 & \cellcolor{green!15}2.663
& 2.08E-5
& 22.67
& 2.31
\\

\bottomrule
\end{tabular}}
\end{table}

\paragraph{Data Fidelity Constraint.}
Similar to the MNIST experiment, we train a AE on CIFAR10 and use its reconstruction consistency as a learned fidelity constraint to improve the quality of generated samples. Quantitative results and qualitative results are summarized in Table~\ref{tab:cifar_vae_metrics} and Figure~\ref{fig:cifar_ae_appendix}, respectively. Consistent with the MNIST results, applying our constraint correction improves the distribution matching metrics compared with unconstrained flow matching samples, indicating that the proposed framework can guide the pretrained model toward higher-fidelity samples without additional training. In contrast, existing constrained generation methods show limited performance under this nonlinear learned constraint. 
It is worth noting that the purpose of this experiment is not to optimize unconditional CIFAR10 generation performance, but to evaluate whether our framework can effectively incorporate learned constraints while preserving the distribution of a pretrained flow model.

\begin{table}[h]
\centering
\caption{Quantitative results on the AE-based fidelity constrained CIFAR10 dataset.}
\label{tab:cifar_vae_metrics}
\resizebox{\textwidth}{!}{
\begin{tabular}{lccccccc}
\toprule
& \multicolumn{4}{c}{Distribution Metrics }
& \multicolumn{1}{c}{Constraint}
& \multicolumn{2}{c}{Efficiency} \\
\cmidrule(lr){2-5} \cmidrule(lr){6-6} \cmidrule(lr){7-8}
Method 
& FID $\downarrow$
& KID $\downarrow$
& MMD $\downarrow$
& $W_2$ $\downarrow$
& $H$ $\downarrow$
& Time (s) $\downarrow$
& Memory (GB) $\downarrow$\\
\midrule
Unconstrained & 11.797 & 0.0053 & 0.0043 & 3.4347 & 2.18E-5 & 13.91& 0.72\\
DiffusionPDE  & 10.779 & 0.0058 & \cellcolor{red!10}0.0042 & \cellcolor{red!10}3.3389 & 1.06E-5 & \cellcolor{green!15}14.82 & \cellcolor{red!10}2.59\\
D-Flow & \cellcolor{red!10}9.984 & \cellcolor{red!10}0.0057 & \cellcolor{red!10}0.0042 & 3.3438 & \cellcolor{green!15}9.09E-6 & 1618.61 & 103.74 \\
ECI& 310.353 & 0.3391 & 0.3083 & 17.6227 & 3.17E-4 & 30.80 & \cellcolor{green!15}0.68 \\
PCFM& 116.953 & 0.0729 & 0.0578 & 10.8794 & 6.78E-4 & 49.53 & 4.60 \\
Ours & \cellcolor{green!15}9.795 &  \cellcolor{green!15}0.0040 & \cellcolor{green!15}0.0033 & \cellcolor{green!15}3.2832 & \cellcolor{red!10}1.04E-5 & \cellcolor{red!10}19.70 & 2.74\\
\bottomrule
\end{tabular}}
\end{table}

\newpage
\begin{figure}[H]
    \centering

    \begin{tikzpicture}
        \node[inner sep=0pt] (img) {
            \includegraphics[width=0.95\textwidth]{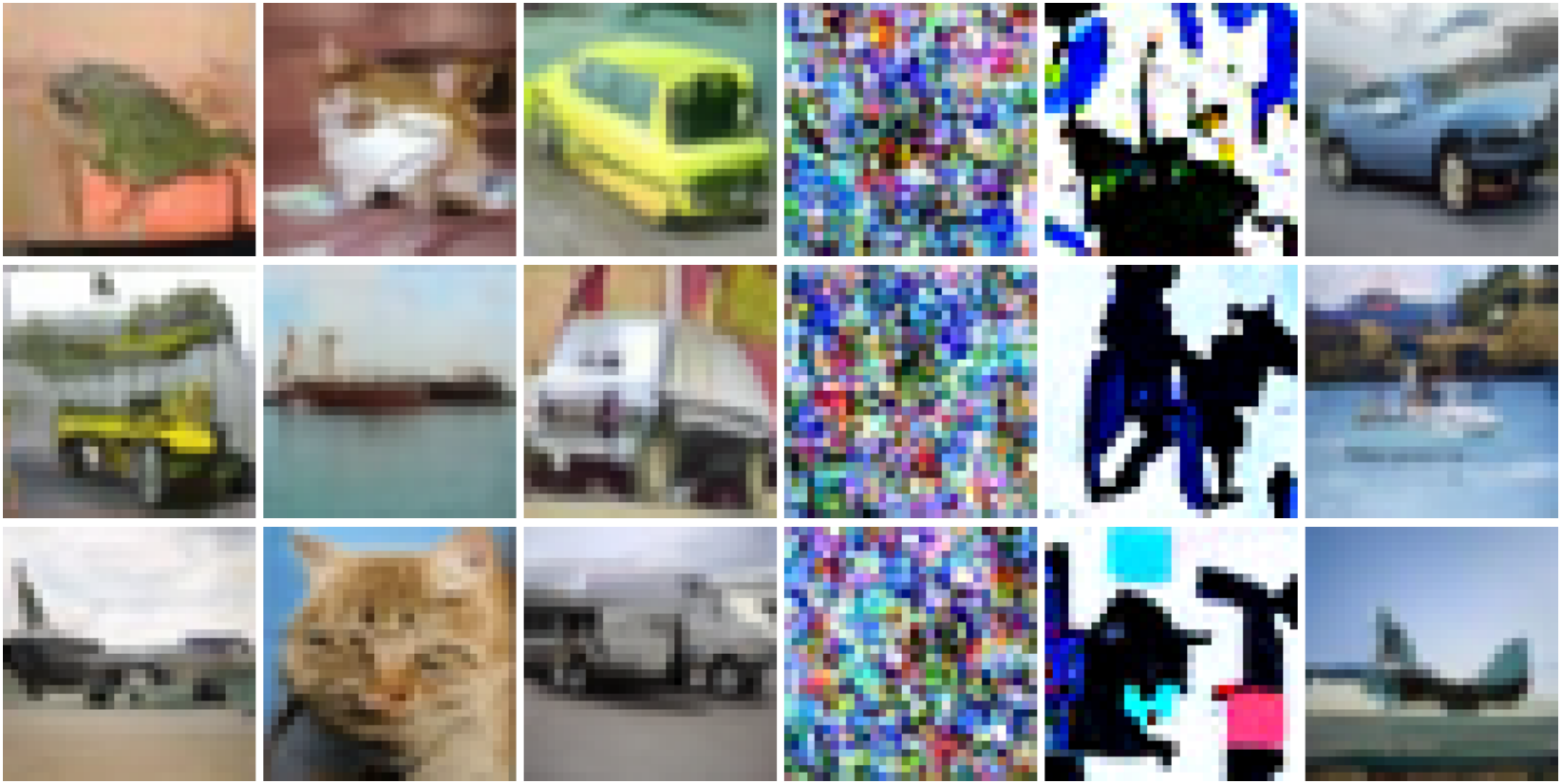}
        };

        \foreach \i/\label in {
            0/Unconstrained,
            1/DiffusionPDE,
            2/D-Flow,
            3/ECI,
            4/PCFM,
            5/\textbf{MintFlow}
        }{
            \node[anchor=south] at
                ($(img.north west)!\the\numexpr(2*\i+1)\relax/12!(img.north east)$)
                {\label};
        }

\node[
    rotate=90,
    anchor=south,
    fill={rgb,255:red,65;green,145;blue,210},
    text=white,
    rounded corners=1pt,
    inner xsep=3pt,
    inner ysep=2pt
] at ($(img.west)+(-0.35cm,0)$) {\small Blue Tone};

        \draw[
    line width=0.8pt,
    dash pattern=on 2pt off 2pt,
    black!70
]
            ($(img.south west)+(-0.7,-0.12cm)$) --
            ($(img.south east)+(+0.2,-0.12cm)$);
    \end{tikzpicture}

    \vspace{0.15cm}

    \begin{tikzpicture}
        \node[inner sep=0pt] (img) {
            \includegraphics[width=0.95\textwidth]{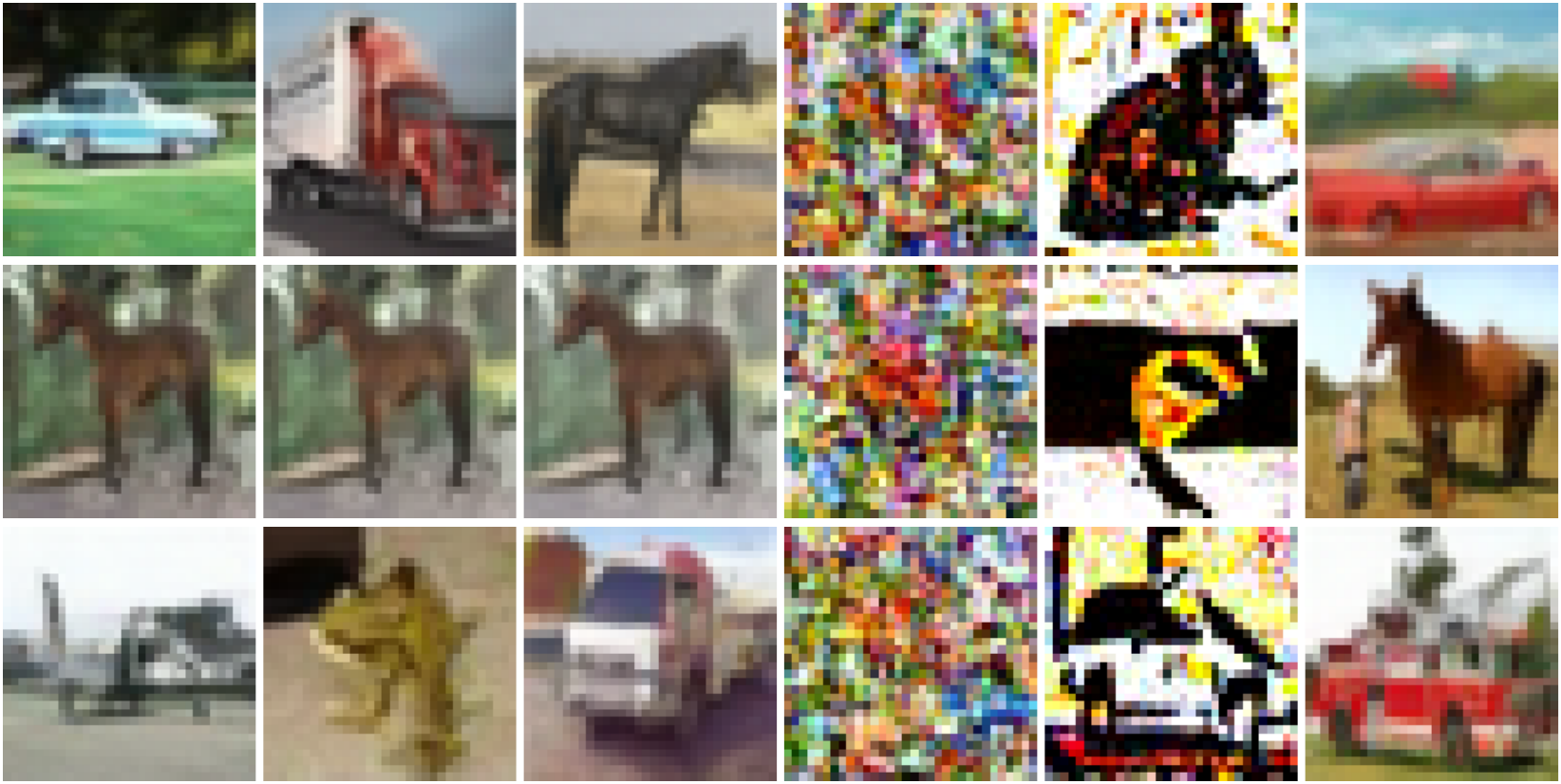}
        };

\node[
    rotate=90,
    anchor=south,
    fill={rgb,255:red,190;green,75;blue,75},
    text=white,
    rounded corners=1pt,
    inner xsep=3pt,
    inner ysep=2pt
] at ($(img.west)+(-0.35cm,0)$) {\small Red Tone};

        \draw[
    line width=0.8pt,
    dash pattern=on 2pt off 2pt,
    black!70
]
            ($(img.south west)+(-0.7,-0.12cm)$) --
            ($(img.south east)+(0.2,-0.12cm)$);
    \end{tikzpicture}

    \vspace{0.15cm}

    \begin{tikzpicture}
        \node[inner sep=0pt] (img) {
            \includegraphics[width=0.95\textwidth]{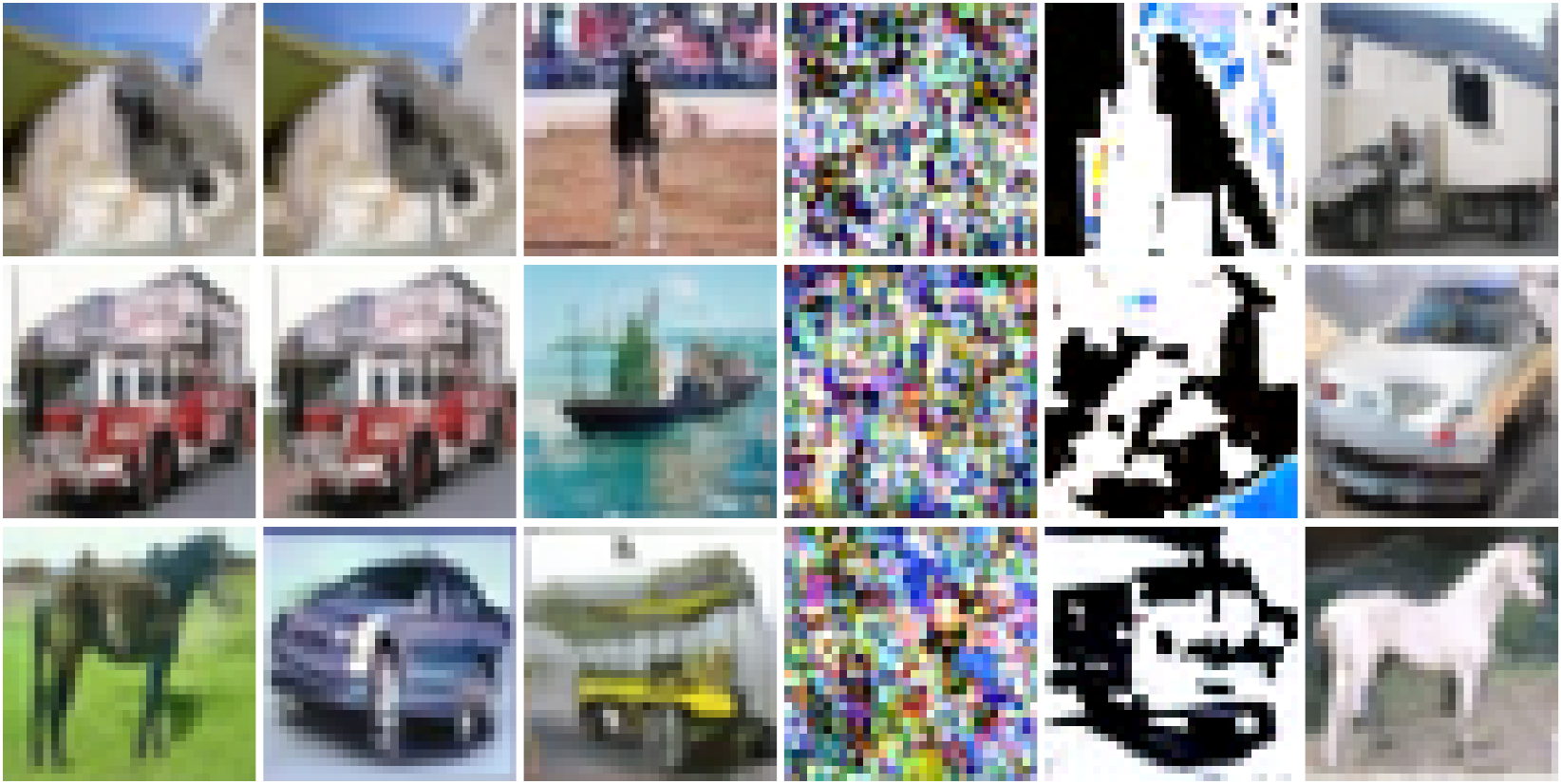}
        };

\node[
    rotate=90,
    anchor=south,
    fill={rgb,255:red,170;green,215;blue,235},
    text=white,
    rounded corners=1pt,
    inner xsep=3pt,
    inner ysep=2pt
] at ($(img.west)+(-0.35cm,0)$) {\small Cold Tone};

    \end{tikzpicture}

    \caption{Qualitative comparison on CIFAR-10 under different color constraints. We compare unconstrained samples with constrained samples generated by different methods. All samples in each row are generated from the same initial noise across methods.}
    \label{fig:cifar_intensity_appendix}
\end{figure}

\newpage
\begin{figure}[H]
    \centering
    \begin{tikzpicture}
        \node[inner sep=0pt] (img) {
            \includegraphics[width=0.95\textwidth]{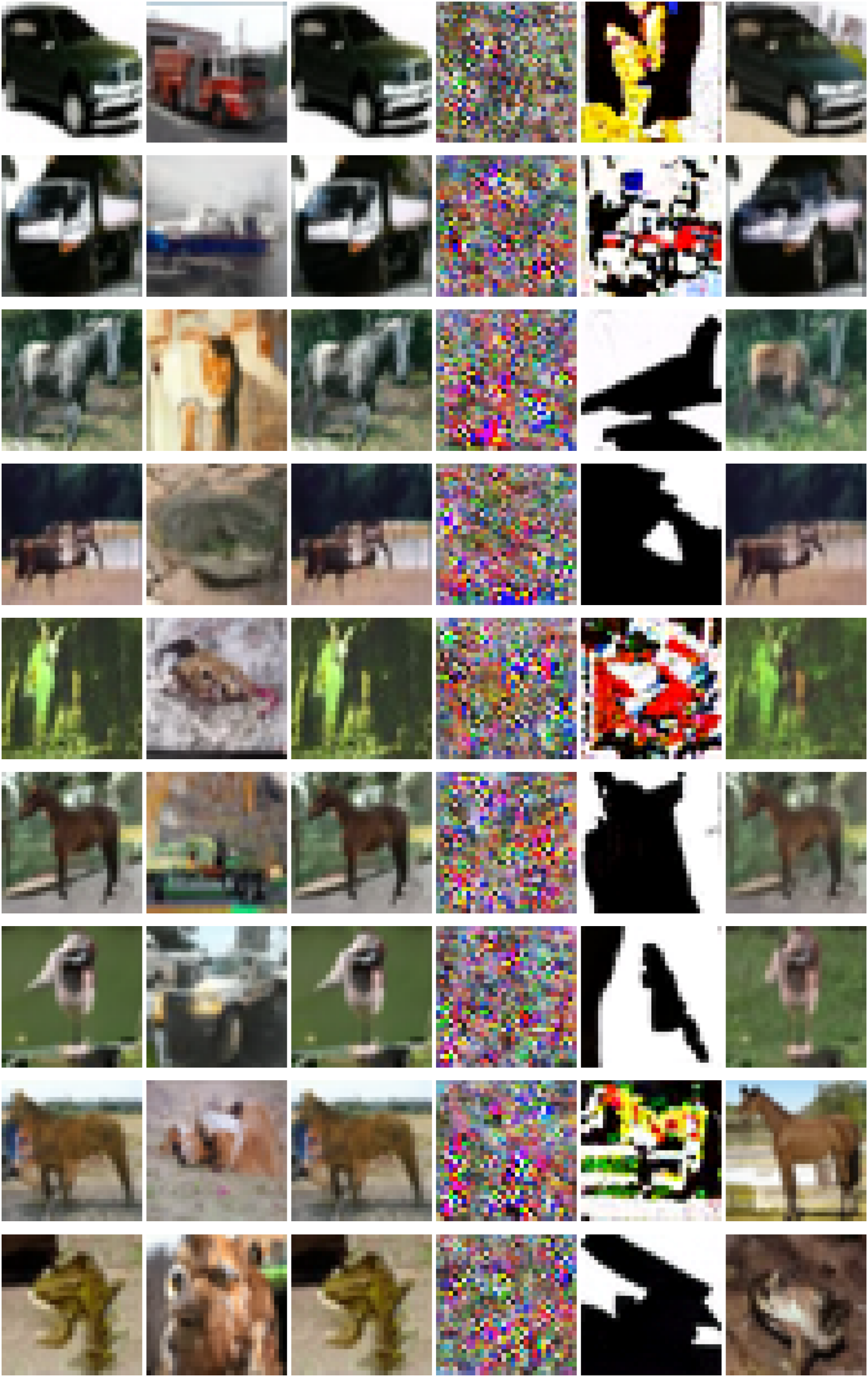}
        };

        \foreach \i/\label in {
            0/Unconstrained,
            1/DiffusionPDE,
            2/D-Flow,
            3/ECI,
            4/PCFM,
            5/\textbf{MintFlow}
        }{
            \node[anchor=south] at
                ($(img.north west)!\the\numexpr(2*\i+1)\relax/12!(img.north east)$)
                {\label};
        }
    \end{tikzpicture}
    \caption{Uncurated results for AE-based fidelity-constrained generation on CIFAR10. We compare unconstrained samples with constrained samples generated by different methods. All samples in each row are generated from the same initial noise across methods.}
    \label{fig:cifar_ae_appendix}
\end{figure}
\section{Application 2: Image Inverse Problems}
\label{appen:inverse_problems}
This appendix provides the benchmark definitions, model and sampling configurations, evaluation protocol, and additional diagnostics supporting the image inverse problem results in Section~\ref{sec:exp_image_inverse}. In particular, we separately quantify three aspects of reconstruction quality:
\begin{itemize}
\item \textbf{Constraint satisfaction:} how well the reconstructed image satisfies the measurement constraint associated with each inverse problem;
\item \textbf{Pairwise reconstruction fidelity:} how closely the reconstruction matches the corresponding reference image;
\item \textbf{Distributional fidelity:} how well the distribution of reconstructed images is preserved relative to the reference image distribution.
\end{itemize}

\subsection{Datasets and Constraints}
We evaluate the image inverse problems on two standard image datasets, \textbf{AFHQ-Cat} \citep{choi2020starganv2} and \textbf{FFHQ} \citep{karras2019style}. All images are resized to $64\times64$ before applying the inverse operators. 

For each reference image $u^{\mathrm{ref}}$, we construct a measurement from the corresponding forward operator and solve for a reconstruction $u$ that satisfies the resulting measurement constraint.

\paragraph{Inpainting.}
For inpainting, we consider a centered $20\times20$ box mask $M$. The measurement consists of the pixels outside the masked region. The constraint is defined as
\[
H(u)=M\odot\left(u-u^{\mathrm{ref}}\right),
\]
where $\odot$ denotes element-wise multiplication and $M$ denotes the binary mask selecting the constrained region. Thus, satisfying the constraint requires the reconstruction to agree with the reference image on the observed pixels, while the missing $20\times20$ region is inferred from the learned image prior.

\paragraph{Super-resolution.}
For super-resolution, we first downsample the $64\times64$ reference image by a factor of $4$, resulting in a $16\times16$ observation. Let $D$ denote the downsampling operator, $D:\mathbb{R}^{3\times64\times64}\rightarrow\mathbb{R}^{3\times16\times16}$. The measurement constraint is
\[
H(u)=D\left(u-u^{\mathrm{ref}}\right).
\]
Therefore, a valid reconstruction must reproduce the observed $16\times16$ image after downsampling, while the high-resolution details are determined by the learned image prior

\paragraph{Deblurring.}
For deblurring, we apply a Gaussian blur operator $B$ with kernel size $11$ and standard deviation $\sigma=1.0$. The corresponding constraint is 
\[
H(x) = B\left(u-u^{\mathrm{ref}}\right).
\]
The inverse problem therefore consists of recovering a sharp image whose blurred version agrees with the observed measurement.

\subsection{Flow-Matching Model and Baselines}

Following \citep{patel2024flowchef}, we conduct all experiments using the same pretrained Rectified Flow++ models \citep{lee2024improvingtrainingrectifiedflows}. RFM models are trained to make the flow matching trajectories as close to straight lines as possible, resulting in simpler and more direct generative flows. Owing to this property, pretrained rectified flow models have been widely adopted in recent approaches for solving inverse problems with flow-based generative models. For all methods, we use 10-step Rectified Flow sampling to ensure a consistent generation budget across methods.

We compare our method against several existing approaches for solving inverse problems with flow-based generative models.
We first consider \textbf{FlowChef} \citep{patel2024flowchef} and \textbf{D-Flow} \citep{ben2024d}, which enforce constraints by optimizing either an intermediate state along the generative trajectory or the initial Gaussian noise, respectively. We additionally compare against \textbf{PnP-Flow} \citep{martin2025pnpflow}, a plug-and-play approach that incorporates the measurement operator into the flow-based sampling process.
Additionally, we include two diffusion-based baselines, DPS \citep{chung2024dps} and FreeDoM \citep{yu2023freedom}, using the Rectified Flow-compatible implementations developed in \citep{patel2024flowchef}. For all baselines, we use the implementations provided in the official FlowChef repository\footnote{\url{https://github.com/FlowChef/flowchef}} without modification. We also use the hyperparameters specified in \citep{patel2024flowchef} for each baseline. Our method is implemented for RFM following the formulation described in Appendix~\ref{appen:rectified_flow}. Unlike our previous formulation developed for general Flow Matching models, our RFM-based implementation exploits the linearity of the RFM flow path. This property allows us to obtain the optimal correction without explicitly solving the adjoint equation backward in time, as described in (\ref{eq:rf_constraint_jacobian}).

\subsection{Evaluation Protocol}
For AFHQ-Cat, we use all images in the official test split. For FFHQ, we randomly select $10{,}000$ images from the dataset using a fixed random seed and use the same set of reference images for all inverse-problem evaluations. For each dataset and each inverse-problem task, we evaluate the reconstructions from four complementary perspectives.

\paragraph{Constraint Satisfaction.}
For each inverse-problem task, we evaluate how well the reconstructed images satisfy the corresponding measurement constraint. Specifically, given a reconstruction $\hat{u}$, we compute the mean squared error of the constraint residual $H(\hat{u})$ over the entire evaluation set. A lower value indicates better satisfaction of the imposed constraint.

\paragraph{Pairwise Reconstruction Fidelity.}
We measure the reconstruction fidelity by comparing each restored image $\hat{u}$ with its corresponding reference image $u^{\mathrm{ref}}$, i.e., the image before applying the degradation operator. We report PSNR, SSIM, and LPIPS to evaluate pixel-level, structural, and perceptual similarity, respectively. Higher PSNR and SSIM and lower LPIPS indicate better reconstruction fidelity.

\paragraph{Distributional Fidelity.}
We further evaluate whether the restored images preserve the overall image distribution. To this end, we compute the Fréchet Inception Distance (FID) between the set of restored images and the ground-truth images of AFHQ-Cat and FFHQ. A lower FID indicates that the distribution of the restored images is closer to that of the GT distribution.

\paragraph{Efficiency.}
Runtime is measured as synchronized wall-clock time per generated sample, and memory is reported as peak allocated GPU memory.

\subsection{Additional Qualitative Results}\label{appen:additional_images_inverse}
Further qualitative results on the AFHQ-Cat dataset for inpainting, super-resolution, and deblurring are presented in Figures~\ref{fig:cat_inpaint_appendix}, \ref{fig:cat_sr_appendix}, and \ref{fig:cat_deblur_appendix}, respectively. Corresponding results on the FFHQ dataset are provided in Figures~\ref{fig:ffhq_inpaint_appendix}, \ref{fig:ffhq_sr_appendix}, and \ref{fig:ffhq_deblur_appendix}, respectively.

\newpage
\begin{figure}[H]
    \centering
    \begin{tikzpicture}
        \node[inner sep=0pt] (img) {
            \includegraphics[width=\textwidth]{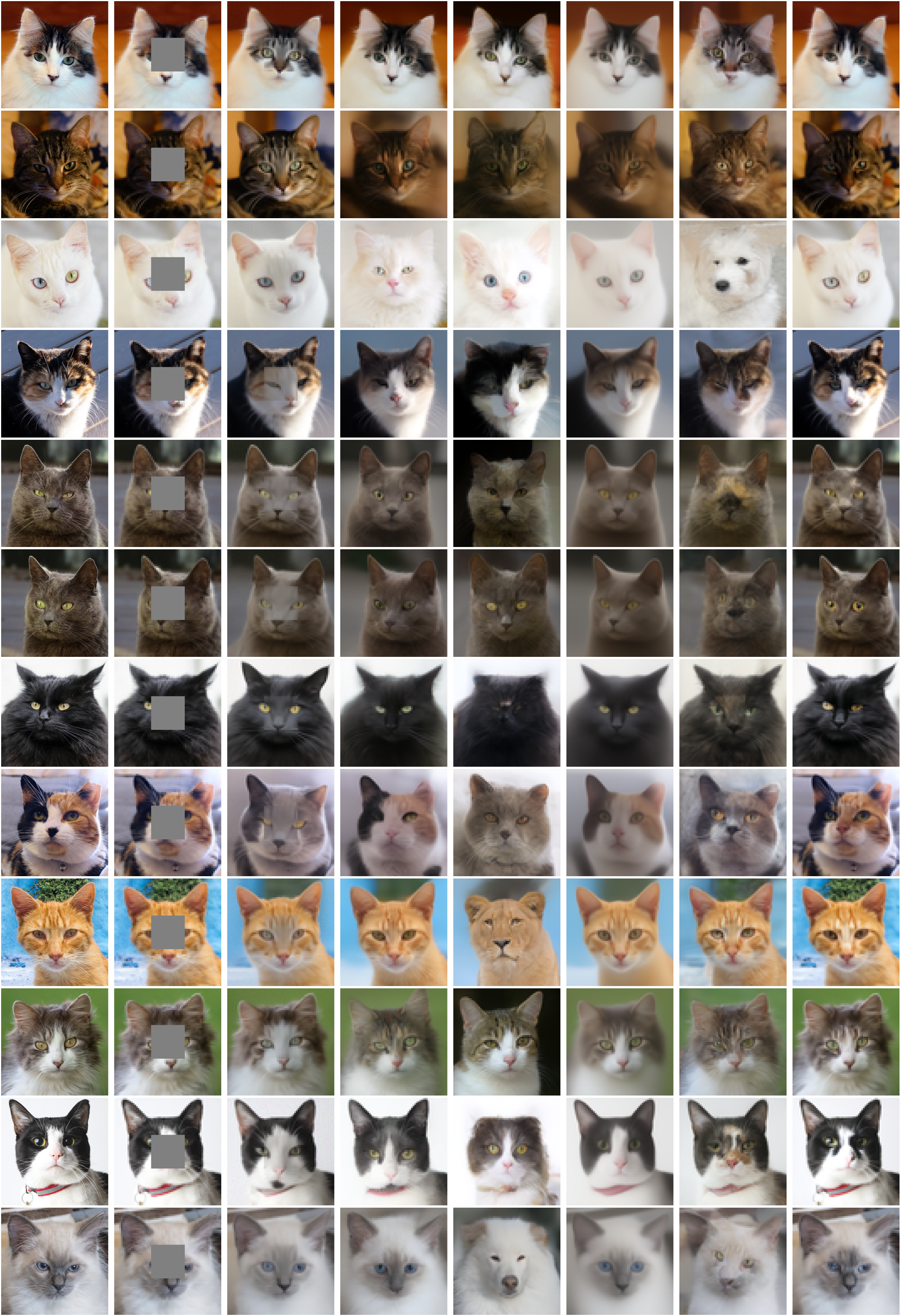}
        };

        \foreach \i/\label in {
            0/GT,
            1/Degraded,
            2/D-Flow,
            3/DPS,
            4/Freedom,
            5/PnPFlow,
            6/FlowChef
        }{
            \node[anchor=south] at
                ($(img.north west)!\the\numexpr\i*2+1\relax/16!(img.north east)$)
                {\label};
        }

        \node[anchor=south] at
            ($(img.north west)!15/16!(img.north east)$)
            {\textbf{MintFlow}};
    \end{tikzpicture}
    \caption{Uncurated inpainting results on the AFHQ-Cat dataset. The figure compares ground-truth images (GT), degraded inputs, and corresponding inpainting results from different methods.}
    \label{fig:cat_inpaint_appendix}
\end{figure}

\newpage
\begin{figure}[H]
    \centering
    \begin{tikzpicture}
        \node[inner sep=0pt] (img) {
            \includegraphics[width=\textwidth]{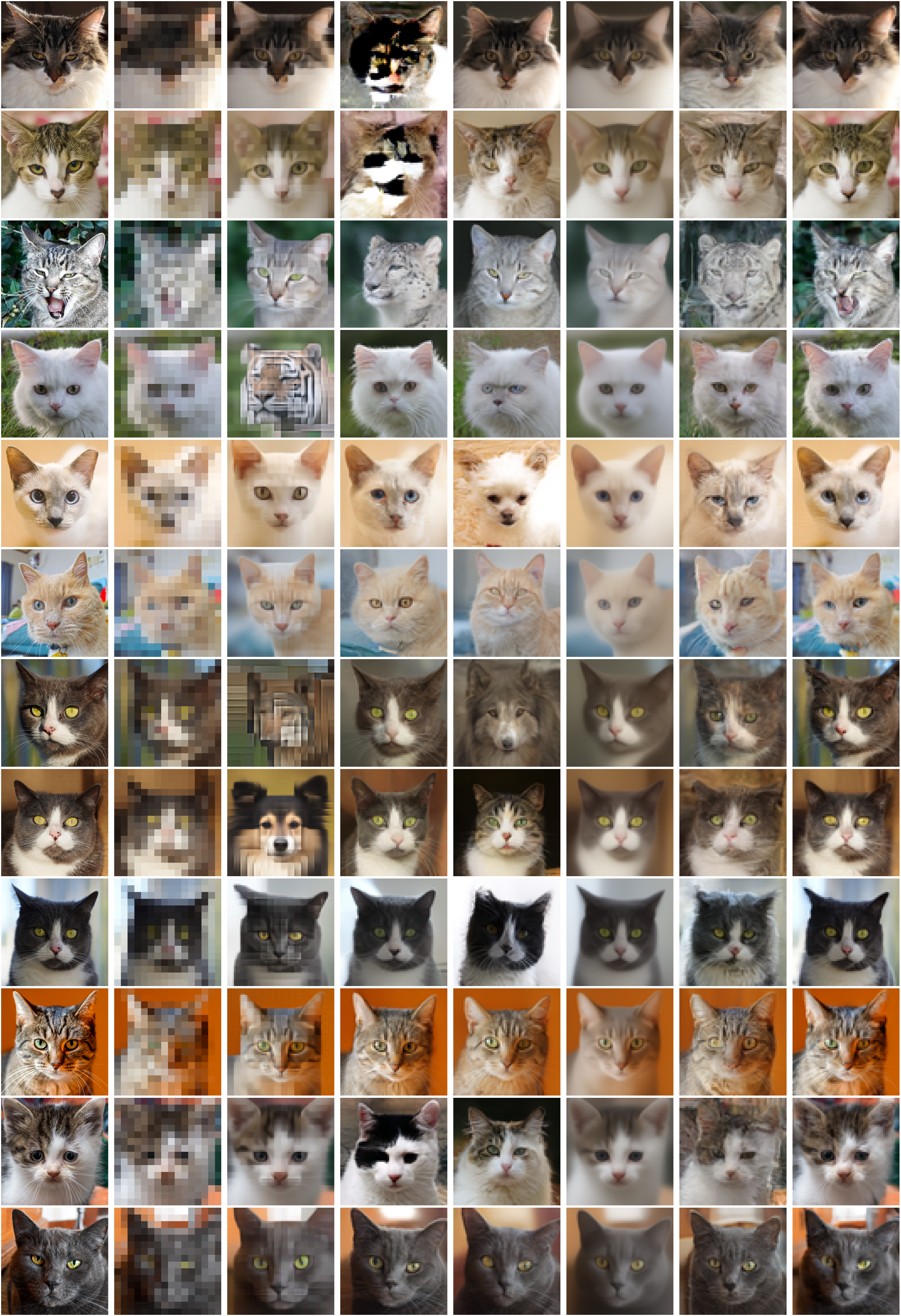}
        };

        \foreach \i/\label in {
            0/GT,
            1/Degraded,
            2/D-Flow,
            3/DPS,
            4/Freedom,
            5/PnPFlow,
            6/FlowChef
        }{
            \node[anchor=south] at
                ($(img.north west)!\the\numexpr\i*2+1\relax/16!(img.north east)$)
                {\label};
        }

        \node[anchor=south] at
            ($(img.north west)!15/16!(img.north east)$)
            {\textbf{MintFlow}};
    \end{tikzpicture}
    \caption{Uncurated super resolution (SR) results on the AFHQ-Cat dataset. The figure compares ground-truth images (GT), degraded inputs, and corresponding SR results from different methods.}
    \label{fig:cat_sr_appendix}
\end{figure}

\newpage
\begin{figure}[H]
    \centering
    \begin{tikzpicture}
        \node[inner sep=0pt] (img) {
            \includegraphics[width=\textwidth]{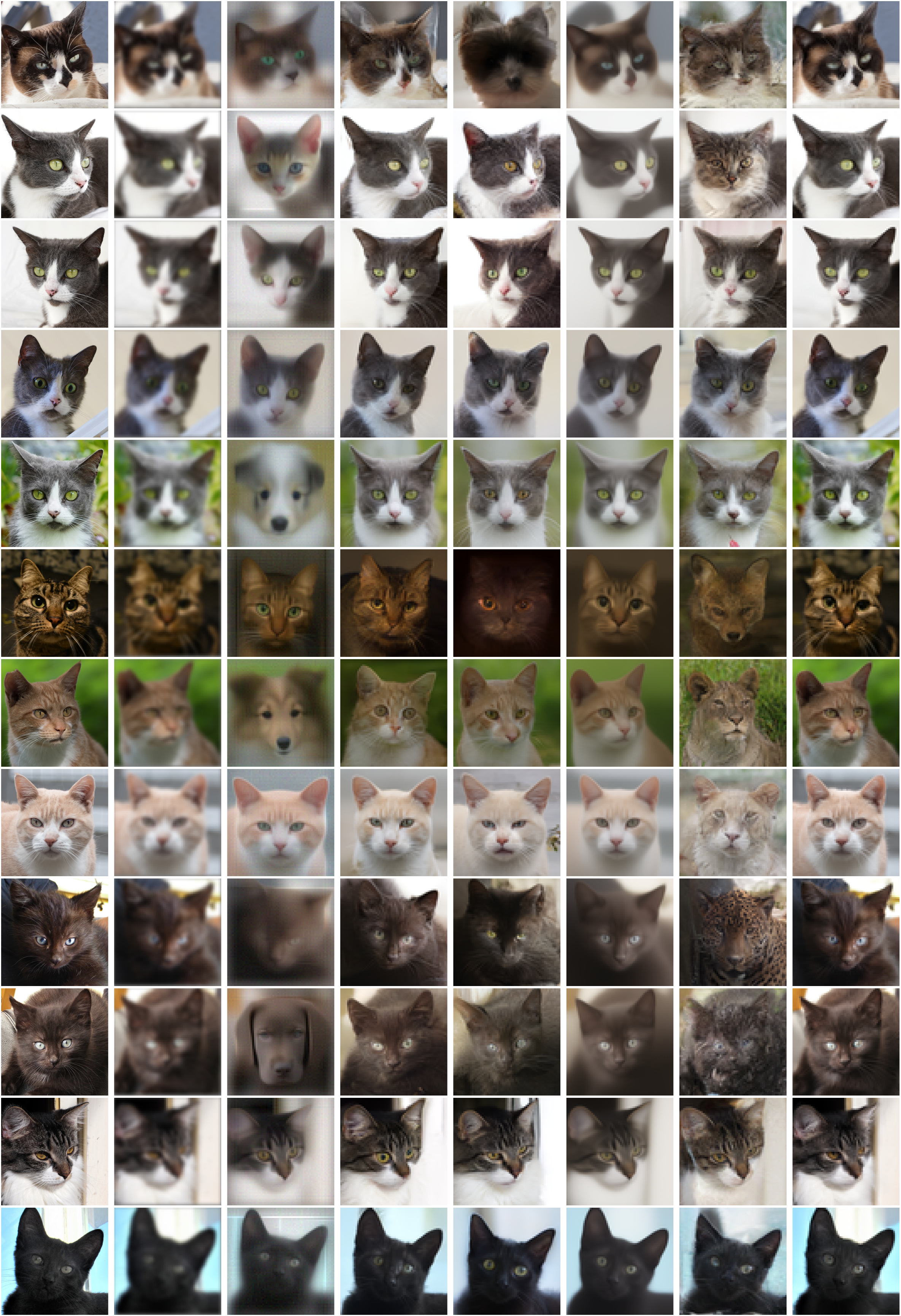}
        };

        \foreach \i/\label in {
            0/GT,
            1/Degraded,
            2/D-Flow,
            3/DPS,
            4/Freedom,
            5/PnPFlow,
            6/FlowChef
        }{
            \node[anchor=south] at
                ($(img.north west)!\the\numexpr\i*2+1\relax/16!(img.north east)$)
                {\label};
        }

        \node[anchor=south] at
            ($(img.north west)!15/16!(img.north east)$)
            {\textbf{MintFlow}};
    \end{tikzpicture}
    \caption{Uncurated deblurring results on the AFHQ-Cat dataset. The figure compares ground-truth images (GT), degraded inputs, and corresponding deblurred results from different methods.}
    \label{fig:cat_deblur_appendix}
\end{figure}

\newpage
\begin{figure}[H]
    \centering
    \begin{tikzpicture}
        \node[inner sep=0pt] (img) {
            \includegraphics[width=\textwidth]{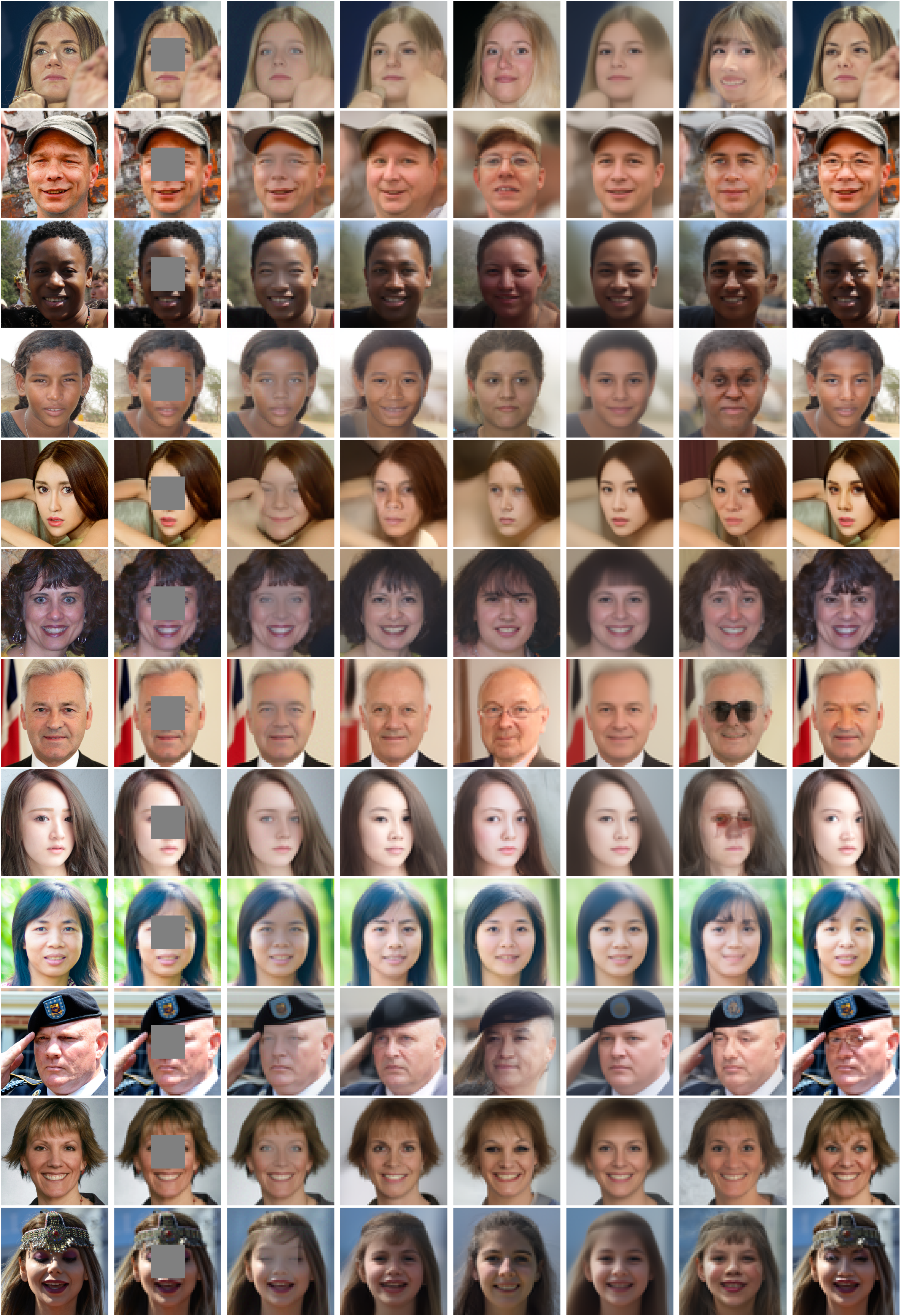}
        };

        \foreach \i/\label in {
            0/GT,
            1/Degraded,
            2/D-Flow,
            3/DPS,
            4/Freedom,
            5/PnPFlow,
            6/FlowChef
        }{
            \node[anchor=south] at
                ($(img.north west)!\the\numexpr\i*2+1\relax/16!(img.north east)$)
                {\label};
        }

        \node[anchor=south] at
            ($(img.north west)!15/16!(img.north east)$)
            {\textbf{MintFlow}};
    \end{tikzpicture}
    \caption{Uncurated inpainting results on the FFHQ dataset. The figure compares ground-truth images (GT), degraded inputs, and corresponding inpainting results from different methods.}
    \label{fig:ffhq_inpaint_appendix}
\end{figure}

\newpage
\begin{figure}[H]
    \centering
    \begin{tikzpicture}
        \node[inner sep=0pt] (img) {
            \includegraphics[width=\textwidth]{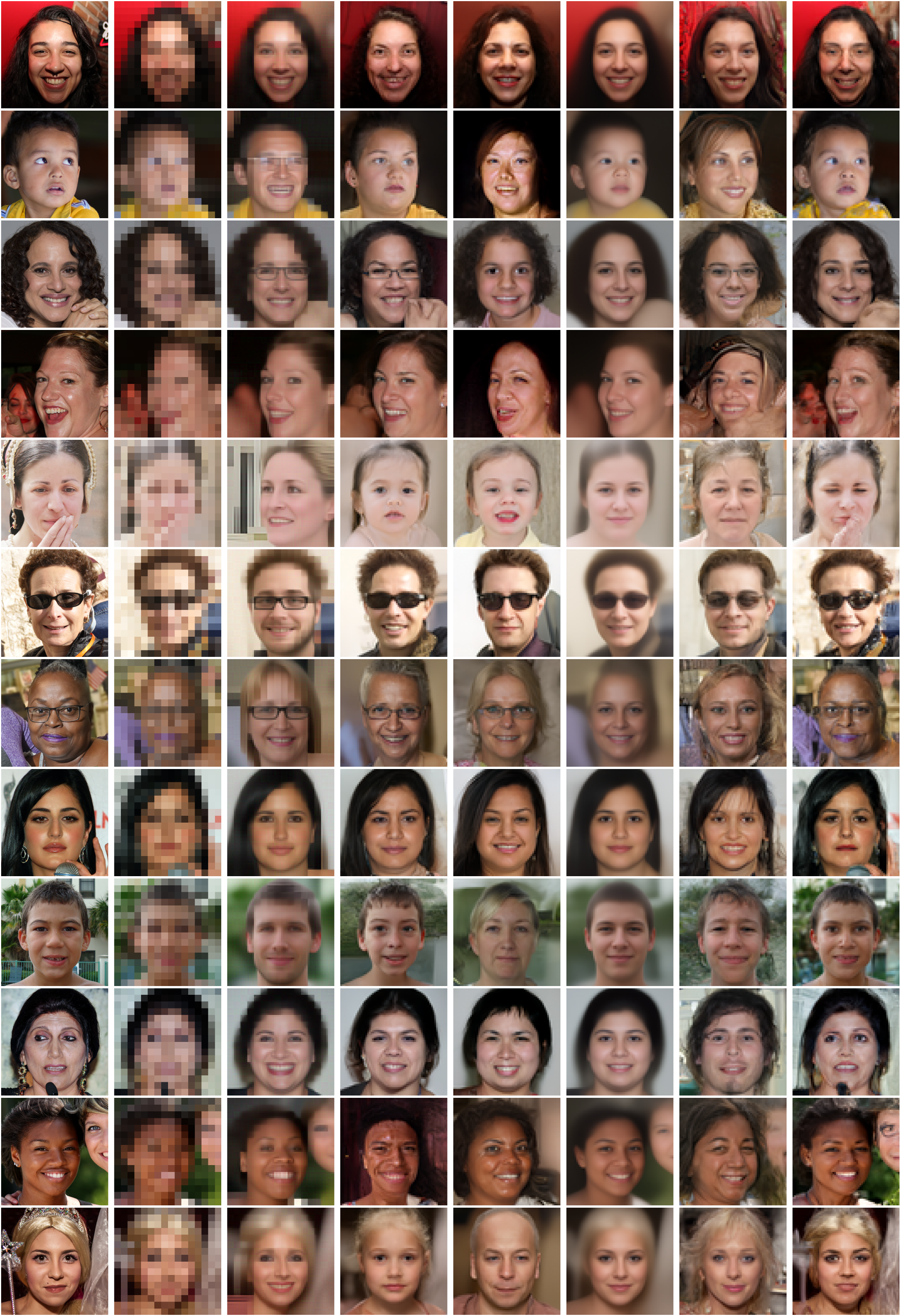}
        };

        \foreach \i/\label in {
            0/GT,
            1/Degraded,
            2/D-Flow,
            3/DPS,
            4/Freedom,
            5/PnPFlow,
            6/FlowChef
        }{
            \node[anchor=south] at
                ($(img.north west)!\the\numexpr\i*2+1\relax/16!(img.north east)$)
                {\label};
        }

        \node[anchor=south] at
            ($(img.north west)!15/16!(img.north east)$)
            {\textbf{MintFlow}};
    \end{tikzpicture}
    \caption{Uncurated super resolution (SR) results on the FFHQ dataset. The figure compares ground-truth images (GT), degraded inputs, and corresponding SR results from different methods.}
    \label{fig:ffhq_sr_appendix}
\end{figure}

\newpage
\begin{figure}[H]
    \centering
    \begin{tikzpicture}
        \node[inner sep=0pt] (img) {
            \includegraphics[width=\textwidth]{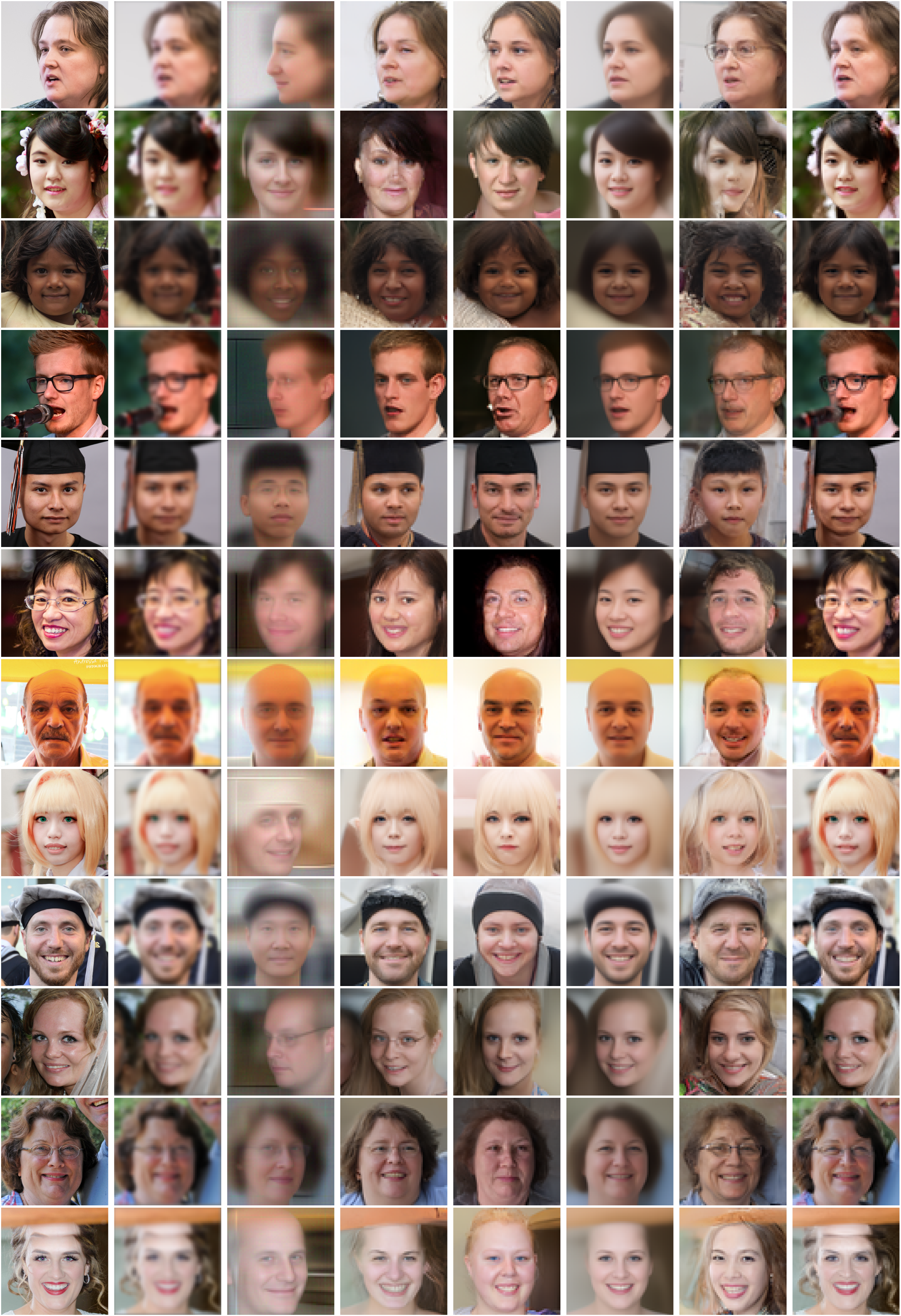}
        };

        \foreach \i/\label in {
            0/GT,
            1/Degraded,
            2/D-Flow,
            3/DPS,
            4/Freedom,
            5/PnPFlow,
            6/FlowChef
        }{
            \node[anchor=south] at
                ($(img.north west)!\the\numexpr\i*2+1\relax/16!(img.north east)$)
                {\label};
        }

        \node[anchor=south] at
            ($(img.north west)!15/16!(img.north east)$)
            {\textbf{MintFlow}};
    \end{tikzpicture}
    \caption{Uncurated deblurring results on the FFHQ dataset. The figure compares ground-truth images (GT), degraded inputs, and corresponding deblurred results from different methods.}
    \label{fig:ffhq_deblur_appendix}
\end{figure}
\section{Application 3: Image Editing}\label{appen:image_editing}
This appendix provides the benchmark definitions, model and sampling configurations, evaluation protocol, and additional diagnostics supporting the image editing results in Section~\ref{sec:exp_image_editing}. In particular, we quantify not only how well the edited images satisfy the given constraints and the computational efficiency of each method, but also how well they fulfill the task-specific text guidance.
\begin{itemize}
\item \textbf{Constraint satisfaction:} how well the edited image satisfies the constraint, i.e., preserving the regions outside the provided mask;
\item \textbf{Prompt adherence:} how well the edited image follows the given text prompt, measured using CLIP score;
\item \textbf{Image similarity:} how visually consistent the edited image remains with the reference image, measured using CLIP score;
\item \textbf{Efficiency:} the computational efficiency of each method in terms of inference speed and memory usage.
\end{itemize}

\subsection{Datasets and Constraints}
We perform large-scale image editing experiments on PIE-Bench \citep{ju2024pnp}. Specifically, we use the images from the \texttt{0\_random\_140} subset available in the official PIE-Bench data repository, along with the provided ground-truth masks. We use the same masks to control all methods for a fair comparison. The editing task is subject to an inpainting-style constraint: regions outside the provided mask should remain identical to the corresponding regions in the reference image. Unlike the aforementioned inverse problem, however, the edited region is additionally required to follow the given text guidance, making it challenging to satisfy both the preservation constraint and the desired semantic edit. Thus, the key challenge of this task is to simultaneously preserve the unmasked regions faithfully while performing the intended text-guided edit within the masked region.

\subsection{Flow-Matching Model and Baselines}
Following \citep{patel2024flowchef}, we use InstaFlow \citep{liu2023instaflow}, a Stable Diffusion v1.5 variant, as the flow-matching model. We evaluate the same set of baselines used for the aforementioned inverse problems in Appendix~\ref{appen:inverse_problems} on InstaFlow. As observed in the previous experiments, D-Flow requires differentiating through the entire forward flow ODE integration, resulting in substantially higher memory consumption. For the image editing task on InstaFlow, this computational overhead exceeds the available memory in our experimental setup. We therefore exclude D-Flow from this experiment.

\subsection{Evaluation Protocol}

We evaluate the editing results using the following metrics.

\paragraph{Constraint Satisfaction.}
We measure how faithfully the edited image preserves the reference image outside the provided mask. Specifically, we compute the mean squared error (MSE) between the edited and reference images over the unmasked regions.

\paragraph{Prompt Adherence.}
We measure how well the edited image follows the given text prompt using CLIP-based similarity \citep{radford2021learning}. Specifically, we use the OpenAI-pretrained CLIP ViT-B/32 model implemented in OpenCLIP \footnote{\url{https://github.com/mlfoundations/open_clip?utm_source=chatgpt.com}}. We report \textbf{CLIP-T}, computed as the cosine similarity between the CLIP image embedding of the edited image and the CLIP text embedding of the corresponding prompt. Higher scores indicate better text-image alignment.

\paragraph{Image Similarity.}
We measure the visual similarity between the edited image and the reference image using \textbf{CLIP-I}. Specifically, CLIP-I is computed as the cosine similarity between their CLIP image embeddings using the same OpenAI-pretrained CLIP ViT-B/32 model. Higher scores indicate greater visual similarity to the reference image.

\paragraph{Efficiency.}
We measure the computational cost of constrained sampling in terms of inference time and peak GPU memory allocation. Inference time measures the wall-clock time required to generate an edited image, while peak memory allocation measures the maximum GPU memory allocated during constrained sampling.

\subsection{Additional Qualitative Results}\label{appen:additional_images_editing}
Additional qualitative results for image editing are presented in Figures \ref{fig:image_edit_appen} and \ref{fig:image_edit_appen2}. Overall, we observe that MintFlow successfully modifies the image according to the given prompt while faithfully preserving the original input image. In contrast, baseline methods tend to overly focus on the prompt instructions, resulting in the loss of the input background, artifacts, or blurry outputs. These results demonstrate that our method effectively maintains the image distribution to generate natural in-distribution images, while successfully satisfying constraints to naturally edit the image according to the prompt while retaining the input context.

\newpage
\begin{figure}[H]
    \centering

    \begin{tikzpicture}


        \node[inner sep=0] (img1) {
            \includegraphics[
                width=\linewidth
            ]{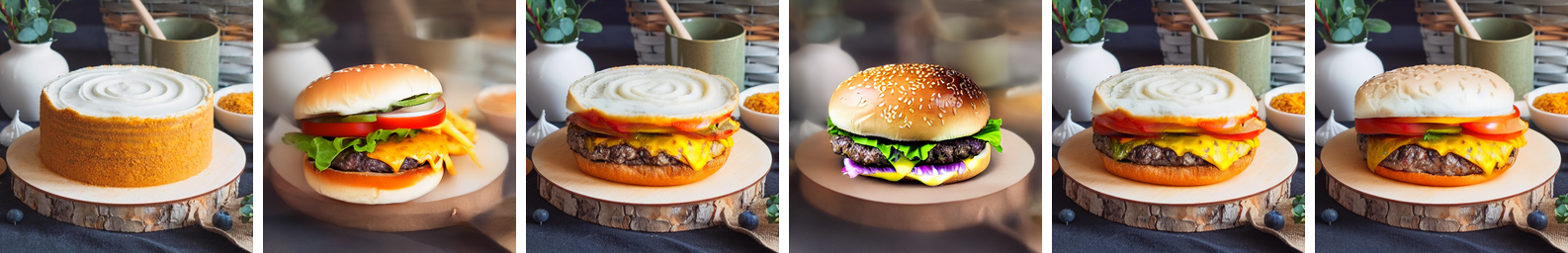}
        };

        \node[
            font=\small
        ] at ($(img1.north west)!1/12!(img1.north east) + (0, 7pt)$)
        {Input};

        \node[
            font=\small
        ] at ($(img1.north west)!3/12!(img1.north east) + (0, 7pt)$)
        {DPS};

        \node[
            font=\small
        ] at ($(img1.north west)!5/12!(img1.north east) + (0, 7pt)$)
        {Freedom};

        \node[
            font=\small
        ] at ($(img1.north west)!7/12!(img1.north east) + (0, 7pt)$)
        {PnPFlow};

        \node[
            font=\small
        ] at ($(img1.north west)!9/12!(img1.north east) + (0, 7pt)$)
        {FlowChef};

        \node[
            font=\small\bfseries
        ] at ($(img1.north west)!11/12!(img1.north east) + (0, 7pt)$)
        {MintFlow};


        \node[
            font=\small,
            anchor=north
        ] (text1) at ($(img1.south) + (0,-3pt)$)
        {
            \tikz[baseline=-1.5ex]{
        \node[
            fill=gray!15,
            rounded corners=2pt,
            inner xsep=3pt,
            inner ysep=1pt
        ] {\textcolor{gray}{\textit{Prompt:}}};
    }
    \;
            \texttt{a \sout{cake} \textcolor{red}{burger} on a wooden tray}
        };


        \node[
            inner sep=0,
            anchor=north
        ] (img2) at ($(text1.south) + (0,-3pt)$) {
            \includegraphics[
                width=\linewidth
            ]{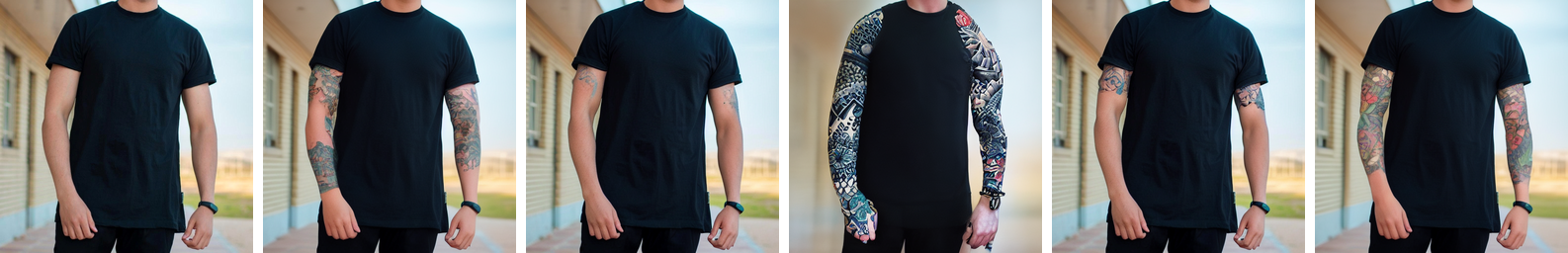}
        };


        \node[
            font=\small,
            anchor=north
        ] (text2) at ($(img2.south) + (0,-3pt)$)
     {
    \tikz[baseline=-1.5ex]{
        \node[
            fill=gray!15,
            rounded corners=2pt,
            inner xsep=3pt,
            inner ysep=1pt
        ] {\textcolor{gray}{\textit{Prompt:}}};
    }
    \;
    \texttt{a person wearing a black T-shirt with \sout{bare arms} \textcolor{red}{tattooed arms}}
};


        \node[
            inner sep=0,
            anchor=north
        ] (img3) at ($(text2.south) + (0,-3pt)$) {
            \includegraphics[
                width=\linewidth
            ]{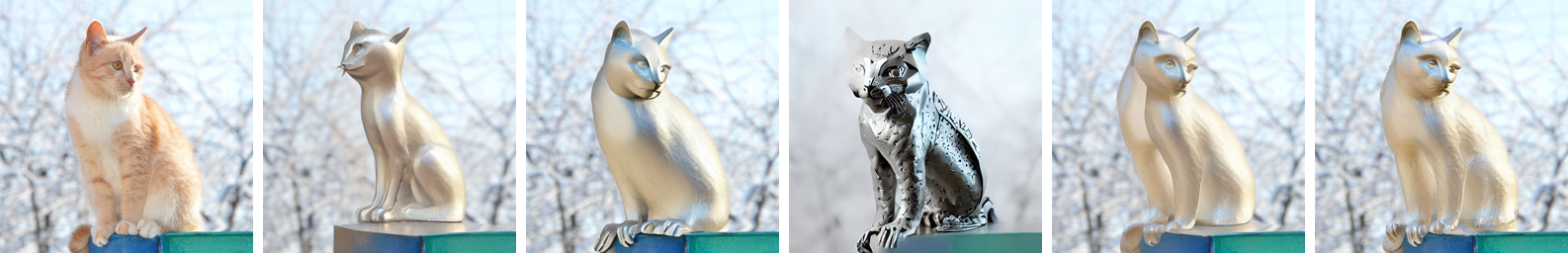}
        };

        \node[
            font=\small,
            anchor=north
        ] (text3) at ($(img3.south) + (0,-3pt)$)
     {
    \tikz[baseline=-1.5ex]{
        \node[
            fill=gray!15,
            rounded corners=2pt,
            inner xsep=3pt,
            inner ysep=1pt
        ] {\textcolor{gray}{\textit{Prompt:}}};
    }
    \;
    \texttt{a \sout{sitting cat} \textcolor{red}{silver cat sculpture}}
};


        \node[
            inner sep=0,
            anchor=north
        ] (img4) at ($(text3.south) + (0,-3pt)$) {
            \includegraphics[
                width=\linewidth
            ]{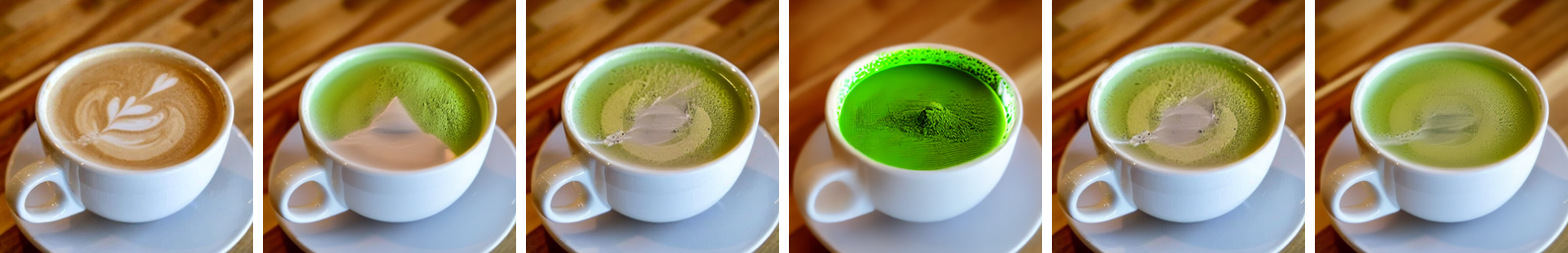}
        };

        \node[
            font=\small,
            anchor=north
        ] (text4) at ($(img4.south) + (0,-3pt)$)
     {
    \tikz[baseline=-1.5ex]{
        \node[
            fill=gray!15,
            rounded corners=2pt,
            inner xsep=3pt,
            inner ysep=1pt
        ] {\textcolor{gray}{\textit{Prompt:}}};
    }
    \;
    \texttt{a cup of \sout{latte} \textcolor{red}{matcha}}
};


        \node[
            inner sep=0,
            anchor=north
        ] (img5) at ($(text4.south) + (0,-3pt)$) {
            \includegraphics[
                width=\linewidth
            ]{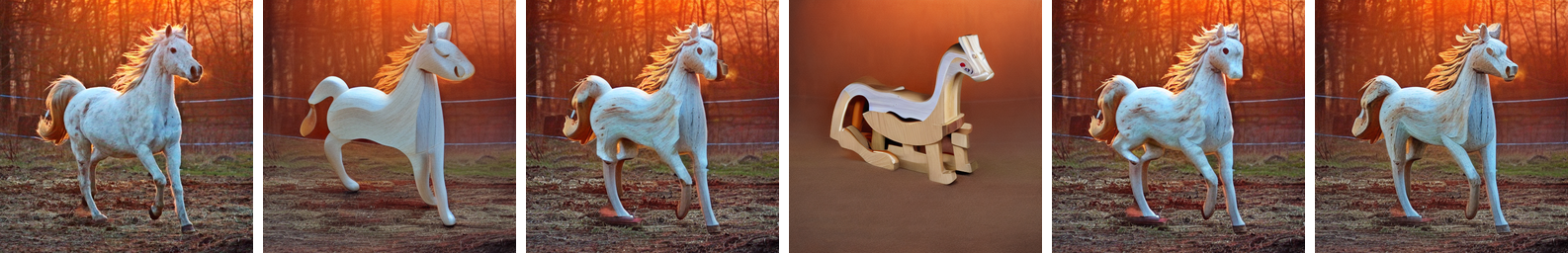}
        };

        \node[
            font=\small,
            anchor=north
        ] (text5) at ($(img5.south) + (0,-3pt)$)
     {
    \tikz[baseline=-1.5ex]{
        \node[
            fill=gray!15,
            rounded corners=2pt,
            inner xsep=3pt,
            inner ysep=1pt
        ] {\textcolor{gray}{\textit{Prompt:}}};
    }
    \;
    \texttt{a running \sout{horse} \textcolor{red}{wooden horse statue}}
};


        \node[
            inner sep=0,
            anchor=north
        ] (img6) at ($(text5.south) + (0,-3pt)$) {
            \includegraphics[
                width=\linewidth
            ]{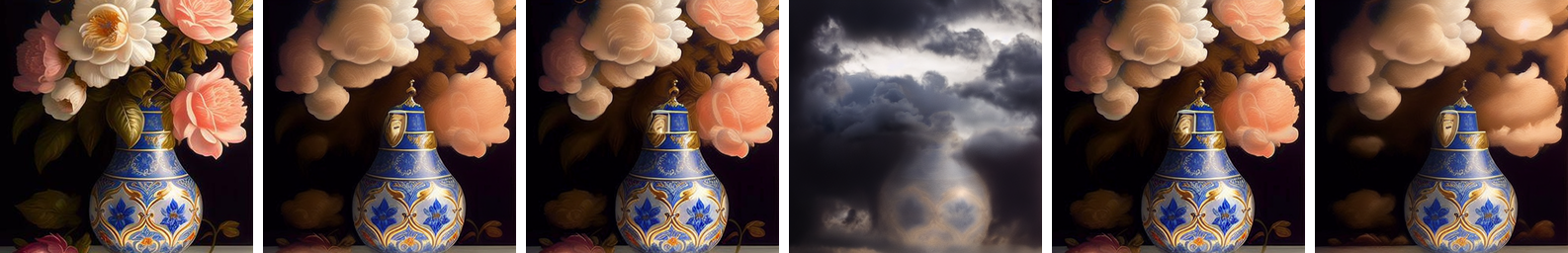}
        };

        \node[
            font=\small,
            anchor=north
        ] (text6) at ($(img6.south) + (0,-3pt)$)
     {
    \tikz[baseline=-1.5ex]{
        \node[
            fill=gray!15,
            rounded corners=2pt,
            inner xsep=3pt,
            inner ysep=1pt
        ] {\textcolor{gray}{\textit{Prompt:}}};
    }
    \;
   \texttt{a flower pot with \sout{flowers} \textcolor{red}{clouds}}
};

    \end{tikzpicture}
\caption{
Additional qualitative comparison of image editing results. Given an input image and its corresponding text prompt, we compare the editing results produced by different methods. Strikethrough and red text indicate the original and edited text, respectively.}
\label{fig:image_edit_appen}
\end{figure}

\newpage
\begin{figure}[H]
    \centering

    \begin{tikzpicture}


        \node[inner sep=0] (img1) {
            \includegraphics[
                width=\linewidth
            ]{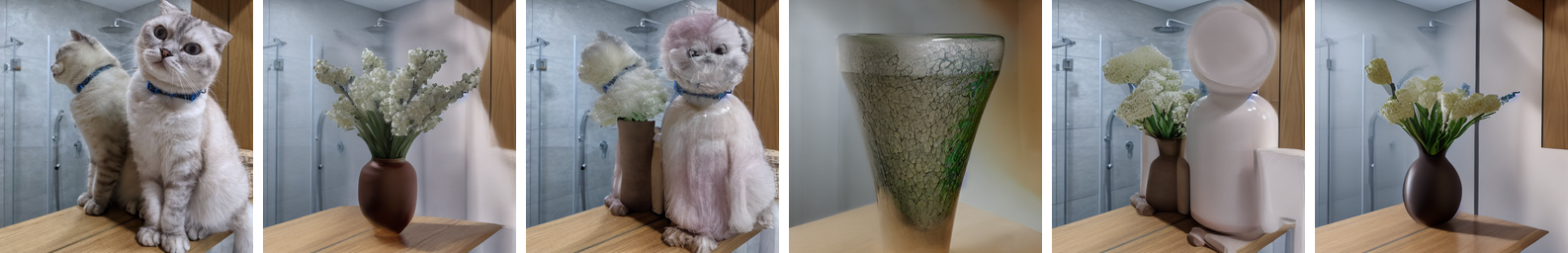}
        };

        \node[
            font=\small
        ] at ($(img1.north west)!1/12!(img1.north east) + (0, 7pt)$)
        {Input};

        \node[
            font=\small
        ] at ($(img1.north west)!3/12!(img1.north east) + (0, 7pt)$)
        {DPS};

        \node[
            font=\small
        ] at ($(img1.north west)!5/12!(img1.north east) + (0, 7pt)$)
        {Freedom};

        \node[
            font=\small
        ] at ($(img1.north west)!7/12!(img1.north east) + (0, 7pt)$)
        {PnPFlow};

        \node[
            font=\small
        ] at ($(img1.north west)!9/12!(img1.north east) + (0, 7pt)$)
        {FlowChef};

        \node[
            font=\small\bfseries
        ] at ($(img1.north west)!11/12!(img1.north east) + (0, 7pt)$)
        {MintFlow};


        \node[
            font=\small,
            anchor=north
        ] (text1) at ($(img1.south) + (0,-3pt)$)
        {
            \tikz[baseline=-1.5ex]{
        \node[
            fill=gray!15,
            rounded corners=2pt,
            inner xsep=3pt,
            inner ysep=1pt
        ] {\textcolor{gray}{\textit{Prompt:}}};
    }
    \;
            \texttt{a \sout{cat} \textcolor{red}{vase} sitting by the shower glass}
        };


        \node[
            inner sep=0,
            anchor=north
        ] (img2) at ($(text1.south) + (0,-3pt)$) {
            \includegraphics[
                width=\linewidth
            ]{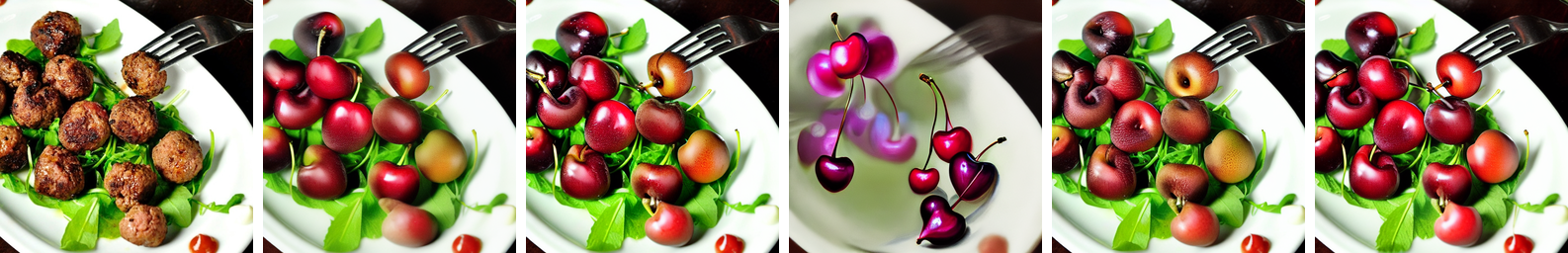}
        };


        \node[
            font=\small,
            anchor=north
        ] (text2) at ($(img2.south) + (0,-3pt)$)
     {
    \tikz[baseline=-1.5ex]{
        \node[
            fill=gray!15,
            rounded corners=2pt,
            inner xsep=3pt,
            inner ysep=1pt
        ] {\textcolor{gray}{\textit{Prompt:}}};
    }
    \;
   \sout{meatballs} \textcolor{red}{cherries} on a plate
};


        \node[
            inner sep=0,
            anchor=north
        ] (img3) at ($(text2.south) + (0,-3pt)$) {
            \includegraphics[
                width=\linewidth
            ]{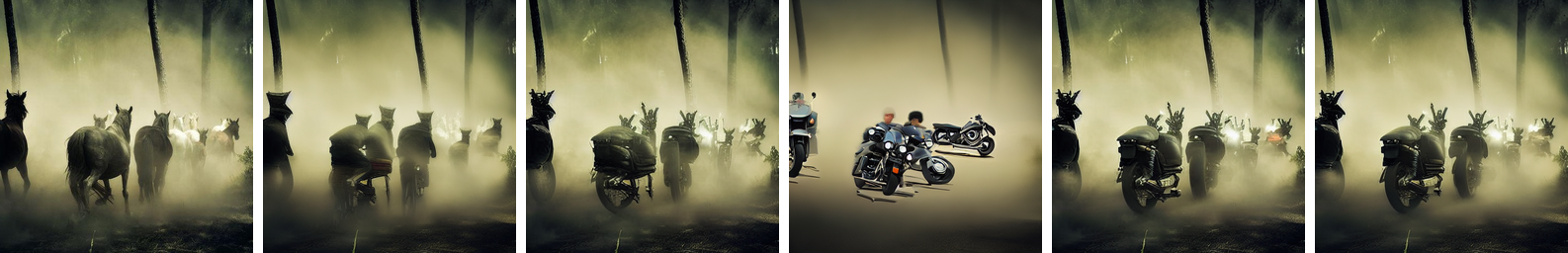}
        };

        \node[
            font=\small,
            anchor=north
        ] (text3) at ($(img3.south) + (0,-3pt)$)
     {
    \tikz[baseline=-1.5ex]{
        \node[
            fill=gray!15,
            rounded corners=2pt,
            inner xsep=3pt,
            inner ysep=1pt
        ] {\textcolor{gray}{\textit{Prompt:}}};
    }
    \;
    \texttt{running \sout{ horses} \textcolor{red}{motorcycles} in a foggy forest}
};


        \node[
            inner sep=0,
            anchor=north
        ] (img4) at ($(text3.south) + (0,-3pt)$) {
            \includegraphics[
                width=\linewidth
            ]{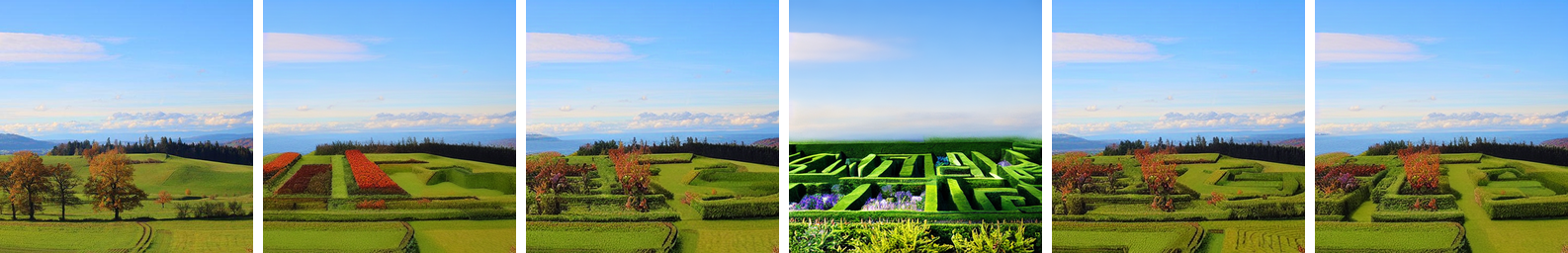}
        };

        \node[
            font=\small,
            anchor=north
        ] (text4) at ($(img4.south) + (0,-3pt)$)
     {
    \tikz[baseline=-1.5ex]{
        \node[
            fill=gray!15,
            rounded corners=2pt,
            inner xsep=3pt,
            inner ysep=1pt
        ] {\textcolor{gray}{\textit{Prompt:}}};
    }
    \;
   \texttt{a \sout{park} \textcolor{red}{maze garden} under a clear blue sky}
};


        \node[
            inner sep=0,
            anchor=north
        ] (img5) at ($(text4.south) + (0,-3pt)$) {
            \includegraphics[
                width=\linewidth
            ]{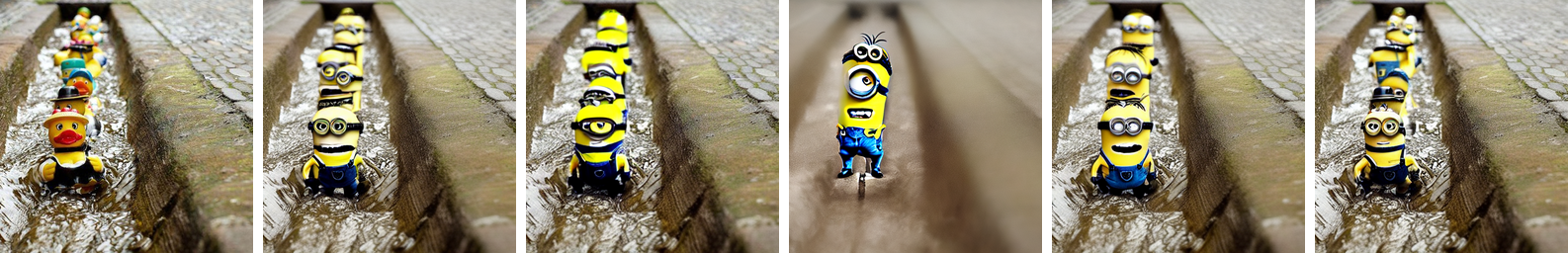}
        };

        \node[
            font=\small,
            anchor=north
        ] (text5) at ($(img5.south) + (0,-3pt)$)
     {
    \tikz[baseline=-1.5ex]{
        \node[
            fill=gray!15,
            rounded corners=2pt,
            inner xsep=3pt,
            inner ysep=1pt
        ] {\textcolor{gray}{\textit{Prompt:}}};
    }
    \;
    \texttt{a row of \sout{rubber ducks} \textcolor{red}{Minions}}
};


        \node[
            inner sep=0,
            anchor=north
        ] (img6) at ($(text5.south) + (0,-3pt)$) {
            \includegraphics[
                width=\linewidth
            ]{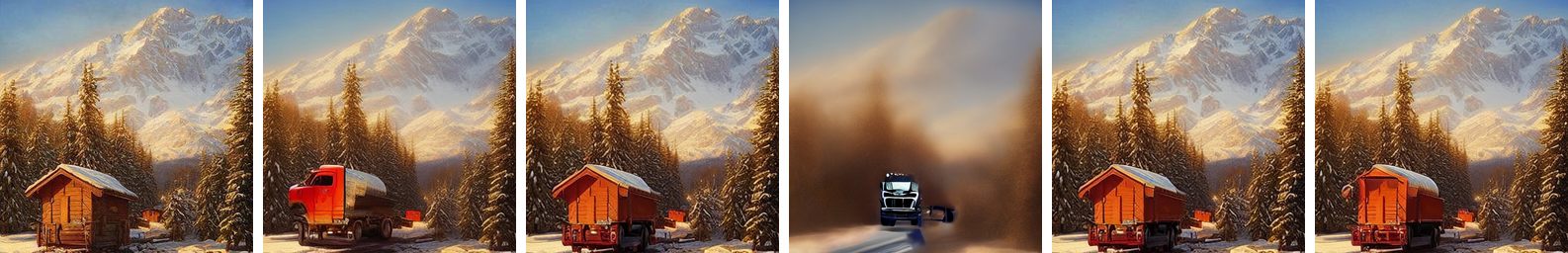}
        };

        \node[
            font=\small,
            anchor=north
        ] (text6) at ($(img6.south) + (0,-3pt)$)
     {
    \tikz[baseline=-1.5ex]{
        \node[
            fill=gray!15,
            rounded corners=2pt,
            inner xsep=3pt,
            inner ysep=1pt
        ] {\textcolor{gray}{\textit{Prompt:}}};
    }
    \;
   \texttt{a red \sout{house} \textcolor{red}{truck} on a snowy mountain}
};

    \end{tikzpicture}
\caption{
Additional qualitative comparison of image editing results. Given an input image and its corresponding text prompt, we compare the editing results produced by different methods. Strikethrough and red text indicate the original and edited text, respectively.}
\label{fig:image_edit_appen2}
\end{figure}
\section{Application 4: Physical System Modeling}
\label{app:pde_details}

This appendix provides the benchmark definitions, model and sampling configurations, evaluation protocol, and additional diagnostics supporting the physical-system results in Section~\ref{sec:physical_systems}. In particular, we separately quantify

\begin{itemize}
    \item satisfaction of the joint conservation, initial-condition, and boundary-condition operator;
    \item set-level agreement with a paired numerical-solution distribution;
    \item local PDE consistency and agreement with method-specific numerical rollouts; and
    \item sensitivity to integration resolution and candidate-time selection.
\end{itemize}

\subsection{Physical Systems}
\label{app:pde_systems}

To distinguish physical time from the generative-flow time $t$ used in
Section~\ref{sec:method}, we denote physical PDE time by $\tau$ throughout this
appendix.

\paragraph{Heat Equation.}
We consider
\begin{equation}
    \partial_\tau u
    =
    \nu\,\partial_{xx}u,
    \qquad
    x\in[0,2\pi),
    \qquad
    \tau\in[0,1].
    \label{eq:heat_pde}
\end{equation}
The analytic solution family used for data generation is
\begin{equation}
    u(x,\tau)
    =
    \sin(x+\phi)\exp(-\nu\tau),
    \qquad
    \nu\sim\mathcal{U}[1,5],
    \qquad
    \phi\sim\mathcal{U}[0,\pi].
    \label{eq:heat_solution}
\end{equation}
The training loader draws $5{,}000$ trajectories on demand from this family;
the configured test split contains $1{,}000$ additional draws. Fields are
represented as $(n_x,n_\tau)=(100,100)$ arrays on an endpoint-excluded
periodic spatial grid and an endpoint-inclusive time grid. These analytic
training trajectories are distinct from the benchmark reference pool defined
below, which solves the PDE from FFM-derived initial conditions.

\paragraph{Reaction-Diffusion Equation.}
The RD1D benchmark follows
\begin{equation}
    \partial_\tau u
    =
    \nu\,\partial_{xx}u
    +
    \rho u(1-u),
    \qquad
    \nu=5\times10^{-3},
    \qquad
    \rho=10^{-2},
    \label{eq:rd1d_pde}
\end{equation}
on
\begin{equation}
    x\in[0,1],
    \qquad
    \tau\in[0,0.99].
\end{equation}
Fields are represented on a $128\times100$ space--time grid.
The spatial coordinates are cell centers of the physical domain $[0,1]$, and
the boundary condition is specified by constant left and right Neumann fluxes
$g_L\in[0,0.05]$ and $g_R\in[-0.05,0]$. The persisted training set is the
Cartesian product of $80$ initial conditions and $80$ flux pairs
($6{,}400$ trajectories); the test set uses $16\times16=256$ trajectories.
Both were generated by the finite-volume solver described below. For each
benchmark source-noise ID, the IC and fluxes used as sampling targets are
derived once from the corresponding Vanilla trajectory and then held
fixed across all methods.

\subsection{Physical Constraint}
\label{app:constraints}
Our physical benchmark jointly imposes conservation, initial-condition, and boundary-condition consistency. We write the constraint operator from Section~\ref{sec:method} as
\begin{equation}
    H(u)
    =
    \begin{bmatrix}
        h_{\mathrm{cons}}(u)\\
        h_{\mathrm{IC}}(u)\\
        h_{\mathrm{BC}}(u)
    \end{bmatrix}.
    \label{eq:joint_H}
\end{equation}
MintFlow applies the trajectory intervention of (\ref{eq:optimal_delta}) to this joint operator rather than correcting the three components independently. The benchmark reports the sample-level statistic
\begin{equation}
    \|H(u)\|_\infty
    =
    \max_j |H_j(u)|,
    \label{eq:joint_max_residual}
\end{equation}
after the task-specific normalizations described below.

\paragraph{Initial condition.}
Let $u_{\mathrm{IC}}(x)$ denote the prescribed initial field associated with
the aligned benchmark sample.
The IC residual is
\begin{equation}
    h_{\mathrm{IC}}(u)
    =
    u(\cdot,0)-u_{\mathrm{IC}}.
    \label{eq:joint_ic}
\end{equation}
The target $u_{\mathrm{IC}}$ is copied from the aligned Vanilla sample at
$\tau=0$. Residuals are evaluated in the raw, unnormalized field; component
scales are applied only when constructing the combined statistic in
(\ref{eq:joint_max_residual}).

\paragraph{Heat conservation.}
Define the discrete spatial mass
\begin{equation}
    M(\tau_k)
    =
    \Delta x
    \sum_{i=1}^{n_x}
    u(x_i,\tau_k),
    \qquad
    \Delta x=\frac{2\pi}{100}.
\end{equation}
The conservation residual is
\begin{equation}
    h_{\mathrm{cons},k}(u)
    =
    M(\tau_k)-M(\tau_0),
    \qquad
    k=1,\ldots,n_\tau-1.
    \label{eq:heat_mass_constraint}
\end{equation}
This is the periodic rectangle rule on the endpoint-excluded grid. It compares
every generated time slice with the generated $\tau=0$ slice; the separate IC
block anchors that slice to $u_{\mathrm{IC}}$. The Heat operator therefore has
$99$ mass residuals and $100$ IC residuals, for a total dimension of $199$.

\paragraph{Heat boundary condition.}
Periodicity is enforced structurally by the Heat representation and therefore
does not introduce an independent residual:
\begin{equation}
    h_{\mathrm{BC}}(u)\equiv0.
    \label{eq:heat_bc_constraint}
\end{equation}
Consequently, the Heat statistic is determined by conservation and IC
consistency.

\paragraph{RD1D global balance.}
Let
\begin{equation}
    M(\tau)
    =
    \int_\Omega u(x,\tau)\,dx.
\end{equation}
The global-balance residual at time $\tau_k$ is
\begin{equation}
\begin{aligned}
    h_{\mathrm{cons},k}(u)
    &=
    M(\tau_k)-M(\tau_0)\\
    &\quad-
    \int_0^{\tau_k}
    \left[
        \rho\int_\Omega u(x,\tau)(1-u(x,\tau))\,dx
        +
        g_L(\tau)-g_R(\tau)
    \right]d\tau .
\end{aligned}
\label{eq:rd1d_balance_constraint}
\end{equation}
Spatial integration uses the cell-centered rectangle rule and temporal
accumulation uses the trapezoidal rule. The fluxes in
(\ref{eq:rd1d_balance_constraint}) are the prescribed external fluxes,
not values re-estimated from the generated field. The nonlinear reaction term
makes this component of $H$ nonlinear in $u$.

\paragraph{RD1D boundary condition.}
Let $\widehat g_L(u,\tau)$ and $\widehat g_R(u,\tau)$ denote the configured
five-point boundary-flux observables evaluated on the stored cell-centered
field.
The BC residual is
\begin{equation}
    h_{\mathrm{BC}}(u)
    =
    \begin{bmatrix}
        \widehat g_L(u,\tau)-g_L(\tau)\\
        \widehat g_R(u,\tau)-g_R(\tau)
    \end{bmatrix}_{\tau}.
    \label{eq:rd1d_bc_constraint}
\end{equation}
The five-point one-sided stencil is evaluated for positive-time slices only;
the classical corner compatibility condition at $\tau=0$ is not imposed. The
benchmark normalizes the two flux blocks by their physical range, $0.05$,
before constructing the joint operator:
\begin{equation}
    H_{\mathrm{RD1D}}(u)
    =
    \begin{bmatrix}
        h_{\mathrm{cons}}(u)\\
        h_{\mathrm{IC}}(u)\\
        h_{\mathrm{BC}}(u)/0.05
    \end{bmatrix}.
    \label{eq:rd1d_joint_operator}
\end{equation}
The global-balance and IC blocks have unit scale. Thus the RD1D operator has $99$ global-balance, $128$ IC, and $2\times99$ BC residuals, for a total dimension of $425$. The boundary observable used by the constraint is defined on stored cell centers; it is not identical to the external face-flux variable used internally by the finite-volume reference solver.

\subsection{Flow-Matching Model and Baselines}
\label{app:pde_model}

Each PDE has a separately pretrained checkpoint with the same architecture: a four-layer Fourier Neural Operator with $32$ Fourier modes per dimension, $64$ hidden channels, $256$-channel lifting and readout layers, and a $32$-dimensional sinusoidal embedding of generative-flow time. The two grid coordinates and time embedding are concatenated with the state, yielding $35$ input channels and one output channel. Direct instantiation of the current model gives $17{,}884{,}993$ trainable parameters.

The generative model is trained for $20{,}000$ optimization iterations using conditional flow matching. Given a data sample $u$ and source sample $z$, the interpolant and target velocity are
\begin{equation}
    u_s=(1-s)z+s u,
    \qquad
    v^\star=u-z,
    \qquad
    s\sim\mathcal{U}[0,1].
\end{equation}
The source $z$ is drawn from a Mat\'ern-$\frac{1}{2}$ Gaussian process with length scale $10^{-3}$ and unit output scale on the normalized two-dimensional grid. Training uses Adam with batch size $256$, learning rate $3\times10^{-4}$, and gradient-norm clipping at $100$. Both retained checkpoints are from iteration $20{,}000$.

We compare Vanilla sampling against PCFM, ECI, DiffusionPDE, D-Flow, and MintFlow. All constrained methods act only during generation and introduce no learned parameters. Within an equation, every method uses the same checkpoint, source-noise tensor, and sample-indexed constraint target. The paper reports each approach as a single method; implementation-level MintFlow choices are examined separately in the ablation study.

Vanilla uses adaptive Dopri5 integration with $100$ requested evaluation points for Heat and $200$ for RD1D. All fixed-step samplers---PCFM, ECI, DiffusionPDE, D-Flow, and MintFlow---use $100$ Euler steps on Heat and $200$ on RD1D. Thus the integration resolution is aligned across methods within each physical system. DiffusionPDE uses guidance strength $1$ on Heat and $10^{-3}$ on RD1D. D-Flow performs $20$ L-BFGS iterations with learning rate $1$. Standard ECI uses one mixing step, and standard PCFM uses $10$ Newton updates per Heat flow step and one per RD1D flow step.

The MintFlow configuration reported in Tables~\ref{tab:physical_constraints_benchmark}
and~\ref{tab:app_physical_constraints_benchmark}
uses $100$ Euler steps on Heat and $200$ on RD1D, with $K=10$ candidate
times in $[0.1,0.98]$, the end-biased warp
$s=s_{\min}+(s_{\max}-s_{\min})[1-(1-q)^2]$, $\lambda=10^{-2}$,
and correction scale $\gamma=1$. These details are provided for reproducibility
and do not denote separate methods in the main comparison.

\subsection{Evaluation Protocol}
\label{app:pde_metrics}

Each benchmark contains $N=1000$ samples per method. Heat uses source-noise seed $20260811$, and RD1D uses seed $20260820$. Reported means and standard deviations therefore summarize sample-to-sample variation in one fixed-seed run; they are not uncertainty estimates across training runs or sampling seeds.

\paragraph{Target bank and numerical reference.}
For each equation, Vanilla is first evaluated on the shared source-noise bank. Its initial slices define the $1{,}000$ IC targets. For Heat, a diffusivity in $[1,5]$ is fitted once from each complete Vanilla trajectory; for RD1D, constant Neumann fluxes are fitted once and clipped to the training ranges. These latent quantities are held fixed across methods at a given sample index. The numerical reference pool then solves each of these $1{,}000$ physical problems: Heat uses the periodic Fourier semigroup and RD1D uses the data generator's exact-reaction, explicit-midpoint finite-volume scheme. Thus the reference is neither the unconditional training distribution nor an unconstrained FFM pool. It is a solver distribution conditioned on the same finite bank of FFM-derived ICs, BCs, and latent physical parameters used by the constrained samplers. Set-level metrics ignore index pairing after constructing the two sample sets.

\paragraph{Constraint satisfaction.}
For sample $i$, $R_\infty(u_i)=\|H(u_i)\|_\infty$ is the largest absolute
entry after the RD1D BC normalization above. The main table reports its sample
mean, while Table~\ref{tab:app_physical_constraints_benchmark} additionally
reports the sample standard deviation.

\paragraph{Set-level distribution metrics.}
We avoid image-specific feature extractors by fitting one frozen, quadrature-weighted, mean-centered proper orthogonal decomposition (POD) to the $1{,}000$ numerical reference trajectories. The transform weights each flattened coordinate by the square root of its tensor-product physical quadrature weight, subtracts the weighted reference mean, and projects onto the leading right singular vectors. No whitening, nonlinear feature map, clipping, or method-specific refitting is applied. With a $99.9\%$ variance threshold and at most $20$ components, the retained dimensions are $5$ for Heat ($99.9463\%$ variance) and $14$ for RD1D ($99.9429\%$ variance).

Let $a_i$ and $b_j$ denote generated and reference POD coefficients. All
distances are evaluated in float64. FID is the Gaussian Fr\'echet distance
\begin{equation}
    \|\mu_a-\mu_b\|_2^2
    +\operatorname{tr}\!\left(
        C_a+C_b-2(C_a^{1/2}C_bC_a^{1/2})^{1/2}
    \right).
    \label{eq:pod_fid}
\end{equation}
It is an FID analogue in POD coordinates, not Inception-feature FID. KID is the unbiased two-sample MMD$^2$ estimator with $k(a,b)=(a^\top b/d+1)^3$; finite-sample estimates may be slightly negative. MMD is the square root of the nonnegative biased MMD$^2$ using the average of five RBF kernels. Their bandwidths are $\{0.25,0.5,1,2,4\}$ times the median pairwise distance fitted once on the reference coefficients. Finally, because both sets have equal size, empirical $W_2$ is computed exactly as
\begin{equation}
    W_2
    =\left[
        \min_{\pi\in S_N}\frac{1}{N}
        \sum_{i=1}^{N}\|a_i-b_{\pi(i)}\|_2^2
    \right]^{1/2}
    \label{eq:pod_w2}
\end{equation}
using the Hungarian assignment algorithm. These four quantities measure agreement with the conditioned numerical distribution. They do not directly estimate divergence from the unconstrained FFM distribution; Vanilla is reported as an evaluated comparator, not used as the metric reference.

\paragraph{Physical consistency.}
The PDE residual is the root mean square of a finite-difference operator in the raw field. Heat uses a second-order periodic spatial Laplacian and a second-order numerical temporal derivative; RD1D uses the corresponding second-order interior spatial stencil. For the simulation error, every method uses its own generated initial slice and the physical parameters and BCs fixed for that sample index:
\begin{equation}
    E_{\mathrm{sim}}(u_i)
    =\frac{\|u_i-u_i^{\mathrm{solver}}\|_2}
           {\|u_i^{\mathrm{solver}}\|_2}.
    \label{eq:simulation_error}
\end{equation}
This metric tests dynamical self-consistency and does not by itself test
whether the generated IC matches its prescribed target. The main table reports mean simulation error. Table~\ref{tab:app_physical_constraints_benchmark}
additionally reports the finite-difference PDE residual and sample standard deviations for both physical-consistency metrics.

\begin{table*}[h]
\centering
\caption{Physical System Modeling Benchmark under initial conditions, boundary conditions, and conservation constraints.
Performance evaluation across $N=1000$ generated samples on the Heat Equation (System 1) and Reaction-Diffusion Equation (System 2). 
The benchmark highlights the trade-off between \emph{Constraint Satisfaction} and \emph{Distribution Preservation}. 
While strict-constraint baselines (ECI, PCFM) achieve near-zero constraint residuals, they severely degrade distributional quality. 
In contrast, MintFlow achieves an optimal Pareto frontier, enforcing physical constraints while preserving the pre-trained data distribution. 
For all metrics, lower values ($\downarrow$) indicate superior performance.
}
\label{tab:app_physical_constraints_benchmark}
\resizebox{\textwidth}{!}{
\begin{tabular}{lccccccc}
\toprule
&
\multicolumn{4}{c}{\textbf{Distribution Fidelity}{$^\dagger$}}
&
\multicolumn{1}{c}{\textbf{Constraint}{$^{\dagger\dagger}$}}
&
\multicolumn{2}{c}{\textbf{Physical Consistency}{$^{\dagger\dagger\dagger}$}}
\\
\cmidrule(lr){2-5}
\cmidrule(lr){6-6}
\cmidrule(lr){7-8}
\textbf{Method}
& \textbf{FID} $\downarrow$
& \textbf{KID} $\downarrow$
& \textbf{MMD} $\downarrow$
& $\mathbf{W_2}\downarrow$
& $\mathbf{H}\downarrow$
& \textbf{PDE Res.}  $\downarrow$
& \textbf{Sim. Error} $\downarrow$
\\
\midrule

\multicolumn{8}{l}{\textbf{\textit{System 1: Heat Equation}}}
\\

Vanilla FFM
& \cellcolor{red!10}$1.255\mathrm{E}{-3}$
& \cellcolor{red!10}$-3.17\mathrm{E}{-4}$
& \cellcolor{red!10}$2.496\mathrm{E}{-2}$
& $4.872\mathrm{E}{-2}$
& $5.16\mathrm{E}{-3}\!\pm\!2.19\mathrm{E}{-3}$
& $1.103\mathrm{E}{0}\!\pm\!5.08\mathrm{E}{-1}$
& $5.57\mathrm{E}{-2}\!\pm\!6.15\mathrm{E}{-2}$
\\

DiffusionPDE
& $4.712\mathrm{E}{-3}$
& $1.019\mathrm{E}{-3}$
& $5.361\mathrm{E}{-2}$
& $7.431\mathrm{E}{-2}$
& $1.86\mathrm{E}{-2}\!\pm\!2.20\mathrm{E}{-2}$
& $1.044\mathrm{E}{0}\!\pm\!6.98\mathrm{E}{-1}$
& $8.49\mathrm{E}{-2}\!\pm\!2.95\mathrm{E}{-2}$
\\

D-Flow
& $1.290\mathrm{E}{-3}$
& $-2.10\mathrm{E}{-4}$
& $2.554\mathrm{E}{-2}$
& \cellcolor{red!10}$4.098\mathrm{E}{-2}$
& $4.71\mathrm{E}{-3}\!\pm\!2.08\mathrm{E}{-3}$
& \cellcolor{red!10}$7.82\mathrm{E}{-1}\!\pm\!3.76\mathrm{E}{-1}$
& \cellcolor{red!10}$4.34\mathrm{E}{-2}\!\pm\!2.46\mathrm{E}{-2}$
\\

ECI
& $1.865\mathrm{E}{-1}$
& $4.004\mathrm{E}{-2}$
& $2.067\mathrm{E}{-1}$
& $4.502\mathrm{E}{-1}$
& \cellcolor{red!10}$2.27\mathrm{E}{-7}\!\pm\!6.45\mathrm{E}{-8}$
& $1.554\mathrm{E}{0}\!\pm\!8.78\mathrm{E}{-1}$
& $4.47\mathrm{E}{-1}\!\pm\!2.86\mathrm{E}{-1}$
\\

PCFM
& $1.851\mathrm{E}{-1}$
& $3.965\mathrm{E}{-2}$
& $2.062\mathrm{E}{-1}$
& $4.485\mathrm{E}{-1}$
& \cellcolor{green!15}$1.30\mathrm{E}{-7}\!\pm\!5.28\mathrm{E}{-8}$
& $1.549\mathrm{E}{0}\!\pm\!8.72\mathrm{E}{-1}$
& $4.45\mathrm{E}{-1}\!\pm\!2.85\mathrm{E}{-1}$
\\

\textbf{MintFlow}
& \cellcolor{green!15}$2.91\mathrm{E}{-5}$
& \cellcolor{green!15}$-5.22\mathrm{E}{-4}$
& \cellcolor{green!15}$5.44\mathrm{E}{-3}$
& \cellcolor{green!15}$1.455\mathrm{E}{-2}$
& $3.57\mathrm{E}{-6}\!\pm\!3.69\mathrm{E}{-5}$
& \cellcolor{green!15}$6.77\mathrm{E}{-1}\!\pm\!3.22\mathrm{E}{-1}$
& \cellcolor{green!15}$1.23\mathrm{E}{-2}\!\pm\!2.09\mathrm{E}{-2}$
\\

\midrule

\multicolumn{8}{l}{\textbf{\textit{System 2: Reaction-Diffusion Equation}}}
\\

Vanilla FFM
& \cellcolor{red!10}$1.74\mathrm{E}{-4}$
& \cellcolor{red!10}$-8.11\mathrm{E}{-6}$
& \cellcolor{red!10}$2.130\mathrm{E}{-2}$
& \cellcolor{red!10}$2.078\mathrm{E}{-2}$
& $1.907\mathrm{E}{0}\!\pm\!7.75\mathrm{E}{-1}$
& $2.98\mathrm{E}{-1}\!\pm\!5.9\mathrm{E}{-2}$
& $3.55\mathrm{E}{-2}\!\pm\!1.72\mathrm{E}{-2}$
\\

DiffusionPDE
& $9.34\mathrm{E}{-4}$
& $1.43\mathrm{E}{-4}$
& $6.063\mathrm{E}{-2}$
& $3.712\mathrm{E}{-2}$
& $9.56\mathrm{E}{-1}\!\pm\!2.50\mathrm{E}{-1}$
& $3.43\mathrm{E}{-1}\!\pm\!7.3\mathrm{E}{-2}$
& $3.29\mathrm{E}{-2}\!\pm\!2.05\mathrm{E}{-2}$
\\

D-Flow
& $3.95\mathrm{E}{-4}$
& $1.45\mathrm{E}{-5}$
& $3.861\mathrm{E}{-2}$
& $4.085\mathrm{E}{-2}$
& $4.44\mathrm{E}{-1}\!\pm\!1.40\mathrm{E}{-1}$
& \cellcolor{red!10}$2.45\mathrm{E}{-1}\!\pm\!4.4\mathrm{E}{-2}$
& \cellcolor{red!10}$3.02\mathrm{E}{-2}\!\pm\!2.03\mathrm{E}{-2}$
\\

ECI
& $1.829\mathrm{E}{-2}$
& $2.59\mathrm{E}{-4}$
& $1.299\mathrm{E}{-1}$
& $1.647\mathrm{E}{-1}$
& \cellcolor{green!15}$8.02\mathrm{E}{-6}\!\pm\!2.72\mathrm{E}{-6}$
& $3.250\mathrm{E}{0}\!\pm\!5.866\mathrm{E}{0}$
& $2.28\mathrm{E}{-1}\!\pm\!2.87\mathrm{E}{-1}$
\\

PCFM
& $1.831\mathrm{E}{-2}$
& $2.61\mathrm{E}{-4}$
& $1.301\mathrm{E}{-1}$
& $1.648\mathrm{E}{-1}$
& \cellcolor{red!10}$8.06\mathrm{E}{-6}\!\pm\!2.81\mathrm{E}{-6}$
& $3.255\mathrm{E}{0}\!\pm\!5.871\mathrm{E}{0}$
& $2.28\mathrm{E}{-1}\!\pm\!2.87\mathrm{E}{-1}$
\\

\textbf{MintFlow}
& \cellcolor{green!15}$1.09\mathrm{E}{-4}$
& \cellcolor{green!15}$-1.47\mathrm{E}{-5}$
& \cellcolor{green!15}$1.342\mathrm{E}{-2}$
& \cellcolor{green!15}$1.870\mathrm{E}{-2}$
& $3.18\mathrm{E}{-4}\!\pm\!2.16\mathrm{E}{-4}$
& \cellcolor{green!15}$2.28\mathrm{E}{-1}\!\pm\!3.0\mathrm{E}{-2}$
& \cellcolor{green!15}$2.58\mathrm{E}{-2}\!\pm\!1.40\mathrm{E}{-2}$
\\

\bottomrule
\end{tabular}
}
\begin{tablenotes}
\footnotesize 
\item $\dagger$ Distribution fidelity evaluates closeness of generated samples to the numerical reference trajectory distribution in a quadrature-weighted, mean-centered POD space. Slightly negative KID values occur near zero due to unbiased polynomial sampling noise.
\item $\dagger\dagger$ $H$ denotes the combined normalized maximum absolute constraint residual (\textbf{$H$}) across conservation laws, initial conditions (IC), and boundary conditions (BC), reported as $\text{mean} \pm \text{std. dev.}$. Heat periodicity is structural, while reaction-diffusion boundary fluxes are normalized by $0.05$.
\item $\dagger\dagger\dagger$ PDE Res. denotes the finite-difference PDE operator residual. Sim. Error measures relative $\ell_2$ trajectory error when integrating a numerical solver forward from each sample's generated initial condition.
\end{tablenotes}
\end{table*}

\subsection{Full Quantitative Results}
\label{app:pde_quantitative}
The expanded results reinforce the trade-off visible in the main table. MintFlow has the smallest point estimate for all four distribution metrics and both physical-consistency metrics on both systems, while ECI and PCFM obtain the smallest constraint residuals. Their stronger explicit constraint satisfaction is accompanied by substantially larger distribution distances, PDE residuals, and simulation errors. D-Flow is the strongest competing baseline on physical consistency, but its constraint residual remains much larger than those of MintFlow, ECI, and PCFM. The standard deviations further show that mean performance can conceal heterogeneous behavior: the standard deviation of MintFlow's Heat constraint residual exceeds its mean, while the RD1D PDE-residual standard deviations of ECI and PCFM exceed their corresponding means. By contrast, MintFlow's RD1D PDE residual and simulation error have comparatively small deviations. Finally, the negative KID values are valid finite-sample outcomes of the unbiased estimator and should be interpreted as estimates close to zero, not as negative distances. All distributional conclusions refer to the conditioned numerical reference rather than directly to the unconditional FFM distribution.

\subsection{Additional Qualitative Results}
\label{app:pde_qualitative}
Figures~\ref{fig:heat_additional_qualitative} and~\ref{fig:rd1d_additional_qualitative} provide fixed-index examples from the retained benchmark runs. For every method and sample, the upper row shows the generated space--time field. We initialize the numerical PDE solver with the initial condition extracted from that field and use the sample-indexed physical parameters and boundary conditions. The lower row reports the pointwise absolute discrepancy from this method-specific numerical trajectory, corresponding to (\ref{eq:simulation_error}) before spatial and temporal aggregation. These plots separate explicit constraint satisfaction from dynamical self-consistency: a field can have a small value of $\|H(u)\|_\infty$ while deviating from the PDE solution implied by its own initial condition. ECI and PCFM show broad errors in several selected cases, whereas MintFlow is closer to its corresponding rollout. Because the panels contain selected samples and omit some evaluated variants, they are diagnostic rather than evidence of set-level dominance.

\begin{figure}[h]
    \centering
    \includegraphics[width=\linewidth]{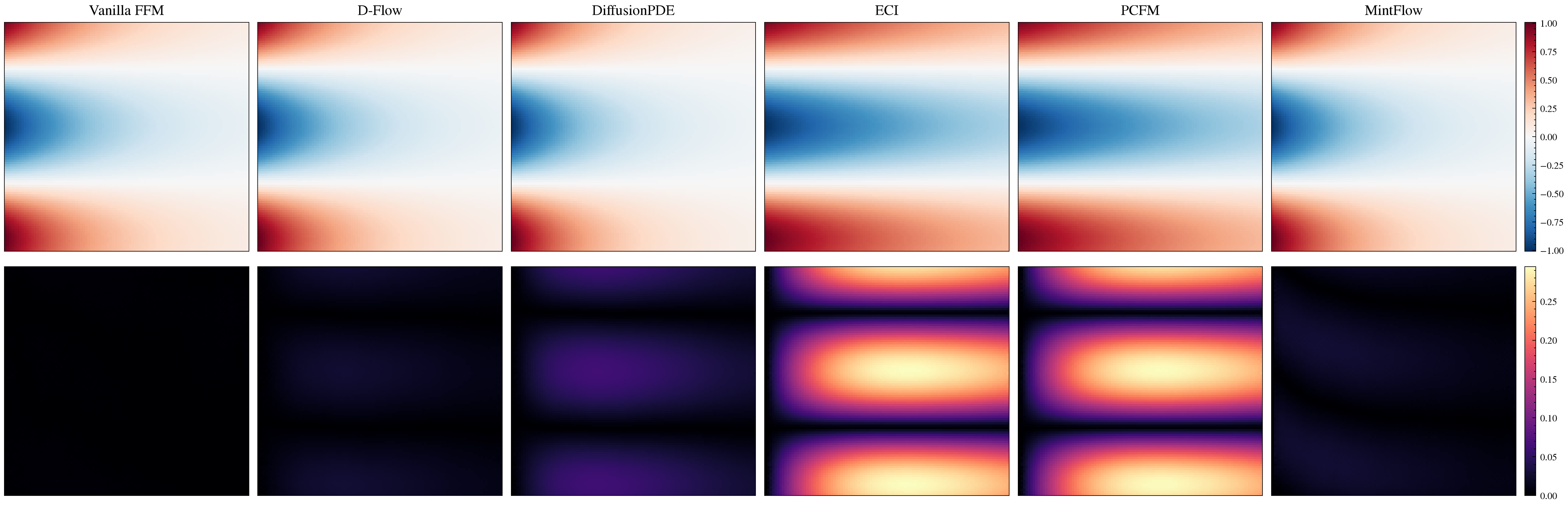}
    \caption{Additional Heat qualitative results. For each example, the first row shows the generated spacetime field for each method. The second row shows the pointwise absolute error to a numerical Heat trajectory initialized from that generated sample's own initial condition.}
    \label{fig:heat_additional_qualitative}
\end{figure}

\begin{figure}[h]
    \centering
    \includegraphics[width=\linewidth]{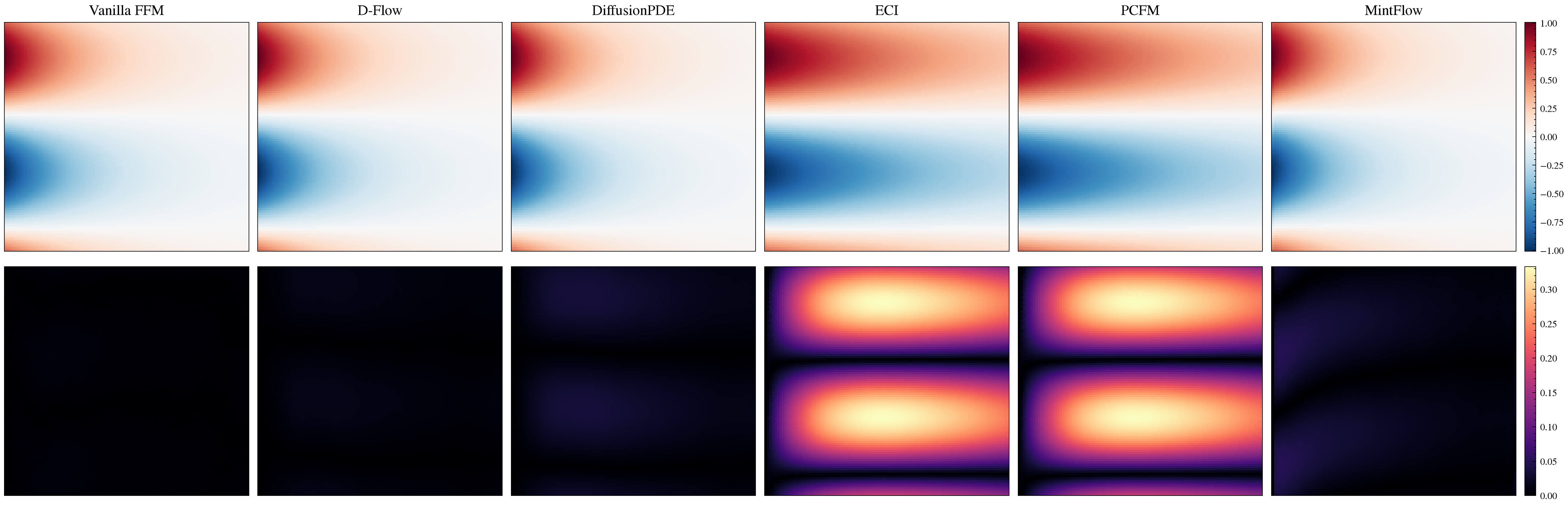}
    \caption{Additional Heat qualitative results.}
    \label{fig:heat_additional_qualitative}
\end{figure}

\begin{figure}[h]
    \centering
    \includegraphics[width=\linewidth]{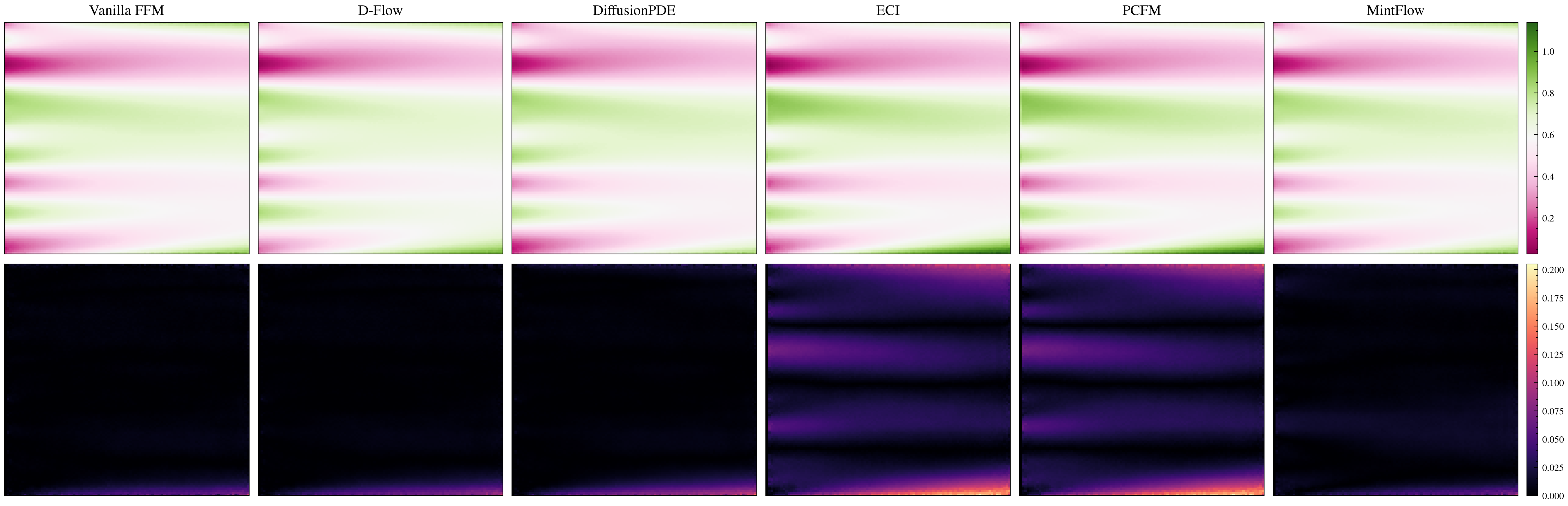}
    \includegraphics[width=\linewidth]{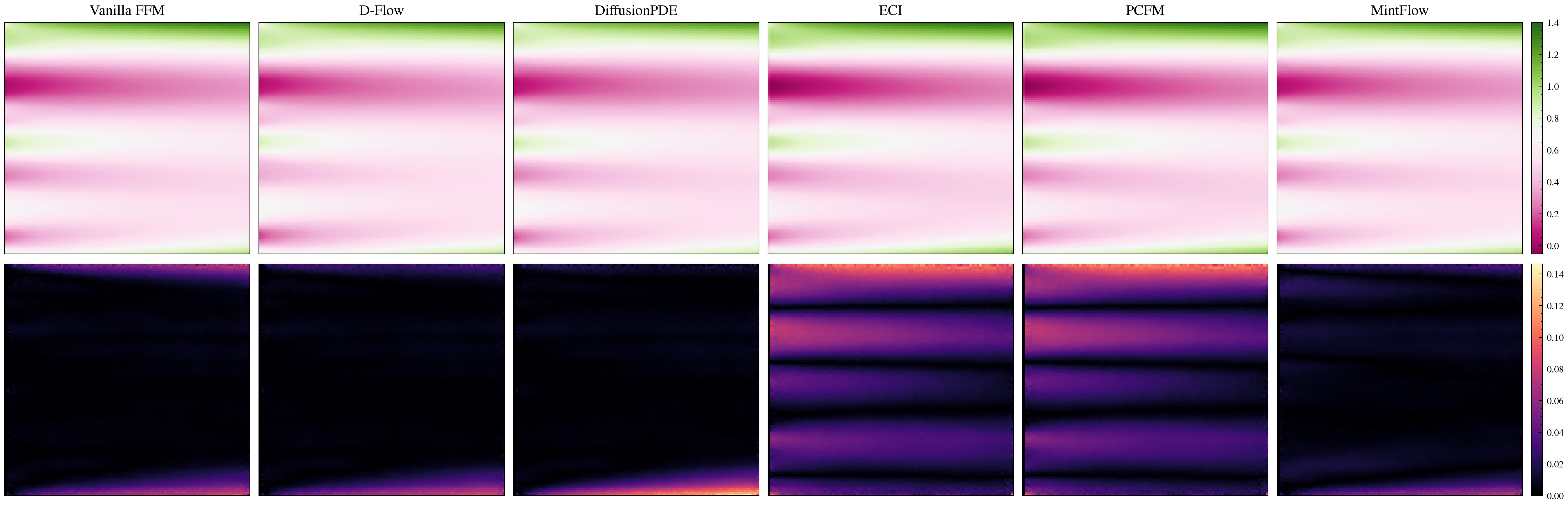}
    \includegraphics[width=\linewidth]{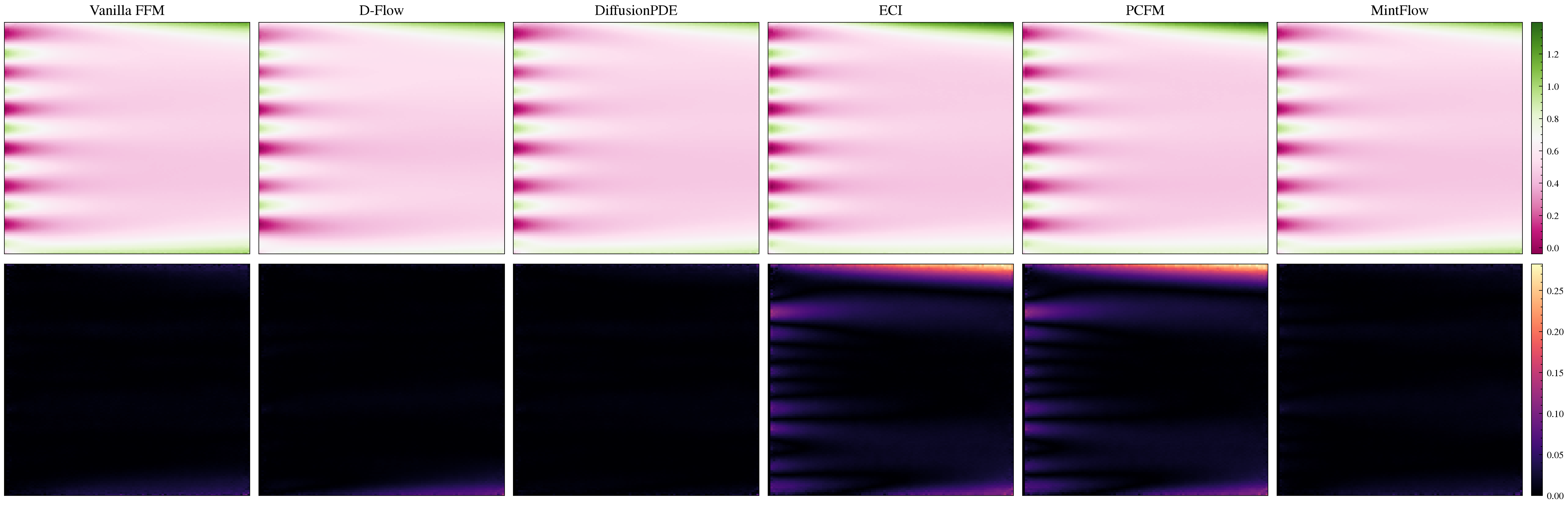}
    \caption{Additional RD1D qualitative results.}
    \label{fig:rd1d_additional_qualitative}
\end{figure}

\newpage
\section{Ablation Studies}\label{app:pde_ablation}

\subsection{Ablation on Integration Resolution}
We ablate MintFlow on the Heat benchmark using the same $N=1000$ source-noise bank, target bank, numerical reference pool, and frozen POD transform as in the main comparison. We vary only sampling integration resolution, holding all other sampler settings fixed. Consequently, differences
across configurations cannot be attributed to resampling or refitting the
representation.

As shown in Figure~\ref{fig:ablation_heat}, Increasing the number of Euler steps of flow matching model from $50$ to $100$, $200$, and $400$
monotonically improves FID, the constraint residual, and same-IC simulation
error. 
Importantly, the discretization of Euler steps affects not only the forward sampling integration error but also the numerical error in solving the backward adjoint equation. A finer discretization therefore yields both a sampling trajectory and a adjoint that more closely approximate their continuous-time counterparts. This leads to a more faithful preservation of the pretrained trajectory and, consequently, improved distributional fidelity, while also providing a more accurate local correction direction for enforcing the terminal constraint. Nevertheless, the framework remains relatively robust across different resolutions, with more noticeable degradation only when fewer than 100 sampling steps are used.

The finite-difference PDE residual improves through $200$ steps but
rises slightly at $400$ steps. The headline Heat comparison nevertheless uses
$100$ steps to match the integration resolution across all fixed-step methods;
the higher-resolution settings quantify numerical sensitivity rather than
define the reported operating point. The headline absolute scores remain those in
Tables~\ref{tab:physical_constraints_benchmark}
and~\ref{tab:app_physical_constraints_benchmark}; the ablation is used to
assess relative sensitivity to individual design choices.

\subsection{Ablation on Candidate Time Selection and optimal-time regularization.}
We further ablate the regularization parameter $\lambda$ in the objective function $J$ in (\ref{eq:optimal_time}) for adaptive intervention time selection. We use the same data and experimental setup as in the preceding section, keeping the source-noise bank, target bank, numerical reference pool, and frozen POD transform fixed. We vary only $\lambda$, while holding all other sampler settings fixed. Consequently, differences across configurations cannot be attributed to resampling or refitting the representation.

Reults in Figure~\ref{fig:ablation_heat} shows that varying $\lambda$ reveals a trade-off between constraint satisfaction and distribution preservation. 
Since $\lambda$ does not directly regularize the intervention magnitude but instead determines the preferred intervention time $s$, a larger $\lambda$ favors later interventions. When the intervention occurs very close to the terminal time, the correction increasingly resembles a direct perturbation of the terminal sample toward the constraint manifold, leaving less opportunity for the pretrained flow to absorb the intervention. Although such a correction can more strongly enforce the constraint, a large perturbation may move the sample farther from the original terminal distribution $\nu_T$. In practice, the performance remains relatively stable for a moderate range of $\lambda$, indicating that the proposed objective provides a robust balance between constraint satisfaction and distributional preservation. 
The weak sensitivity at small $\lambda$ indicates that the candidate grid, rather
than the regularization term, governs the selected correction times in this
regime.

\subsection{Ablation on Intervention steps}
We conduct an ablation study on the number of intervention steps to investigate how our method responds to repeated trajectory interventions. In image editing, the generated image must simultaneously follow the target prompt and satisfy the constraint, making a single intervention generally insufficient to achieve both objectives. Therefore, in all preceding experiments, we consistently use $N=20$ intervention steps, following the setting of baseline methods \citep{patel2024flowchef}. In this section, we vary the number of intervention steps $N$ and examine its effect on the editing results.

As shown in Fig.~\ref{fig:intervention_steps_ablation}, using too few intervention steps often produces images that are aligned with the target prompt but deviate substantially from the input image, indicating that a single or small number of interventions is insufficient to generate an image that satisfies the target prompt while preserving the input constraint. As $N$ increases, the editing results progressively become more consistent with the input image while continuing to satisfy the target prompt. After a certain number of intervention steps, the results converge and remain largely unchanged despite further increasing $N$, suggesting that a moderate number of interventions is sufficient to achieve the desired editing effect.

Importantly, even when a large number of interventions are applied, our method continues to generate plausible images without departing from the learned distribution. This indicates that the proposed minimal-intervention strategy does not introduce excessive perturbations that could push the generated samples out of the learned data distribution. Overall, these observations demonstrate the effectiveness and stability of our intervention strategy, while also providing practical guidance on the number of interventions required for image editing.

\begin{figure*}[t]
    \centering

    \setlength{\tabcolsep}{2pt}

    \begin{tabular}{ccccccc}
        Input &
        $N=5$ &
        $N=10$ &
        $N=15$ &
        $N=20$ &
        $N=25$ &
        $N=30$
        \\[3pt]
        \includegraphics[width=0.137\linewidth]{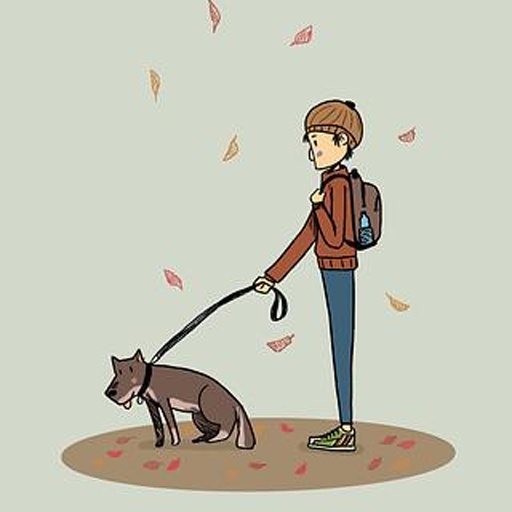} &
        \includegraphics[width=0.137\linewidth]{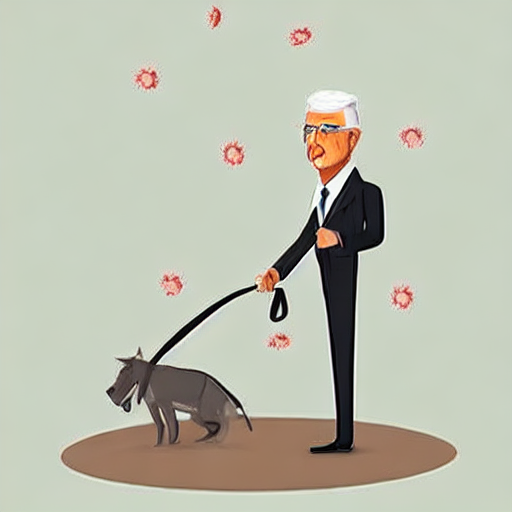} &
        \includegraphics[width=0.137\linewidth]{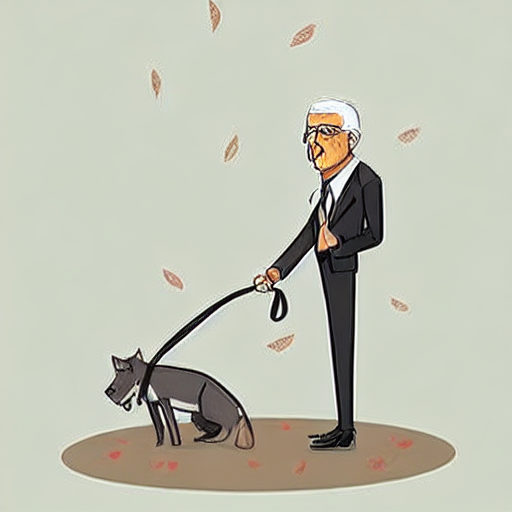} &
        \includegraphics[width=0.137\linewidth]{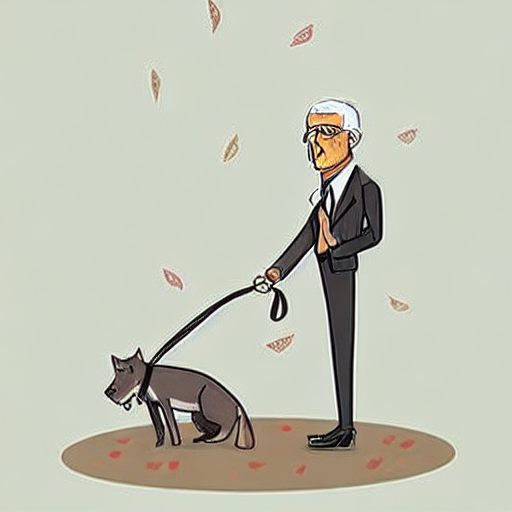} &
        \includegraphics[width=0.137\linewidth]{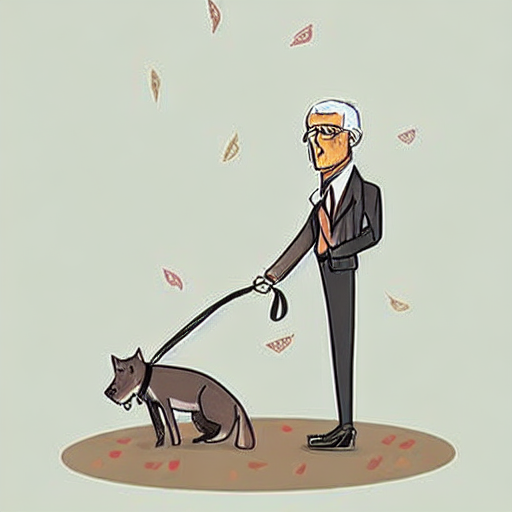} &
        \includegraphics[width=0.137\linewidth]{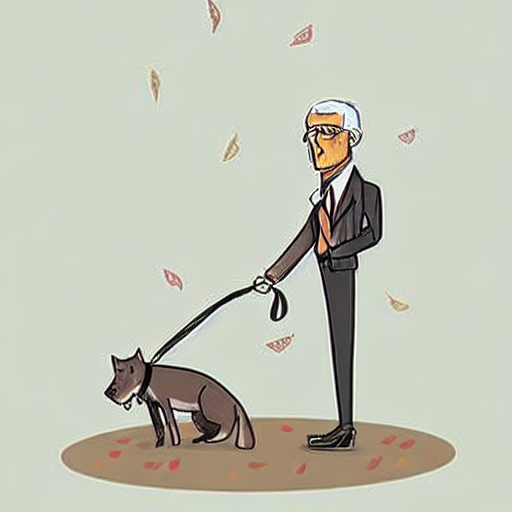} &
        \includegraphics[width=0.137\linewidth]{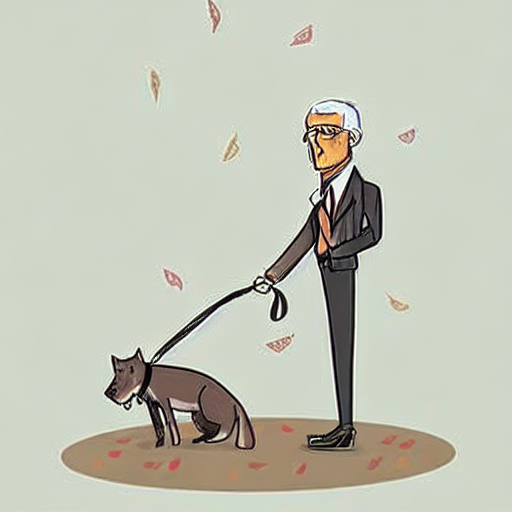}
    \end{tabular}

    \vspace{5pt}

    \begin{tikzpicture}[baseline=-0.5ex]
        \node[
            fill=gray!15,
            rounded corners=2pt,
            inner xsep=4pt,
            inner ysep=2pt
        ] {
            \textcolor{gray}{\textit{Prompt:}}
        };
    \end{tikzpicture}
    \;
    \texttt{a \sout{boy wearing a sweater} \textcolor{red}{elderly man wearing a suit} holding a leash with a dog}

\vspace{10pt}

\begin{tabular}{ccccccc}
    \includegraphics[width=0.137\linewidth]{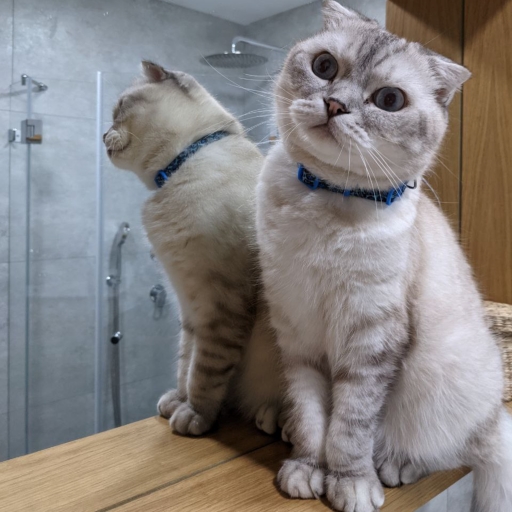} &
    \includegraphics[width=0.137\linewidth]{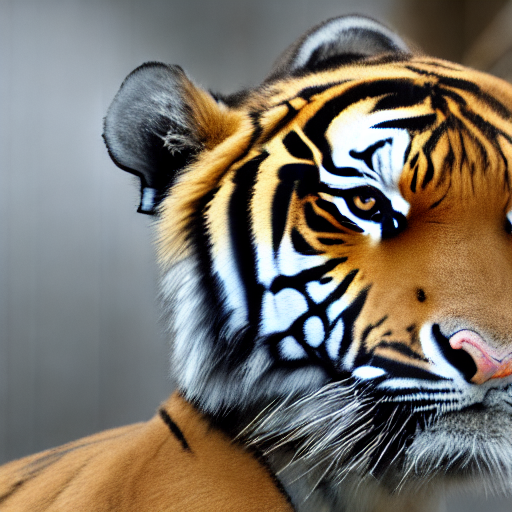} &
    \includegraphics[width=0.137\linewidth]{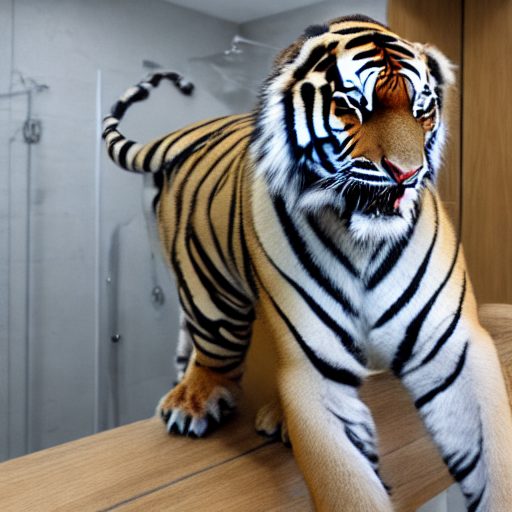} &
    \includegraphics[width=0.137\linewidth]{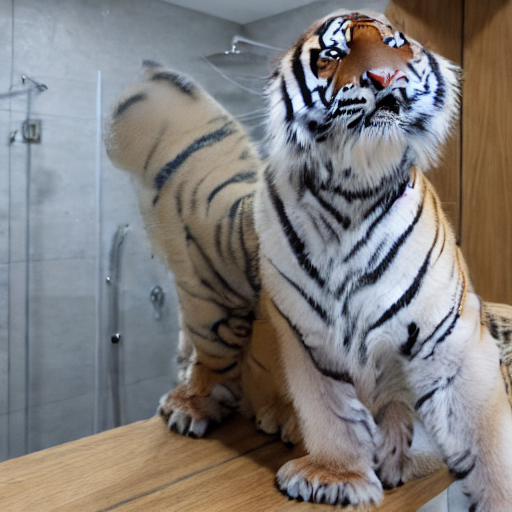} &
    \includegraphics[width=0.137\linewidth]{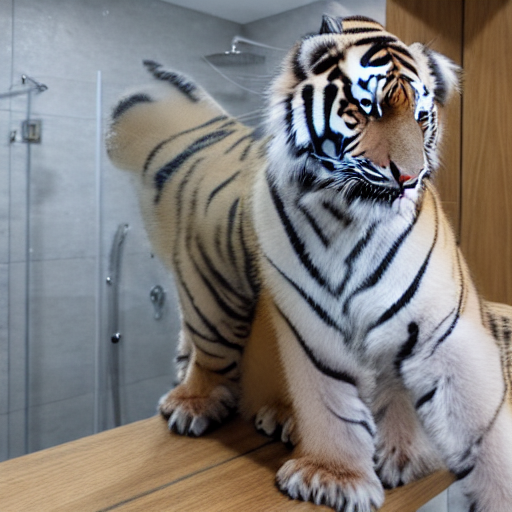} &
    \includegraphics[width=0.137\linewidth]{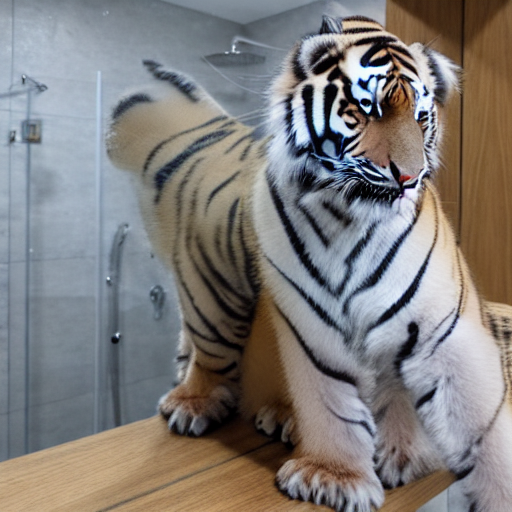} &
    \includegraphics[width=0.137\linewidth]{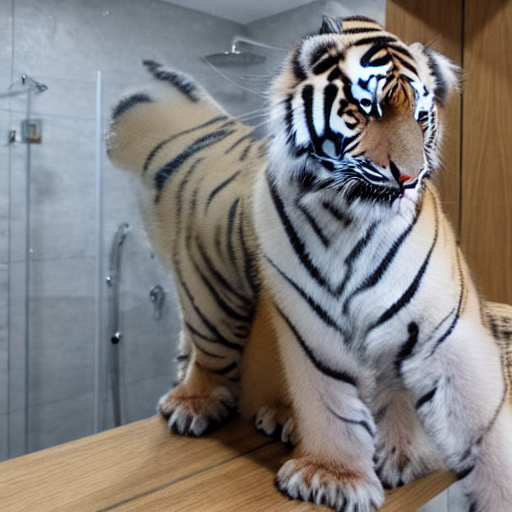}
\end{tabular}

\vspace{5pt}

\begin{tikzpicture}[baseline=-0.5ex]
    \node[
        fill=gray!15,
        rounded corners=2pt,
        inner xsep=4pt,
        inner ysep=2pt
    ] {
        \textcolor{gray}{\textit{Prompt:}}
    };
\end{tikzpicture}
\;
\texttt{a \sout{cat} \textcolor{red}{tiger} sitting by the shower glass}
    \caption{
Comparison of image editing results with different numbers of interventions $N$.
When $N$ is too small, the generated images align with the target prompt but fail to satisfy the constraint, resulting in substantial deviations from the input image.
As $N$ increases, the method better preserves the background and overall appearance of the input image while naturally modifying only the masked regions to satisfy the target prompt.
Beyond a certain $N$, the results remain largely unchanged, suggesting that a moderate number of interventions is sufficient, while even with more interventions, our method remains within the learned distribution, demonstrating its stability.
}\label{fig:intervention_steps_ablation}
\end{figure*}

\end{document}